%% file: main.tex
\documentclass[nohyperref]{article}

\usepackage{microtype}
\usepackage{graphicx}
\usepackage{subfigure}
\usepackage{booktabs} %

\usepackage{hyperref}

\usepackage{nccmath}

\usepackage{tikz}
\usetikzlibrary{shapes.geometric, arrows}
\usetikzlibrary{decorations.pathreplacing}
\usetikzlibrary{fit}
\usetikzlibrary{calc}
\usetikzlibrary{backgrounds}
\usetikzlibrary{patterns}
\usetikzlibrary{positioning,arrows.meta,bending}

\tikzstyle{mathtext}=[text badly centered, font={\fontsize{12pt}{12pt}\selectfont}]
\tikzstyle{smalltext}=[font={\fontsize{8pt}{8pt}\selectfont}]
\tikzstyle{arrow} = [thick,->]
\tikzstyle{dasharrow} = [dashed,->]
\tikzstyle{doublearrow} = [thick,->,-{Implies},double]

\newcommand\Tstrut{\rule{0pt}{2.6ex}}         %
\newcommand\Bstrut{\rule[-0.9ex]{0pt}{0pt}}   %

\usepackage[accepted]{icml2022}

\usepackage{multirow}
\usepackage{physics}
\usepackage{wrapfig}
\usepackage{cancel}

\usepackage{amsmath}
\usepackage{amssymb}
\usepackage{mathtools}
\usepackage{amsthm}

\usepackage[capitalize,noabbrev]{cleveref}

\theoremstyle{plain}
\newtheorem{theorem}{Theorem}[section]
\newtheorem{proposition}[theorem]{Proposition}
\newtheorem{lemma}[theorem]{Lemma}
\newtheorem{corollary}[theorem]{Corollary}
\theoremstyle{definition}

\newtheorem{assumption}[theorem]{Assumption}
\newtheorem*{remark}{Remark}
\newenvironment{proof-sketch}{\noindent{\textit{Sketch of proof.}}\hspace*{1em}}{\qed\bigskip}

\usepackage{ifthen}
\newif\ifdebug
\debugfalse
\ifdebug
\usepackage[textsize=tiny,color=lightgray,textwidth=2cm]{todonotes}
\paperwidth=\dimexpr \paperwidth + 2cm\relax
\oddsidemargin=\dimexpr\oddsidemargin + 1cm\relax
\evensidemargin=\dimexpr\evensidemargin + 1cm\relax
\else
\usepackage[disable]{todonotes}
\fi

\input{common}

\icmltitlerunning{Maximum Likelihood Training for Score-Based Diffusion ODEs}

\begin{document}

\twocolumn[
\icmltitle{Maximum Likelihood Training for Score-Based Diffusion ODEs\\ by High-Order Denoising Score Matching}

\icmlsetsymbol{equal}{*}

\begin{icmlauthorlist}
\icmlauthor{Cheng Lu}{tsinghua}
\icmlauthor{Kaiwen Zheng}{tsinghua}
\icmlauthor{Fan Bao}{tsinghua}
\icmlauthor{Jianfei Chen}{tsinghua}
\icmlauthor{Chongxuan Li}{ruc}
\icmlauthor{Jun Zhu}{tsinghua}
\end{icmlauthorlist}

\icmlaffiliation{tsinghua}{Dept. of Comp. Sci. \& Tech., Institute for AI, Tsinghua-Bosch Joint Center for ML, BNRist Center, State Key Lab for Intell. Tech. \& Sys., Tsinghua University; Peng Cheng Laboratory}
\icmlaffiliation{ruc}{Gaoling School of AI, Renmin University of China; Beijing Key Lab of Big Data Management and Analysis Methods , Beijing, China}

\icmlcorrespondingauthor{Jianfei Chen}{jianfeic@tsinghua.edu.cn}
\icmlcorrespondingauthor{Jun Zhu}{dcszj@tsinghua.edu.cn}

\icmlkeywords{neural ODE, high order score matching, score-based generative model}

\vskip 0.3in
]

\printAffiliationsAndNotice{}  %

\begin{abstract}
Score-based generative models have excellent performance in terms of generation quality and likelihood. They model the data distribution by matching a parameterized score network with first-order data score functions. The score network can be used to define an ODE (``score-based diffusion ODE'') for exact likelihood evaluation. However, the relationship between the likelihood of the ODE and the score matching objective is unclear. In this work, we prove that matching the first-order score is not sufficient to maximize the likelihood of the ODE, by showing a gap between the maximum likelihood and score matching objectives. To fill up this gap, we show that the negative likelihood of the ODE can be bounded by controlling the first, second, and third-order score matching errors; and we further present a novel high-order denoising score matching method to enable maximum likelihood training of score-based diffusion ODEs. Our algorithm guarantees that the higher-order matching error is bounded by the training error and the lower-order errors. We empirically observe that by high-order score matching, score-based diffusion ODEs achieve better likelihood on both synthetic data and CIFAR-10, while retaining the high generation quality.

\end{abstract}

\input{1-intro}

\input{2-ode_likelihood}
\input{3-high_order_SM}
\input{5-implementation}
\input{4-related}
\input{6-experiments}

\vspace{-.1cm}
\section{Conclusion}
\vspace{-.1cm}

We propose maximum likelihood training for {\ode}s by novel high-order denoising score matching methods. We analyze the relationship between the score matching objectives and the KL divergence from the data distribution to the {\ode} distribution. Based on it, we provide an upper bound of the KL divergence, which can be controlled by minimizing the first, second, and third-order score matching errors of score models. To minimize the high-order score matching errors, we further propose a high-order DSM algorithm, such that the higher-order score matching error can be bounded by exactly the training error and the lower-order score matching errors. The optimal solution for the score model is still the same as the original training objective of SGMs. Empirically, our method can greatly improve the model density of {\ode}s of the Variance Exploding type on several density modeling benchmarks. Finally, we believe that our training method is also suitable for other SGMs, including the Variance Preserving (VP) type~\cite{song2020score}, the latent space type~\cite{vahdat2021score} and the critically-damped Langevin diffusion type~\cite{dockhorn2021score}. Such extensions are left for future work.

\section*{Acknowledgements}

This work was supported by National Key Research and Development Project of China (No. 2021ZD0110502); NSF of China Projects (Nos. 62061136001, 61620106010, 62076145, U19B2034, U1811461, U19A2081, 6197222, 62106120); Beijing NSF Project (No. JQ19016); Beijing Outstanding Young Scientist Program NO. BJJWZYJH012019100020098; a grant from Tsinghua Institute for Guo Qiang; the NVIDIA NVAIL Program with GPU/DGX Acceleration; the High Performance Computing Center, Tsinghua University; and Major Innovation \& Planning Interdisciplinary Platform for the ``Double-First Class" Initiative, Renmin University of China.

\bibliography{ref}
\bibliographystyle{icml2022}

\newpage
\appendix
\onecolumn
\input{0-appendix}

\end{document}

%% file: common.tex
\usepackage{color}
\usepackage{bm}
\usepackage{hyperref}
\usepackage{amsmath,amsfonts}
\usepackage{blindtext}
\usepackage{listings}
\usepackage{mathtools}

\newif\ifdebug
\debugtrue

\newcommand{\sde}{ScoreSDE}
\newcommand{\ode}{ScoreODE}

\newcommand{\FISH}{\text{Fisher}}
\newcommand{\DSM}{\text{DSM}}
\newcommand{\SM}{\text{SM}}
\newcommand{\DIFF}{\text{Diff}}
\newcommand{\SDE}{\text{SDE}}
\newcommand{\ODE}{\text{ODE}}

\DeclarePairedDelimiterX{\infdivx}[2]{(}{)}{%
	#1\;\delimsize\|\;#2%
}

\newcommand{\kl}{D_{\mathrm{KL}}\infdivx}
\newcommand{\fisher}{D_{\mathrm{F}}\infdivx}

\newcommand{\vect}[1]{\bm{#1}}

\newcommand{\argmin}{\operatornamewithlimits{argmin}}

\newcommand{\x}{\xv}
\newcommand{\y}{\yv}
\newcommand{\z}{\zv}
\newcommand{\s}{\sv}

\newcommand{\dm}{\mathrm{d}}

\newcommand{\E}{\mathbb{E}}
\newcommand{\R}{\mathbb{R}}
\newcommand{\N}{\mathcal{N}}

\newcommand{\dxv}{\mathrm{d}\xv}

\newcommand{\dzv}{\mathrm{d}\zv}
\newcommand{\dwv}{\mathrm{d}\wv}
\newcommand{\dt}{\mathrm{d}t}

\newcommand{\epsilonv}{\vect\epsilon}

\newcommand{\av}{\vect a}
\newcommand{\bv}{\vect b}

\newcommand{\fv}{\vect f}

\newcommand{\hv}{\vect h}

\newcommand{\ellv}{\vect \ell}

\newcommand{\sv}{\vect s}

\newcommand{\vv}{\vect v}
\newcommand{\wv}{\vect w}
\newcommand{\xv}{\vect x}
\newcommand{\yv}{\vect y}
\newcommand{\zv}{\vect z}

\newcommand{\Iv}{\vect I}

\newcommand{\Jc}{\mathcal J}

\newcommand{\Nc}{\mathcal N}
\newcommand{\Oc}{\mathcal O}

\newcommand{\Uc}{\mathcal U}

\newcommand{\Cov}{\mbox{Cov}}
\newcommand{\diag}{\mbox{diag}}

%% file: 1-intro.tex
\section{Introduction}
\begin{figure}[t]
	\centering
	\begin{minipage}{0.32\linewidth}
		\centering
			\includegraphics[width=\linewidth]{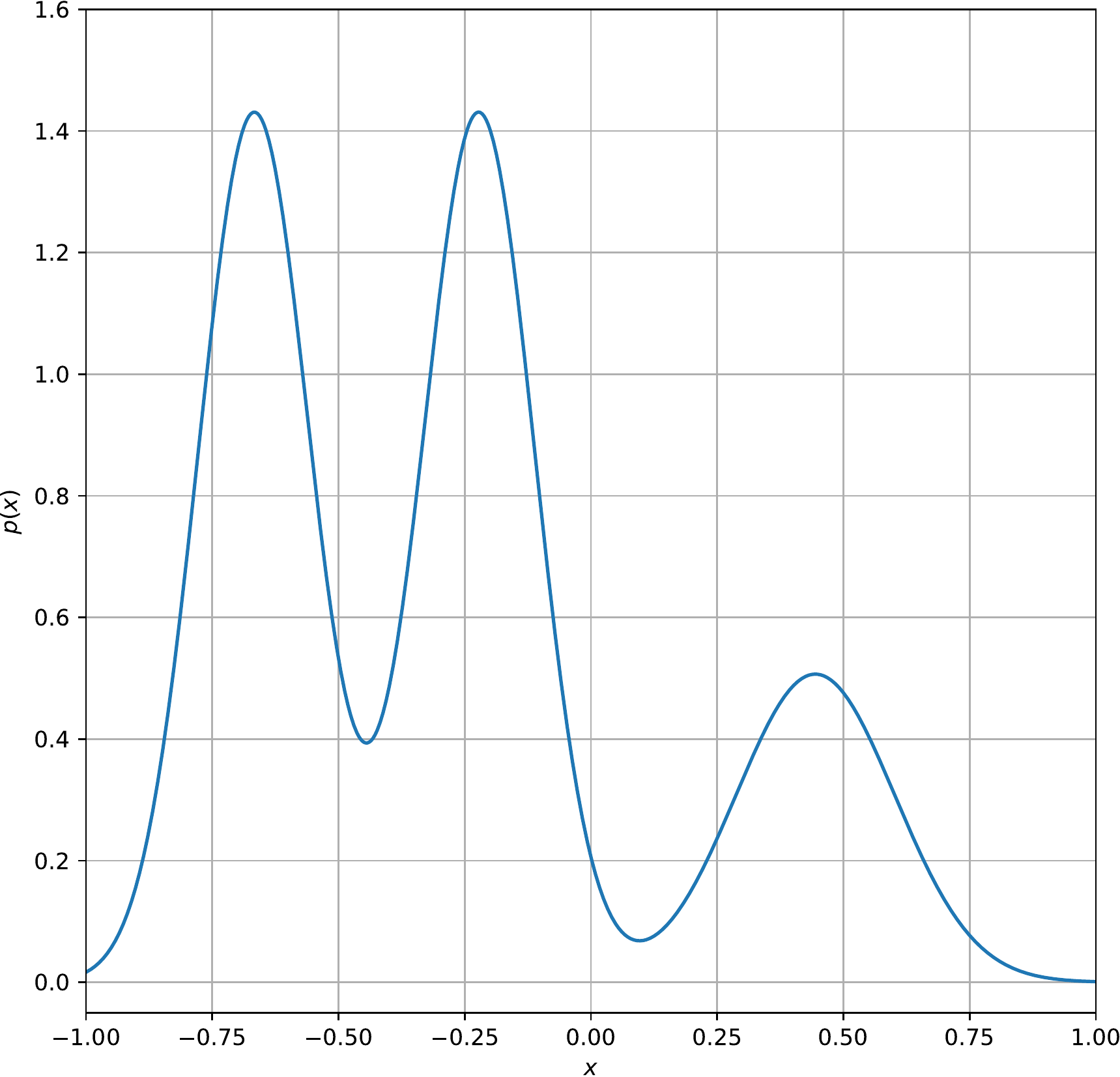}\\
\small{(a) Data \\ \ }
	\end{minipage}
	\begin{minipage}{0.32\linewidth}
		\centering		
 	\includegraphics[width=\linewidth]{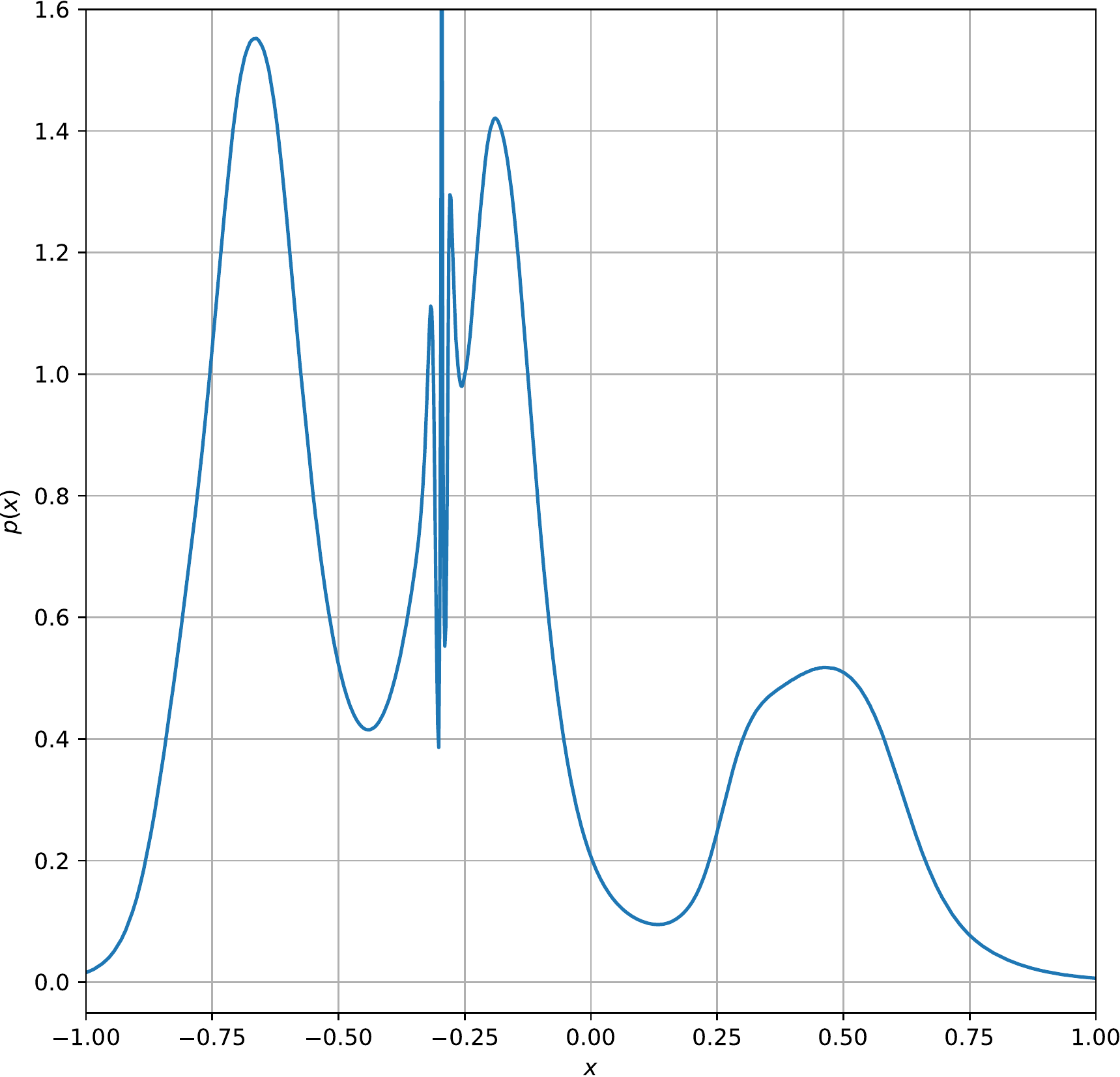}\\
\small{(b) {\ode}\\ (first-order SM) %
}
   \end{minipage}
\begin{minipage}{0.32\linewidth}
		\centering		
 	\includegraphics[width=\linewidth]{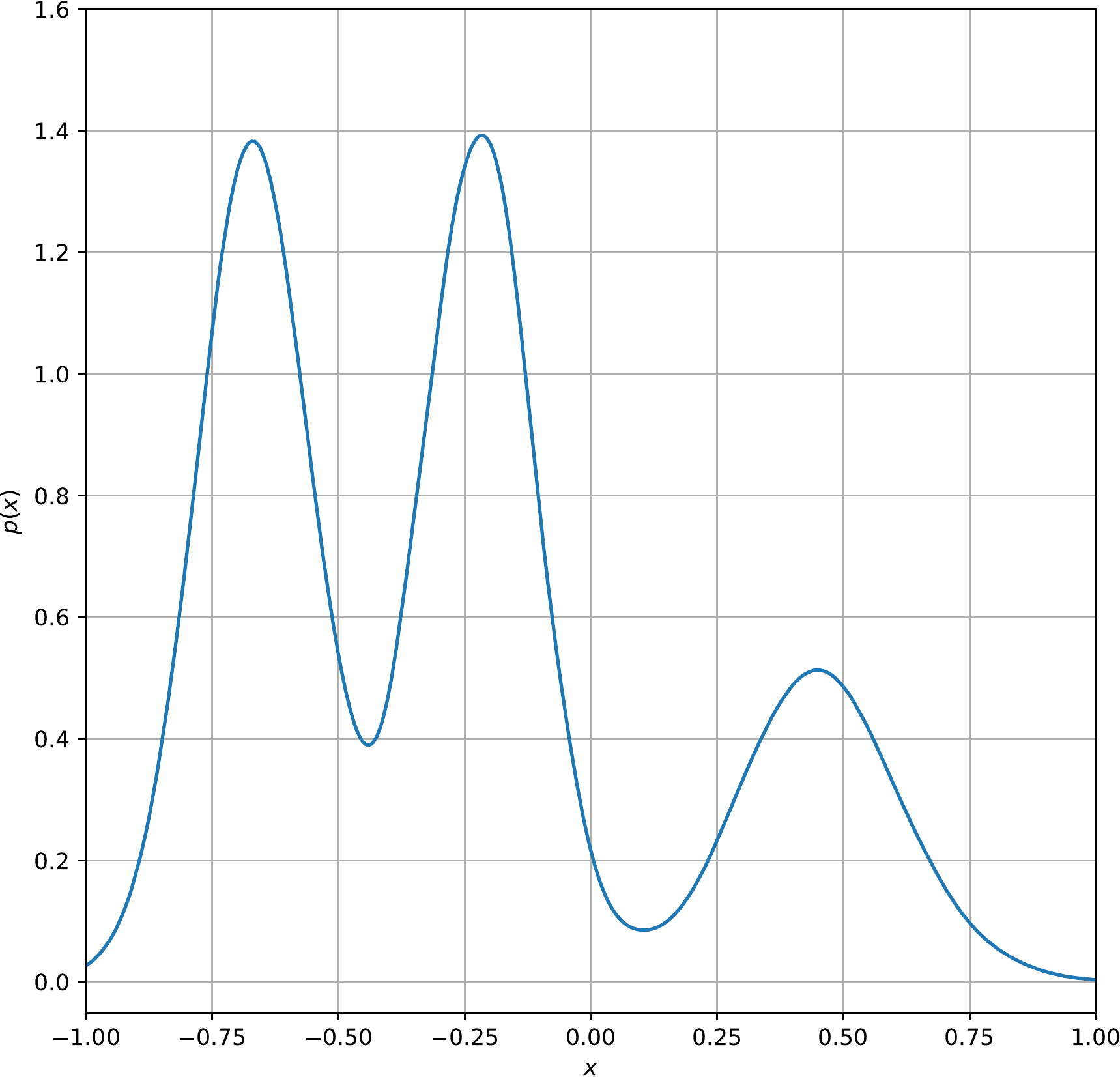}\\
\small{(c) {\ode}\\ (third-order SM) %
}
   \end{minipage}
   \vspace{-.05in}
	\caption{\label{fig:mog_intro}
An 1-D mixture-of-Gaussians data distribution (a) and \ode 's density (b,c). (b): {\ode} trained by minimizing the first-order score matching (SM) objective in~\cite{song2021maximum}, which fails to estimate the density around $x=-0.25$. (c): {\ode} trained by minimizing our proposed third-order SM objective, which approximates the data density well.}
\vspace{-.15in}
\end{figure}

Score-based generative models (SGMs) have shown promising %
performance in both sample quality and likelihood evaluation~\cite{song2020score,vahdat2021score,dockhorn2021score,kingma2021variational}, with applications %
on many tasks, such as image generation~\cite{dhariwal2021diffusion,meng2021sdedit}, speech synthesis~\cite{chen2020wavegrad,kong2020diffwave}, shape generation~\cite{luo2021diffusion,zhou20213d} and lossless compression~\cite{kingma2021variational}. SGMs define a forward diffusion process by a stochastic differential equation (SDE) that gradually perturbs the data distribution to a simple noise distribution. The forward diffusion process has an equivalent reverse-time process with analytical form~\cite{anderson1982reverse}, which depends on the \textit{score function} (the gradient of the log probability density) of the forward process at each time step. SGMs learn a parameterized neural network (called ``score model'')~\cite{song2020sliced,song2019generative,song2020improved,song2020score} to estimate the score functions, and further define probabilistic models by the score model.

Two types of probabilistic models are proposed with the learned score model. One is the score-based diffusion SDE (\citet{song2020score}, ``\sde'' for short), which is defined by approximately reversing the diffusion process from the noise distribution by the score model and it can generate high-quality samples. The other is the score-based diffusion ordinary differential equation (\citet{song2020score, song2021maximum}, ``\ode'' for short), which is defined by approximating the probability flow ODE~\cite{song2020score} of the forward process whose marginal distribution is the same as the forward process at each time. 
\ode s can be viewed as continuous normalizing flows~\cite{chen2018neural}. Thus, unlikely SDEs, they can compute exact likelihood by ODE solvers~\cite{grathwohl2018ffjord}.

In practice, the score models are trained by minimizing an objective of a time-weighted combination of score matching losses~\cite{hyvarinen2005estimation,vincent2011connection}. It is proven that for a specific weighting function, minimizing the score matching objective is equivalent to maximizing the likelihood of the {\sde}~\cite{song2021maximum}.
However, minimizing the score matching objective does not necessarily maximize the likelihood of the \ode.
This is problematic since \ode s are used for likelihood estimation. 
In fact, \ode s trained with the score matching objective may even fail to estimate the likelihood of very simple data distributions, as shown in Fig.~\ref{fig:mog_intro}.

In this paper, we systematically analyze the relationship between the score matching objective and the maximum likelihood objective of {\ode}s, and present a new algorithm to learn \ode s via maximizing likelihood. Theoretically, we derive an equivalent form of the KL divergence\footnote{Maximum likelihood estimation is equivalent to minimizing a KL divergence.} between the data distribution and the \ode~distribution, which explicitly reveals the gap between the KL divergence and the original first-order score matching objective. 
Computationally, as directly optimizing the equivalent form requires an ODE solver, which is time-consuming~\cite{chen2018neural}, we derive an upper bound of the KL divergence by controlling the approximation errors of the first, second, and third-order score matching. The optimal solution for this score model by minimizing the upper bound is still the data score function, so the learned score model can still be used for the sample methods in SGMs (such as PC samplers in~\cite{song2020score}) to generate high-quality samples.

Based on the analyses, we propose a novel high-order denoising score matching algorithm to train the score models, which theoretically guarantees bounded approximation errors of high-order score matching. We further propose some scale-up techniques for practice. Our experimental results empirically show that the proposed training method can improve the likelihood of \ode s on both synthetic data and CIFAR-10, while retaining the high generation quality.

\section{Score-Based Generative Models}
\label{sec:background}
We first review the dynamics and probability distributions defined by SGMs. The relationship between these dynamics is illustrated in Fig.~\ref{fig:relationship}(a). 
\subsection{\sde s and Maximum Likelihood Training}
Let $q_0(\x_0)$ denote an unknown $d$-dimensional data distribution. SGMs 
define an SDE from time $0$ to $T$ (called ``forward process") with $\x_0\sim q_0(\x_0)$ and
\begin{equation}
\label{eqn:forward_sde}
    \dxv_t=\fv(\x_t,t)\dt+g(t)\dwv_t,
\end{equation}
where $\x_t$ is the solution at time $t$, $\fv(\cdot,t):\R^d\rightarrow\R^d$ and $g(t)\in\R$ are fixed functions such that the marginal distribution of $\x_T$ is approximately $\N(\x_T|\vect{0},\sigma_T^2\Iv)$ with $\sigma_T>0$, and $\wv_t\in\R^d$ is the standard Wiener process. Let $q_t(\x_t)$ denote the marginal distribution of $\x_t$. Under some regularity conditions, the forward process in Eqn.~\eqref{eqn:forward_sde} has an equivalent reverse process~\cite{anderson1982reverse} from time $T$ to $0$ with $\x_T\sim q_T(\x_T)$:
\begin{equation}
\label{eqn:reverse_sde_q}
    \dxv_t=[\fv(\x_t,t)-g(t)^2\nabla_{\x}\log q_t(\x_t)]\dt+g(t)\dm \bar{\wv}_t,
\end{equation}
where $\bar{\wv}_t$ is a standard Wiener process in the reverse-time. The marginal distribution of $\x_t$ in the reverse process in Eqn.~\eqref{eqn:reverse_sde_q} is also $q_t(\x_t)$. The only unknown term in Eqn.~\eqref{eqn:reverse_sde_q} is the time-dependent score function $\nabla_{\x}\log q_t(\x_t)$. SGMs use a neural network $\s_\theta(\x_t,t)$ parameterized by $\theta$ to estimate the score function for all $\x_t\in\R^d$ and $t\in[0,T]$. The neural network is trained by matching $\s_\theta(\x_t,t)$ and $\nabla_{\x}\log q_t(\x_t)$ with
the \emph{score matching objective}:
\begin{equation*}
\begin{aligned}
    &\Jc_{\SM}(\theta;\lambda(\cdot))\\
    &\coloneqq  \frac{1}{2}\int_0^T\lambda(t)\E_{q_t(\x_t)}\Big[\|\s_\theta(\x_t,t)-\nabla_{\x}\log q_t(\x_t)\|_2^2\Big]\dt,
\end{aligned}
\end{equation*}
where $\lambda(\cdot)>0$ is a weighting function. By approximating the score function with $\s_\theta(\x_t,t)$, SGMs define a parameterized reverse SDE from time $T$ to time $0$ with probability at each time $t$ denoted as $p^{\SDE}_t(\x_t)$ (omitting the subscript $\theta$). In particular, suppose $\x_T\sim p^{\SDE}_T(\x_T)=\Nc(\x_T|\vect{0},\sigma^2_T\Iv)$, the parameterized reverse SDE is defined by
\begin{equation}
\label{eqn:reverse_sde_p}
    \dxv_t=[\fv(\x_t,t)-g(t)^2\s_\theta(\x_t,t)]\dt+g(t)\dm \bar{\wv}_t,
\end{equation}
where the marginal distribution of $\x_t$ is $p^{\SDE}_t(\x_t)$. By solving the parameterized reverse SDE, SGMs achieve excellent sample quality in many tasks~\cite{song2020score,dhariwal2021diffusion}. Moreover, the KL divergence between $q_0(\x_0)$ and $p_0^{\SDE}(\x_0)$ can be bounded by the score matching error with the weight $g(\cdot)^2$~\cite{song2021maximum}:
\begin{equation}
\label{eqn:sde_elbo}
    \kl{q_0}{p_0^{\SDE}}\leq \kl{q_T}{p_T^{\SDE}}+\Jc_{\SM}(\theta;g(\cdot)^2).
\end{equation}
Therefore, minimizing $\Jc_{\SM}(\theta;g(\cdot)^2)$ is equivalent to maximum likelihood training of $p^{\SDE}_0$. In the rest of this work, we mainly consider the maximum likelihood perspective, namely $\lambda(\cdot)=g(\cdot)^2$, and denote $\Jc_{\SM}(\theta)\coloneqq\Jc_{\SM}(\theta;g(\cdot)^2)$.

\subsection{\ode s and Exact Likelihood Evaluation}
The parameterized \sde~cannot be used for exact likelihood evaluation, because we cannot evaluate $p_0^{\SDE}(\x_0)$ exactly for a given point $\x_0\in\R^d$. However, every SDE has an associated \textit{probability flow ODE}~\cite{song2020score} whose marginal distribution at each time $t$ is the same as that of the SDE. Unlike SDEs, the log-density of the ODE can be exactly computed by the ``Instantaneous Change of Variables"~\cite{chen2018neural}. Particularly, the probability flow ODE of the forward process in Eqn.~\eqref{eqn:forward_sde} is
\begin{equation}
\label{eqn:forward_ode}
    \frac{\dxv_t}{\dt}=\hv_q(\x_t,t)\coloneqq\fv(\x_t,t)-\frac{1}{2}g(t)^2\nabla_{\x}\log q_t(\x_t).
\end{equation}
If we start from $\x_0\sim q_0(\x_0)$, the distribution of each $\x_t$ during the ODE trajectory is also $q_t(\x_t)$.
Note that the only unknown term in Eqn.~\eqref{eqn:forward_ode} is also the score function $\nabla_{\x}\log q_t(\x_t)$. By approximating the score function with $\s_\theta(\x_t,t)$, SGMs define a parameterized ODE called ``score-based diffusion ODE'' (\ode) with probability $p^{\ODE}_t(\x_t)$ (omitting the subscript $\theta$) at each time $t$. In particular, suppose $\x_T\sim p_T^{\ODE}(\x_T)=p_T^{\SDE}(\x_T)=\N(\x_T|\vect{0},\sigma_T^2\Iv)$, the {\ode} is defined by
\begin{equation}
\label{eqn:score_ode}
    \frac{\dxv_t}{\dt}=\hv_p(\x_t,t)\coloneqq\fv(\x_t,t)-\frac{1}{2}g(t)^2\s_\theta(\x_t,t),
\end{equation}
where the marginal distribution of $\x_t$ is $p_t^{\ODE}(\x_t)$. By the ``Instantaneous Change of Variables"~\cite{chen2018neural}, $\log p_t^{\ODE}(\x_t)$ of $\x_t$ can be computed by integrating:
\begin{equation}
\label{eqn:change_of_variable_score_ode}
    \frac{\dm \log p^{\ODE}_t(\x_t)}{\dt}=-\tr(\nabla_{\x}\hv_p(\x_t,t)).
\end{equation}
If $\s_\theta(\x_t,t)\equiv\nabla_{\x}\log q_t(\x_t,t)$ for all $\x_t\in\R^d$ and $t\in[0,T]$, we have $\hv_p(\x_t,t)=\hv_q(\x_t,t)$. In this case, for all $\x_0\in\R^d$, $\log p_0^{\ODE}(\x_0)-\log q_0(\x_0)=\log p_T^{\ODE}(\x_T)-\log q_T(\x_T)\approx 0$ (because $q_T\approx p^{\ODE}_T$), where $\x_T$ is the point at time $T$ given by Eqn.~\eqref{eqn:forward_ode} starting with $\x_0$ at time $0$. Inspired by this, SGMs use the {\ode} with the score model for exact likelihood evaluations.

However, if $\s_\theta(\x_t,t)\neq\nabla_{\x}\log q_t(\x_t,t)$, there is no theoretical guarantee for bounding $\kl{q_0}{p_0^{\ODE}}$ of {\ode}s by the score matching objective $\Jc_{\SM}(\theta)$. As a result, after minimizing $\Jc_{\SM}(\theta)$, \citet{song2021maximum} find that the {\sde} of the ``Variance Exploding'' type~\cite{song2020score} achieves state-of-the-art sample quality, but the log-likelihood of the corresponding {\ode} is much worse (e.g., $\approx3.45$ bits/dim on CIFAR-10 dataset) than that of other comparable models (e.g., $\approx3.13$ bits/dim of the "Variance Preserving" type~\cite{song2020score}). Such observations pose a serious concern on the behavior of the likelihood of $p^{\ODE}_0$, which is not well understood. We aim to provide a systematical analysis in this work.

%% file: 2-ode_likelihood.tex
\section{Relationship between Score Matching and KL Divergence of \ode s}
\label{sec:ode_likelihood}

\begin{figure*}[t]
\centering
	\begin{minipage}{0.53\linewidth}
\begin{tikzpicture}[node distance=1.0cm]
    \draw [decorate, decoration = {brace,
            raise=5pt,
            amplitude=5pt,
            aspect=0.5},thick] (0, -2) --  (0,0)
    node(q)[pos=0.5,left=10pt,black]{$q_t(\x_t)$};
    \node(fs)[smalltext, right of=q, yshift=1.0cm, anchor=west, xshift=-0.2cm] {\scriptsize Forward SDE (Eqn.\eqref{eqn:forward_sde})};
    \node(rs)[smalltext, right of=q, anchor=west, xshift=-0.2cm] {\scriptsize Reverse SDE (Eqn.\eqref{eqn:reverse_sde_q})};
    \node(qode)[smalltext, right of=q, yshift=-1.0cm, anchor=west, xshift=-0.2cm] {\scriptsize Probability flow ODE (Eqn.\eqref{eqn:forward_ode})};
    \node(psde)[right of=rs, anchor=west, xshift=3.5cm]{$p_t^{\SDE}(\x_t)$};
    \node(pode)[right of=qode, xshift=3.cm, anchor=west]{$p_t^{\ODE}(\x_t)$};
    \draw [arrow] (rs) -> (psde);
    \node(Def)[smalltext, below of=psde, xshift=-2.35cm, yshift=0.5cm] {\scriptsize approximate $\nabla_{\x}\log q_t$ by $\s_\theta$};
    \draw [arrow] (qode) -> (pode);
    \node(Def)[smalltext, above of=psde, xshift=-2.3cm, yshift=-0.7cm] {\scriptsize Eqn.~\eqref{eqn:reverse_sde_p}};
    \node(Def)[smalltext, below of=pode, xshift=-1.85cm, yshift=0.7cm] {\scriptsize Eqn.~\eqref{eqn:score_ode}};
    \node(pneq)[rotate=90, below of=psde, xshift=-0.5cm, yshift=1.0cm]{$\neq$};
    \node(Def)[smalltext, below of=psde, xshift=0.9cm, yshift=0.5cm] {\scriptsize (Appendix.~\ref{appendix:distribution_gap})};
\end{tikzpicture}\\
	\centering
    \small{(a) Relationship between $q_t$, $p_t^{\SDE}$ and $p_t^{\ODE}$.}
	\end{minipage}
		\begin{minipage}{0.45\linewidth}
		\centering
        \begin{tikzpicture}[node distance=1.0cm]
        \node(klsde) {$\kl{q_0}{p_0^{\SDE}}$};
        \node(klode)[right of=klsde, xshift=2.1cm] {$\kl{q_0}{p_0^{\ODE}}$};
        \node(sdeup)[below of=klsde, yshift=-0.4cm]{$\Jc_{\SM}(\theta)$};
        \draw [dasharrow] (klsde.south) -> (sdeup.north);
        \node(odeup)[below of=klode, yshift=-0.4cm]{$\sqrt{\Jc_{\SM}(\theta)}\cdot\sqrt{\Jc_{\FISH}(\theta)}$};
        \draw [arrow] (klode.south) -> (odeup.north);
        \node(sdesm)[below of=sdeup, yshift=-0.2cm]{\scriptsize first-order SM};
        \node(odesm1)[below of=odeup, xshift=-0.9cm, yshift=-0.2cm]{\scriptsize first-order SM};
        \node(odesm2)[below of=odeup, xshift=-0.1cm, anchor=west, yshift=-0.2cm]{\scriptsize second, third-order SM};
        \node(thrm)[below of=odesm2, yshift=0.6cm]{\scriptsize (Theorem~\ref{thrm:second}, \ref{coro:second_trace}, \ref{thrm:third})};
        \draw [dasharrow] (sdeup.south) -> (sdesm.north);
        \draw [dasharrow] (odesm1.north |- odeup.south) -> (odesm1.north);
        \draw [arrow] (odesm2.north |- odeup.south) -> (odesm2.north);
        \node(Def)[below of=klsde, xshift=-0.1cm, yshift=0.35cm, anchor=east]{\scriptsize \cite{song2021maximum}};
        \node(Def)[below of=klode, xshift=0.1cm, anchor=west, yshift=0.35cm]{\scriptsize (Theorem~\ref{thrm:kl-ode} \& Eqn.~\eqref{eqn:cauchy_ode})};
        \node(Def)[below of=klode, xshift=-1.55cm, yshift=0.35cm] {\scriptsize $\leftarrow$minimizing upper bound$\rightarrow$};
        \node(Def)[above of=odesm2, yshift=-0.4cm, xshift=0.8cm]{\scriptsize (Theorem~\ref{thrm:fisher_bound})};
        \end{tikzpicture}\\
        \small{(b) Relationship between SM objectives and KL divergence.}
		\end{minipage}
\centering
\caption{An illustration of the theoretical analysis of score-based generative models and score matching (SM) objectives.\label{fig:relationship}}
\vspace{-.1in}
\end{figure*}
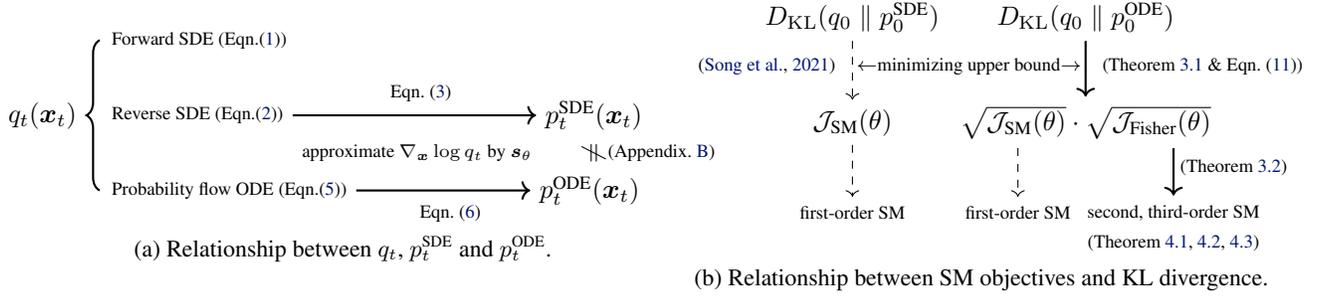

The likelihood evaluation needs the distribution $p_0^{\ODE}$ of {\ode}s, which is usually different from the distribution $p^{\SDE}_0$ of {\sde}s (see Appendix~\ref{appendix:distribution_gap} for detailed discussions). Although the score matching objective $\Jc_{\SM}(\theta)$ can upper bound the KL divergence of {\sde}s (up to constants) according to Eqn.~\eqref{eqn:sde_elbo}, it is not 
enough to upper bound the KL divergence of {\ode}s. To the best of our knowledge, there is no theoretical analysis of the relationship between the score matching objectives and the KL divergence of {\ode}s.

In this section, we reveal the relationship between the first-order score matching objective $\Jc_{\SM}(\theta)$ and the KL divergence $\kl{q_0}{p_0^{\ODE}}$ of {\ode}s, and upper bound $\kl{q_0}{p_0^{\ODE}}$ by involving high-order score matching. Fig.~\ref{fig:relationship} demonstrates our theoretical analysis and the connection with previous works.

\subsection{KL Divergence of \ode s}

In the following theorem, we propose a formulation of the KL divergence $\kl{q_0}{p_0^{\ODE}}$ by integration, which shows that there is a non-zero gap between the KL divergence and the score matching objective $\Jc_{\SM}(\theta)$.%

\begin{theorem}[Proof in Appendix~\ref{appendix:thrm_kl_ode}]
\label{thrm:kl-ode}
Let $q_0$ be the data distribution and $q_t$ be the marginal distribution at time $t$ through the forward diffusion process defined in Eqn.~\eqref{eqn:forward_sde}, and $p^{\ODE}_t$ be the marginal distribution at time $t$ through the {\ode} defined in Eqn.~\eqref{eqn:score_ode}. Under some regularity conditions in Appendix~\ref{appendix:assumptions}, we have
\begin{equation*}
\begin{medsize}
\begin{aligned}
    \nonumber
    \kl{q_0}{p^{\ODE}_0}\!&=\!\kl{q_T}{p^{\ODE}_T}+\Jc_{\ODE}(\theta)\\
    &=\!\!\!\!\!\!\underbrace{\kl{q_T}{p^{\ODE}_T}+\Jc_{\SM}(\theta)}_{\text{upper bound of $\kl{q_0}{p_0^{\SDE}}$ in Eqn.~\eqref{eqn:sde_elbo}}}\!\!\!+\ \Jc_{\DIFF}(\theta),
\end{aligned}
\end{medsize}
\end{equation*}
where
\begin{small}
\begin{align}
    \nonumber
    \Jc_{\ODE}(\theta)\!&\coloneqq\!\frac{1}{2}\!\int_0^T \!\!\!g(t)^2\E_{q_t(\x_t)}\!\Big[(\s_\theta(\x_t,t)\!-\!\nabla_{\x}\log q_t(\x_t))^\top\\ 
    &(\nabla_{\x}\log p_t^{\ODE}(\x_t)\!-\!\nabla_{\x}\log q_t(\x_t))\Big]\dt,
    \label{eqn:kl_ode}\\
    \nonumber
    \Jc_{\DIFF}(\theta)\!&\coloneqq\!\frac{1}{2}\int_0^T \!\!\!g(t)^2\E_{q_t(\x_t)}\!\Big[(\s_\theta(\x_t,t)\!-\!\nabla_{\x}\log q_t(\x_t))^\top\\
    &(\nabla_{\x}\log p_t^{\ODE}(\x_t)\!-\!\s_\theta(\x_t,t))\Big]\dt.
    \label{eqn:J_diff}
\end{align}
\end{small}%
\end{theorem}
As $p^{\ODE}_T$ is fixed, minimizing $\kl{q_0}{p^{\ODE}_0}$ is equivalent to minimizing $\Jc_{\ODE}(\theta)$, which includes $\Jc_{\SM}(\theta)$ and $\Jc_{\DIFF}(\theta)$. While minimizing $\Jc_{\SM}(\theta)$ reduces the difference between $\s_\theta(\cdot,t)$ and the data score $\nabla_{\x}\log q_t$, it cannot control the error of $\Jc_{\DIFF}(\theta)$, which includes the difference between $\s_\theta(\cdot,t)$ and the model score $\nabla_{\x}\log p_t^{\ODE}$. As we show in Fig.~\ref{fig:mog_intro}, only minimizing $\Jc_{\SM}(\theta)$ may cause problematic model density of the {\ode}.

\subsection{Bounding the KL Divergence of \ode s by High-Order Score Matching Errors}
\label{sec:kl_bounding}
The KL divergence in Theorem~\ref{thrm:kl-ode} includes the model score function $\nabla_{\x}\log p^{\ODE}_t(\x_t)$, which needs ODE solvers to compute (see Appendix.~\ref{appendix:change_of_score} for details) and thus is hard to scale up. Below we derive an upper bound of $\kl{q_0}{p_0^{\ODE}}$ by involving high-order score matching, which can avoid the computation of $\nabla_{\x}\log p^{\ODE}_t(\x_t)$. Specifically, we show that by bounding the first, second and third-order score matching errors for $\s_\theta, \nabla_{\x}\s_\theta$ and $\nabla_{\x}\tr(\nabla_{\x}\s_\theta)$ with the corresponding order score functions\footnote{In this work, we refer to $\nabla_{\x}^2\log q_t(\x_t)$ as the second-order score function of $q_t$, and $\nabla_{\x}\tr(\nabla_{\x}^2\log q_t(\x_t))$ as the third-order score function of $q_t$.} of $q_t$, we can further upper bound $\kl{q_0}{p_0^{\ODE}}$.

Let $\fisher{q}{p}\coloneqq \E_{q(\x)}\|\nabla_{\x}\log p(\x)-\nabla_{\x}\log q(\x)\|_2^2$ denote the Fisher divergence between two distributions $q$ and $p$. Define
\begin{equation}
    \Jc_{\FISH}(\theta)\coloneqq \frac{1}{2}\int_0^T g(t)^2\fisher{q_t}{p^{\ODE}_t}\dt.
\end{equation}
By straightforwardly using Cauchy-Schwarz inequality for Eqn.~\eqref{eqn:kl_ode}, we have
\begin{equation}
\label{eqn:cauchy_ode}
    \Jc_{\ODE}(\theta)\leq \sqrt{\Jc_{\SM}(\theta)}\cdot\sqrt{\Jc_{\FISH}(\theta)}.
\end{equation}%
The upper bound shows that we can minimize the KL divergence $\kl{q_0}{p_0^{\ODE}}$ of {\ode} by minimizing $\Jc_{\SM}(\theta)$ and $\Jc_{\FISH}(\theta)$ simultaneously. %
When $\s_\theta(\x_t,t)\equiv\nabla_{\x}\log q_t(\x_t)$ for all $\x_t\in\R^d$ and $t\in[0,T]$, the equality holds and $\Jc_{\ODE}(\theta)=\Jc_{\SM}(\theta)=0$.

We can regard $\Jc_{\FISH}(\theta)$ as a weighted Fisher divergence between $q_t$ and $p_t^{\ODE}$ during $t\in[0,T]$. Directly computing the Fisher divergence also needs ODE solvers (see Appendix.\ref{appendix:change_of_score} for details) and is hard to scale up. To address this problem, we present the following theorem for bounding $\fisher{q_t}{p^{\ODE}_t}$, which introduces high-order score matching and avoids the computation costs of ODE solvers.

\begin{theorem}[Proof in Appendix~\ref{appendix:thrm_fisher_bound}]
\label{thrm:fisher_bound}
Assume that there exists $C>0$, such that for all $t\in [0,T]$ and $\x_t\in\R^d$, $\|\nabla_{\x}^2\log p_t^{\ODE}(\x_t)\|_2\leq C$ (where $\nabla^2$ denotes the second-order derivative (Hessian), and $\|\cdot\|_2$ is the spectral norm of a matrix). And assume that there exist $\delta_1,\delta_2,\delta_3>0$, such that for all $\x_t\in\R^d$, $t\in[0,T]$, $s_\theta(\x_t,t)$ satisfies

\begin{small}
\begin{equation*}
\begin{aligned}
    &\|\s_\theta(\x_t,t)-\nabla_{\x}\log q_t(\x_t)\|_2\leq\delta_1,\\
    &\|\nabla_{\x}\s_\theta(\x_t,t)-\nabla^2_{\x}\log q_t(\x_t)\|_F\leq\delta_2,\\
    &\|\nabla_{\x}\tr(\nabla_{\x}\s_\theta(\x_t,t))-\nabla_{\x}\tr(\nabla_{\x}^2\log q_t(\x_t))\|_2\leq\delta_3,
\end{aligned}
\end{equation*}
\end{small}%
where $\|\cdot\|_F$ is the Frobenius norm of matrices. Then there exists $U(t;\delta_1,\delta_2,\delta_3,C,q)\geq 0$ which depends on the forward process $q$ while is independent of $\theta$, such that $\fisher{q_t}{p^{\ODE}_t}\leq U(t;\delta_1,\delta_2,\delta_3,C,q)$. Moreover, for all $t\in[0,T]$, if $g(t)\neq 0$, then $U(t;\delta_1,\delta_2,\delta_3,C,q)$ is a strictly increasing function of $\delta_1$, $\delta_2$ and $\delta_3$, respectively.%

\end{theorem}
Theorem~\ref{thrm:fisher_bound} shows that by bounding the errors between $\s_\theta$, $\nabla_{\x}\s_\theta$ and $\nabla_{\x}\tr(\nabla_{\x}\s_\theta)$ with the first, second and third-order score functions of $q_t$, we can upper bound $\fisher{q_t}{p_t^{\ODE}}$ and then upper bound $\Jc_{\FISH}(\theta)$. As $\Jc_{\SM}(\theta)$ can also be upper bounded by the first-order score matching error of $\s_\theta$, we can further upper bound $\kl{q_0}{p_0^{\ODE}}$ by the first, second and third-order score matching errors.

Note that the high-order score matching for $\s_\theta$ is compatible for the traditional SGMs, because the optimal solution for $\s_\theta(\x_t,t)$ in the non-parametric limit is still $\nabla_{\x}\log q_t(\x_t)$. So after minimizing $\kl{q_0}{p_0^{\ODE}}$ by high-order score matching, we can still use $\s_\theta$ for the sample methods in SGMs, while directly minimizing $\kl{q_0}{p_0^{\ODE}}$ by exact likelihood ODE solvers cannot ensure that (see Appendix.~\ref{appendix:mle_scoreode_by_odesolver} for detailed reasons).

%% file: 3-high_order_SM.tex
\section{Error-Bounded High-Order Denoising Score Matching (DSM)}
\label{sec:high_order_SM}

The analysis in Sec.~\ref{sec:kl_bounding} shows that we can upper bound the KL divergence $\kl{q_0}{p_0^{\ODE}}$ of the {\ode} by controlling the first, second and third-order score matching errors of the score model $\s_\theta$ (and its corresponding derivatives) at each time $t\in[0,T]$. Inspired by this, we generalize the traditional denoising score matching (DSM)~\cite{vincent2011connection} to second and third orders to minimize the high-order score matching errors. In this section, we propose a novel high-order DSM method, such that the higher-order score matching error can be exactly bounded by the training error and the lower-order score matching error. All the proofs in this section are in Appendix.~\ref{appendix:proof_high_order_DSM}.

Without loss of generality, in this section, we focus on a fixed time $t\in(0,T]$, and propose the error-bounded high-order DSM method for matching the forward process distribution $q_t$ at the fixed time $t$. The whole training objective for the entire process by $t\in[0,T]$ is detailed in Sec.~\ref{sec:training}. Moreover, for SGMs, there commonly exist $\alpha_t\in\R$ and $\sigma_t\in\R_{>0}$, such that the forward process satisfies $q(\x_t|\x_0)=\N(\x_t|\alpha_t\x_0,\sigma_t^2\Iv)$ for $\x_0,\x_t\sim q_{t0}(\x_t,\x_0)$~\cite{song2020score}. Therefore, we always assume $q(\x_t|\x_0)=\N(\x_t|\alpha_t\x_0,\sigma_t^2\Iv)$  in the rest of this work.

\subsection{First-Order Denoising Score Matching}
We present some preliminaries on first-order denoising score matching~\cite{vincent2011connection,song2020score} in this section.

The first-order score matching objective $\Jc_{\SM}(\theta)$ needs the data score functions $\nabla_{\x}\log q_t(\x_t)$ for all $t\in[0,T]$, which are unknown in practice. To address this issue, SGMs leverage the denoising score matching method~\cite{vincent2011connection} to train the score models. For a fixed time $t$, one can learn a first-order score model $\s_1(\cdot,t;\theta):\R^d\rightarrow\R^d$ parameterized by $\theta$ which minimizes
\begin{equation*}
    \E_{q_t(\x_t)}\Big[\big\|\s_1(\x_t,t;\theta)-\nabla_{\x}\log q_t(\x_t)\big\|_2^2\Big],
\end{equation*}
by optimizing the (first-order) DSM objective:
\begin{equation}
\label{eqn:dsm-1-obj}
    \theta^*=\argmin_{\theta}\E_{\x_0,\epsilonv}\left[\frac{1}{\sigma_t^2}\big\|\sigma_t\s_1(\x_t,t;\theta)+\epsilonv\big\|_2^2\right],
\end{equation}
where $\epsilonv\sim\N(\vect{0},\Iv)$ and $\x_t=\alpha_t\x_0+\sigma_t\epsilonv$.

\subsection{High-Order Denoising Score Matching}
In this section, we generalize the first-order DSM to second and third orders for matching the second-order and third-order score functions defined in \cref{thrm:fisher_bound}.
We firstly present the second-order method below.
\begin{theorem}
\label{thrm:second}
(Error-Bounded Second-Order DSM) Suppose that $\hat\s_1(\x_t,t)$ is an estimation for the first-order data score function $\nabla_{\x}\log q_t(\x_t)$, then we can learn a second-order score model $\s_2(\cdot,t;\theta): \R^{d}\rightarrow\R^{d\times d}$ parameterized by $\theta$ which minimizes
\begin{equation*}
    \E_{q_t(\x_t)}\left[\left\|\s_2(\x_t,t;\theta)-\nabla^2_{\x}\log q_t(\x_t)\right\|_F^2\right],
\end{equation*}
by optimizing
{
\small
\begin{equation}
\label{eqn:dsm-2-obj}
\!\!\!\theta^*=\argmin_{\theta}\E_{\x_0,\epsilonv}\left[\frac{1}{\sigma_t^4}\big\|\sigma_t^2\s_2(\x_t,t;\theta)+\Iv-\ellv_1\ellv_1^\top\big\|^2_F\right],
\end{equation}
}%
where
\begin{equation}
\label{eqn:ell_1}
\begin{aligned}
    &\ellv_1(\epsilonv,\x_0,t)\coloneqq\sigma_t\hat\s_1(\x_t,t)+\epsilonv,\\
    &\x_t=\alpha_t\x_0+\sigma_t\epsilonv,\quad\epsilonv\sim \Nc(\vect{0},\Iv).
\end{aligned}
\end{equation}
Moreover, denote the first-order score matching error as $\delta_1(\x_t,t)\coloneqq\|\hat\s_1(\x_t,t)-\nabla_{\x}\log q_t(\x_t)\|_2$, then $\forall \x_t,  \theta$, the score matching error for $\s_2(\x_t,t;\theta)$ can be bounded by
\begin{equation*}
\begin{split}
    &\left\|\s_2(\x_t,t;\theta)-\nabla_{\x}^2\log q_t(\x_t)\right\|_F\\
    \leq\ &\left\|\s_2(\x_t,t,\theta)-\s_2(\x_t,t;\theta^*)\right\|_F+\delta_1^2(\x_t,t).
\end{split}
\end{equation*}
\end{theorem}
Theorem~\ref{thrm:second} shows that the DSM objective in problem~(\ref{eqn:dsm-2-obj}) is a valid surrogate for second-order score matching,
because the difference between the model score $\s_2(\x_t,t;\theta)$ and the true second-order score $\nabla^2_{\x}\log q_t(\x_t)$ can be bounded by the training error $\|\s_2(\x_t,t;\theta)-\s_2(\x_t,t;\theta^*)\|_F$ and the first-order score matching error $\delta_1(\x_t,t)$. Note that previous second-order DSM method proposed in~\citet{meng2021estimating} does not have such error-bounded property, and we refer to Sec.~\ref{sec:related} and Appendix~\ref{appendix:differ_high_order} for the detailed comparison. In addition, recent work~\citep[Theorem 3.3]{bao2022estimating} about the optimal covariance of the diffusion models considers the estimation error of the mean of the diffusion models, which is equivalent to our proposed error-bounded second-order DSM with considering the first-order score matching error.

The DSM objective in problem~(\ref{eqn:dsm-2-obj}) requires learning a matrix-valued function $\nabla^2_{\x}\log q_t(\x_t)$, but sometimes we only need the trace $\tr(\nabla^2_{\x}\log q_t(\x_t))$ of the second-order score function. 
Below we present a corollary for only matching the trace of the second-order score function.
\begin{corollary}
\label{coro:second_trace}
(Error-Bounded Trace of Second-Order DSM) We can learn a second-order trace score model $\s_2^{\text{trace}}(\cdot,t;\theta):\R^{d}\rightarrow\R$ which minimizes
\begin{equation*}
    \E_{q_t(\x_t)}\left[\left|\s_2^{\text{trace}}(\x_t,t;\theta)-\tr(\nabla_{\x}^2\log q_t(\x_t))\right|^2\right],
\end{equation*}
by optimizing
{\small
\begin{equation}
\label{eqn:dsm-2-trace-obj}
    \theta^*\!=\argmin_{\theta}\E_{\x_0,\epsilonv}\!\left[\frac{1}{\sigma_t^4}\Big|\sigma_t^2\s_2^{\text{trace}}(\x_t,t;\theta)\!+\!d\!-\!\|\ellv_1\|_2^2\Big|^2\right].
\end{equation}
}%
The estimation error for $\s_2^{\text{trace}}(\x_t,t;\theta)$ can be bounded by:
\begin{equation*}
 \begin{split}
&\left|\s_2^{\text{trace}}(\x_t,t;\theta)-\tr(\nabla^2_{\x}\log q_t(\x_t))\right|\\
\leq\ &\left|\s_2^{\text{trace}}(\x_t,t;\theta)-\s_2^{\text{trace}}(\x_t,t;\theta^*)\right|+\delta_1^2(\x_t,t).
 \end{split} 
\end{equation*}
\end{corollary}
Finally, we present the third-order DSM method. The score matching error can also be bounded by the training error and the first, second-order score matching errors.
\begin{theorem}
\label{thrm:third}
(Error-Bounded Third-Order DSM) Suppose that $\hat\s_1(\x_t,t)$ is an estimation for $\nabla_{\x}\log q_t(\x_t)$ and $\hat\s_2(\x_t,t)$ is an estimation for $\nabla^2_{\x}\log q_t(\x_t)$, then we can learn a third-order score model $\s_3(\cdot,t;\theta): \R^{d}\rightarrow\R^{d}$ which minimizes
\begin{equation*}
    \E_{q_t(\x_t)}\left[\left\|\s_3(\x_t,t;\theta)-\nabla_{\x}\tr(\nabla_{\x}^2\log q_t(\x_t))\right\|_2^2\right],
\end{equation*}
by optimizing
\begin{equation}
\label{eqn:dsm-3-obj}
\begin{split}
    &\theta^*=\argmin_{\theta}\E_{\x_0,\epsilonv}\left[\frac{1}{\sigma_t^6}\big\|\sigma_t^3\s_3(\x_t,t;\theta)+\ellv_3\big\|_2^2\right]
\end{split}
\end{equation}
where
\begin{equation}
\label{eqn:ell_3}
\begin{split}
    &\ellv_1(\epsilonv,\x_0,t)\coloneqq\sigma_t\hat\s_1(\x_t,t)+\epsilonv,\\
    &\ellv_2(\epsilonv,\x_0,t)\coloneqq\sigma_t^2\hat\s_2(\x_t,t)+\Iv,\\
    &\ellv_3(\epsilonv,\x_0,t)\coloneqq\left(\|\ellv_1\|_2^2\Iv-\tr(\ellv_2)\Iv-2\ellv_2\right)\ellv_1,\\
    &\x_t=\alpha_t\x_0+\sigma_t\epsilonv,\quad \epsilonv\sim\N(\vect{0},\Iv).
\end{split}
\end{equation}
Denote the first-order score matching error as $\delta_1(\x_t,t)\coloneqq\|\hat\s_1(\x_t,t)-\nabla_{\x}\log q_t(\x_t)\|_2$ and the second-order score matching errors as $\delta_2(\x_t,t)\coloneqq\|\hat\s_2(\x_t,t)-\nabla_{\x}^2\log q_t(\x_t)\|_F$ and $\delta_{2,\text{tr}}(\x_t,t)\coloneqq|\tr(\hat\s_2(\x_t,t))-\tr(\nabla_{\x}^2\log q_t(\x_t))|$. Then $\forall \x_t, \theta$, the score matching error for $\s_3(\x_t,t;\theta)$ can be bounded by:
\begin{equation*}
\begin{split}
    &\left\|\s_3(\x_t,t;\theta)-\nabla_{\x}\tr(\nabla_{\x}^2\log q_t(\x_t))\right\|_2\\
    \leq\ &\left\|\s_3(\x_t,t;\theta)-\s_3(\x_t,t;\theta^*)\right\|_2+\big(\delta_1^2+\delta_{2,\text{tr}}+2\delta_2\big)\delta_1^2
\end{split}
\end{equation*}
\end{theorem}
\cref{thrm:third} shows that the third-order score matching needs first-order score matching, second-order score matching and trace of second-order score matching. The third-order score matching error can also be bounded by the training error and the lower-order score matching errors. We believe that our construction for the error-bounded high-order DSM can be extended to even higher orders by carefully designing the training objectives. In this paper, we only focus on the second and third-order methods, and leave this extension for future work.

%% file: 5-implementation.tex
\section{Training Score Models by High-Order DSM}
\label{sec:training}
Building upon the high-order DSM for a specific time $t$ (see  Sec.~\ref{sec:high_order_SM}), we present our algorithm to train {\ode}s in the perspective of maximum likelihood by considering all timesteps $t\in[0,T]$. 
Practically, we further leverage the ``noise-prediction'' trick~\cite{kingma2021variational,ho2020denoising} for variance reduction and the Skilling-Hutchinson trace estimator~\cite{skilling1989eigenvalues, hutchinson1989stochastic} to compute the involved high-order derivatives efficiently. The whole training algorithm is presented detailedly in Appendix.~\ref{appendix:training_algorithm}.

\subsection{Variance Reduction by Time-Reweighting}
\label{sec:variance_reduction}
Theoretically, to train score models for all $t\in[0,T]$, we need to integrate the DSM objectives in Eqn.~\eqref{eqn:dsm-1-obj}\eqref{eqn:dsm-2-obj}\eqref{eqn:dsm-2-trace-obj}\eqref{eqn:dsm-3-obj} from $t=0$ to $t=T$, which needs ODE solvers and is time-consuming. Instead, in practice, we follow the method in \cite{song2020score,ho2020denoising} which uses Monte-Carlo method to unbiasedly estimate the objectives by sample $t\in[0,T]$ from a proposal distribution $p(t)$, avoiding ODE solvers. The main problem for the Monte-Carlo method is that sometimes the sample variance of the objectives may be large due to the different value of $\frac{1}{\sigma_t}$. To reduce the variance, we take the time-reweighted objectives by multiplying $\sigma_t^2$,$\sigma_t^4$,$\sigma_t^6$ with the corresponding first, second, third-order score matching objectives at each time $t$, which is known as the ``noise-prediction'' trick~\cite{kingma2021variational,ho2020denoising} for the first-order DSM objective. Specifically, assume that $\x_0\sim q_0(\x_0)$, the time proposal distribution $p(t)=\Uc[0,T]$ is the uniform distribution in $[0,T]$,  the random noise $\epsilonv\sim\N(\vect{0},\Iv)$ follows the standard Gaussian distribution, and let $\x_t=\alpha_t\x_0+\sigma_t\epsilonv$. The training objective for the first-order DSM in Eqn.~\eqref{eqn:dsm-1-obj} through all $t\in[0,T]$ is
\begin{equation}
\label{eqn:dsm-1-final}
    \Jc_{\DSM}^{(1)}(\theta)\coloneqq \E_{t,\x_0,\epsilonv}\Big[\big\|\sigma_t\s_\theta(\x_t,t)+\epsilonv\big\|_2^2\Big],
\end{equation}
which empirically has low sample variance~\cite{song2020score,ho2020denoising}.

As for the second and third-order DSM objectives, we take $\hat\s_1(\x_t,t)\coloneqq \s_\theta(\x_t,t)$ for the first-order score function estimation, and $\hat\s_2(\x_t,t)\coloneqq \nabla_{\x}\s_\theta(\x_t,t)$ for the second-order score function estimation, and both disable the gradient computations for $\theta$ (which can be easily implemented by \texttt{stop\_gradient} or \texttt{detach}). The reason for disabling gradients is because the high-order DSM only needs the estimation values of the lower-order score functions, as shown in Theorem.~\ref{thrm:second} and \ref{thrm:third}. The training objectives through all $t\in[0,T]$ for the second-order DSM in Eqn.~\eqref{eqn:dsm-2-obj}, for the trace of second-order DSM in Eqn.~\eqref{eqn:dsm-2-trace-obj} and for the third-order DSM in Eqn.~\eqref{eqn:dsm-3-obj} are
\begin{equation*}
\begin{aligned}
    \Jc_{\DSM}^{(2)}(\theta)\!&\coloneqq\! \E_{t,\x_0,\epsilonv}\Big[\big\|\sigma_t^2\nabla_{\x}\s_\theta(\x_t,t)+\Iv-\ellv_1\ellv_1^\top\big\|_F^2\Big],\\
    \Jc_{\DSM}^{(2,\text{tr})}(\theta)\!&\coloneqq\! \E_{t,\x_0,\epsilonv}\Big[\big|\sigma_t^2\tr(\nabla_{\x}\s_\theta(\x_t,t))\!+\!d\!-\!\|\ellv_1\|_2^2\big|^2\Big],\\
    \Jc_{\DSM}^{(3)}(\theta)\!&\coloneqq\! \E_{t,\x_0,\epsilonv}\Big[\big\|\sigma_t^3\nabla_{\x}\!\tr(\nabla_{\x}\s_\theta(\x_t,t))\!+\ellv_3\big\|_2^2\Big],
\end{aligned}
\end{equation*}
respectively, where $\ellv_1,\ellv_3$ are the same as that in Eqn.~\eqref{eqn:ell_3}. Combining with the first-order objective in Eqn.~\eqref{eqn:dsm-1-final}, our final training objective is
\begin{equation}
\label{eqn:final_objective}
    \min_{\theta}\Jc_{\DSM}^{(1)}(\theta)+\lambda_1\left(\Jc_{\DSM}^{(2)}(\theta)+\Jc_{\DSM}^{(2,tr)}(\theta)\right)+\lambda_2\Jc_{\DSM}^{(3)}(\theta),
\end{equation}
where $\lambda_1,\lambda_2\geq0$ are hyperparameters, and we discuss in Appendix.~\ref{appendix:choose_lambda} for details.
\begin{remark}
Although as shown in Eqn.~\eqref{eqn:change_of_variable_score_ode}, we can exactly compute $\log p_0^{\ODE}(\x_0)$ for a given data point $\x_0$ via solving the ODE~\cite{chen2018neural, grathwohl2018ffjord}, this method is hard to scale up for the large neural networks used in SGMs. For example, it takes $2\sim 3$ minutes for evaluating $\log p^{\ODE}_0$ for a single batch of the {\ode} used in~\cite{song2020score}. Instead, our proposed method leverages Monte-Carlo methods to avoid the ODE solvers. %
\end{remark}

\subsection{Scalability and Numerical Stability}
As the second and third-order DSM objectives need to compute the high-order derivatives of a neural network and are expensive for high-dimensional data, we use the Skilling-Hutchinson trace estimator~\cite{skilling1989eigenvalues, hutchinson1989stochastic} to unbiasedly estimate the trace of the Jacobian~\cite{grathwohl2018ffjord} and the Frobenius norm of the Jacobian~\cite{finlay2020train}. The detailed objectives for high-dimensional data can be found in Appendix.~\ref{appendix:ubiased_estimation}.

In practice, we often face numerical instability problems for $t$ near to $0$. We follow~\citet{song2021maximum} to choose a small starting time $\epsilon>0$, and both of the training and the evaluation are performed for $t\in[\epsilon,T]$ instead of $[0,T]$. Note that the likelihood evaluation is still exact, because we use $p_{\epsilon}^{\ODE}(\x_0)$ to compute the log-likelihood of $\x_0$, which is still a well-defined density model~\cite{song2021maximum}.

%% file: 4-related.tex
\vspace{-.1cm}
\section{Related Work}
\label{sec:related}
\vspace{-.05cm}
{\ode}s are special formulations of Neural ODEs (NODEs)~\cite{chen2018neural}, and our proposed method can be viewed as maximum likelihood training of NODEs by high-order score matching. Traditional maximum likelihood training for NODEs aims to match the NODE distribution at $t=0$ with the data distribution, and it cannot control the distribution between $0$ and $T$. \citet{finlay2020train} show that training NODEs by simply maximizing likelihood could result in unnecessary complex dynamics, which is hard to solve. Instead, our high-order score matching objective is to match the distributions between the forward process distribution and the NODE distribution at each time $t$, and empirically the dynamics is kind of smooth. Moreover, our proposed algorithm uses the Monto-Carlo method to unbiasedly estimate the objectives, which does not need any black-box ODE solvers. Therefore, our algorithm is suitable for maximum likelihood training for large-scale NODEs.

Recently, \citet{meng2021estimating} propose a high-order DSM method for estimating the second-order data score functions. However, the training objective of our proposed error-bounded high-order DSM is different from that of~\citet{meng2021estimating}. Our proposed algorithm can guarantee a bounded error of the high-order score matching exactly by  the lower-order estimation errors and the training error, while the score matching error raised by minimizing the objective in~\citet{meng2021estimating} may be unbounded, even when the lower-order score matching error is small and the training error of the high-order DSM is zero (see Appendix~\ref{appendix:differ_high_order} for detailed analysis). Moreover, our method can also be used to train a separate high-order score model in other applications, such as uncertainty quantification and Ozaki sampling presented in~\citet{meng2021estimating}.

%% file: 6-experiments.tex
\section{Experiments}
\vspace{-.05cm}

In this section, we demonstrate that our proposed high-order DSM algorithm can improve the likelihood of {\ode}s, while retaining the high sample quality of the corresponding {\sde}s. Particularly, we use the Variance Exploding (VE)~\cite{song2020score} type diffusion models, which empirically have shown high sample quality by the {\sde} but poor likelihood by the {\ode}~\cite{song2021maximum} when trained by minimizing the first-order score matching objective $\Jc_{\SM}(\theta)$. We implement our experiments by JAX~\cite{jax2018github},  which is efficient for computing the derivatives of the score models. In all experiments, we choose the start time $\epsilon=10^{-5}$, which follows the default settings in~\citet{song2020score}. In this section, we refer to ``first-order score matching'' as minimizing the objective in Eqn.~\eqref{eqn:final_objective} with $\lambda_1=\lambda_2=0$, and ``second-order score matching'' as minimizing the one with $\lambda_2=0$. Please see Appendix.~\ref{appendix:choose_lambda} for detailed settings about $\lambda_1$ and $\lambda_2$. The released code can be found at \url{https://github.com/LuChengTHU/mle_score_ode}.
\begin{remark}
\citet{song2021maximum} use a proposal distribution $t\sim p(t)$ to adjust the weighting for different time $t$ to minimize $\Jc_{\SM}(\theta)$. In our experiments, we mainly focus on the VE type, whose proposal distribution $p(t)=\Uc[0,T]$ is the same as that of our final objective in Eqn.~\eqref{eqn:final_objective}.
\end{remark}

\begin{figure}[t]
	\centering		
 	\includegraphics[width=.7\linewidth]{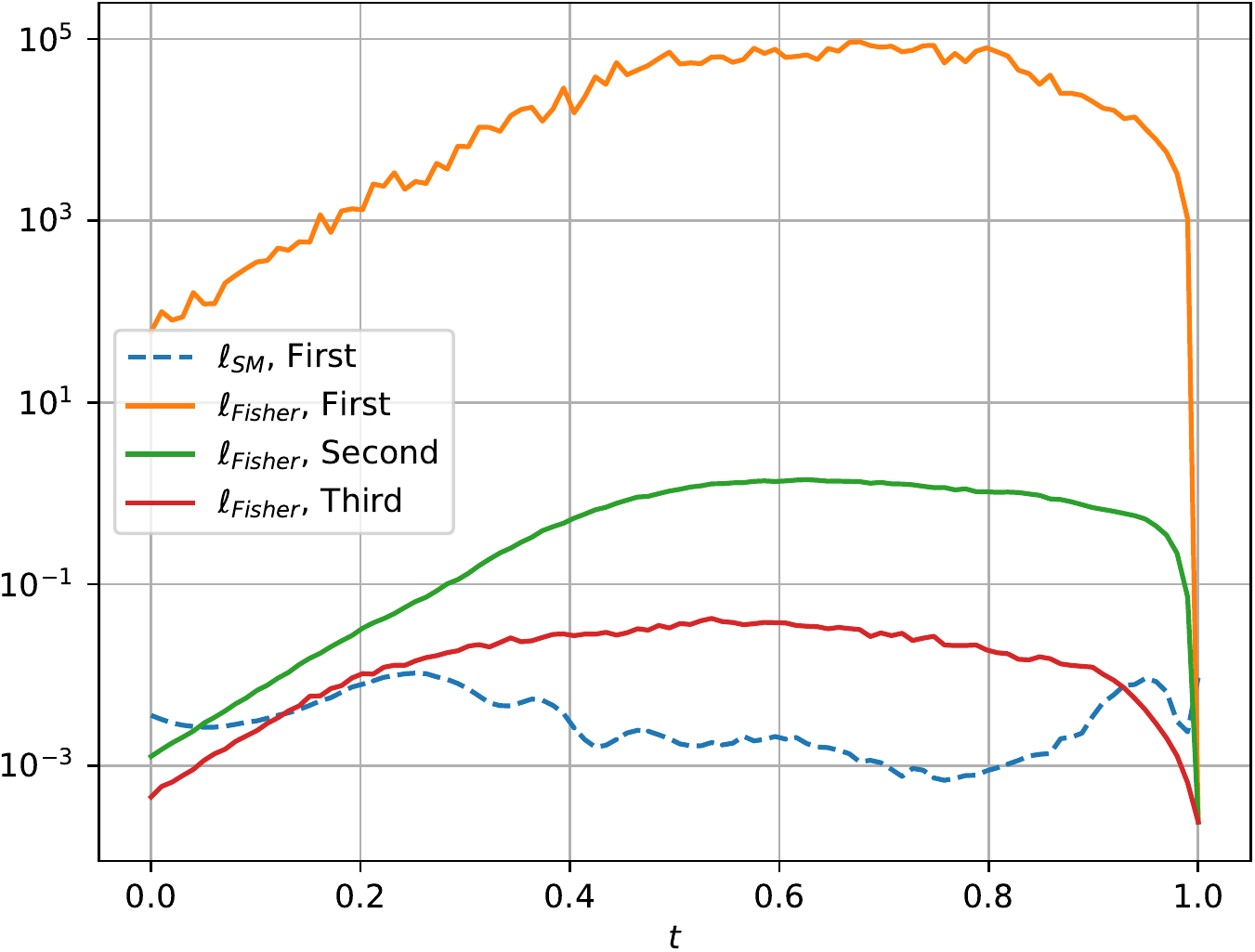}\\
 	\vspace{-.1in}
	\caption{$\ell_{\FISH}(t)$ and $\ell_{\SM}(t)$ of {\ode}s (VE type) on 1-D mixture of Gaussians, trained by minimizing the first, second, third-order score matching objectives.\label{fig:mog_plot}}
 	\vspace{-.1in}
\end{figure}

\begin{figure*}[t]
	\centering
	\begin{minipage}{.28\linewidth}
		\centering
			\includegraphics[width=.9\linewidth]{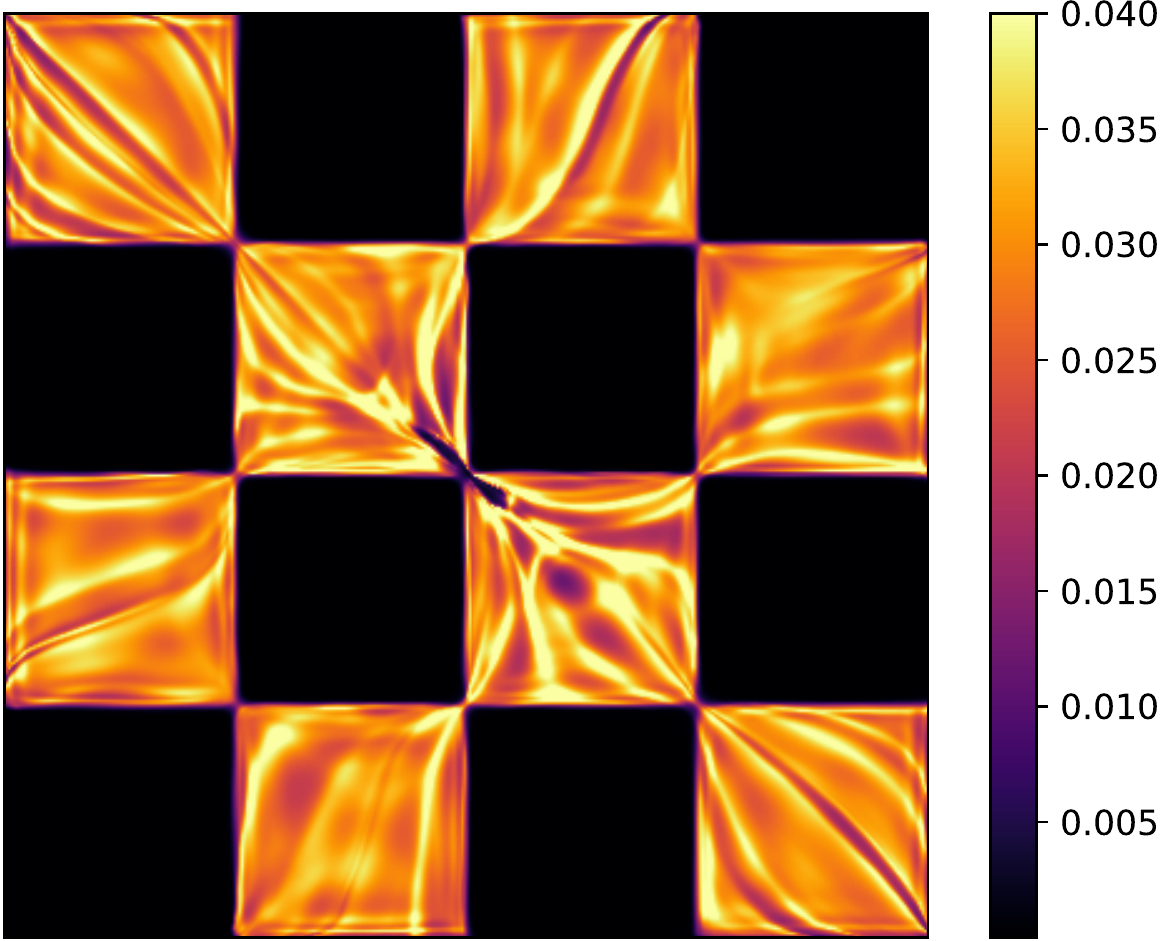}\\
\small{(a) First-order score matching}
	\end{minipage}
	\begin{minipage}{.28\linewidth}
	\centering
	\includegraphics[width=.9\linewidth]{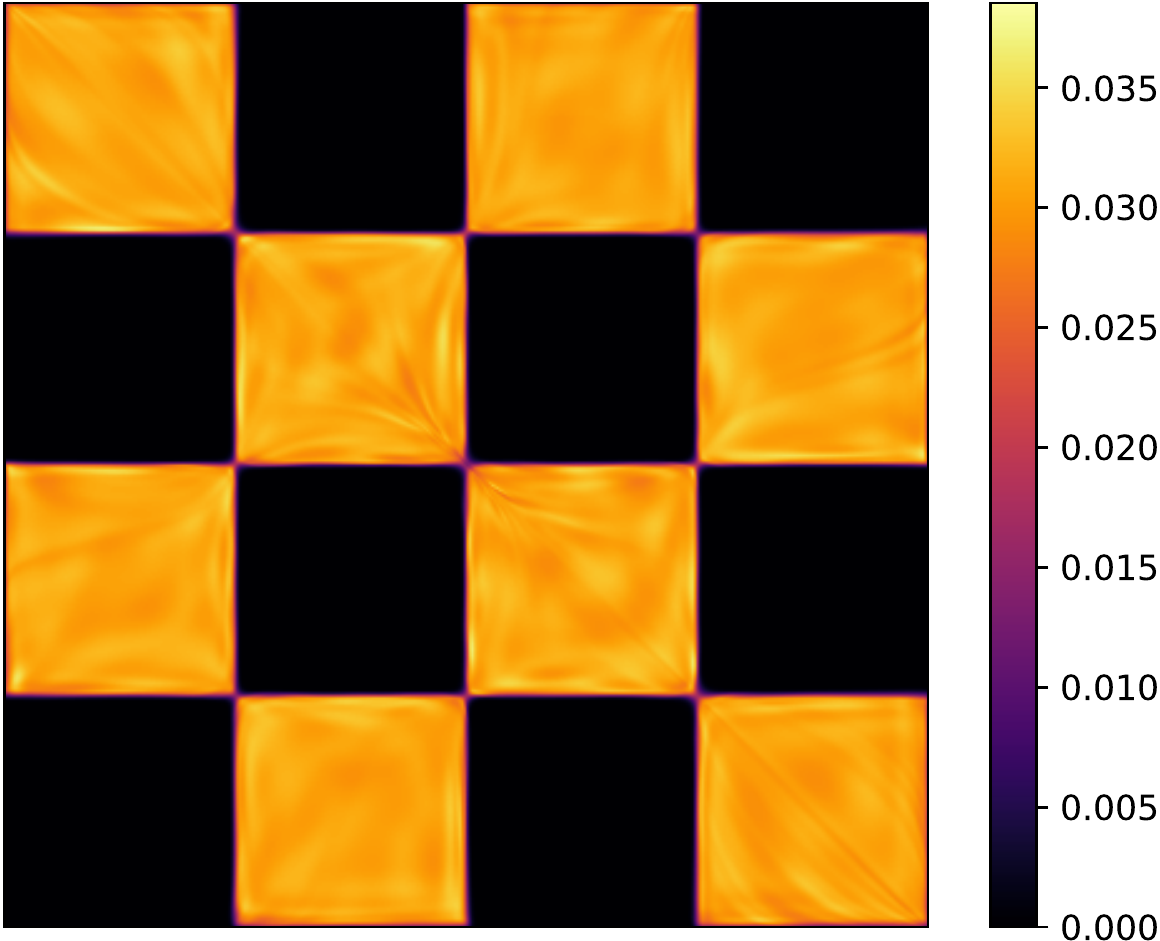}\\
\small{	(b) Second-order score matching}
\end{minipage}
	\begin{minipage}{.28\linewidth}
		\centering		
 	\includegraphics[width=.9\linewidth]{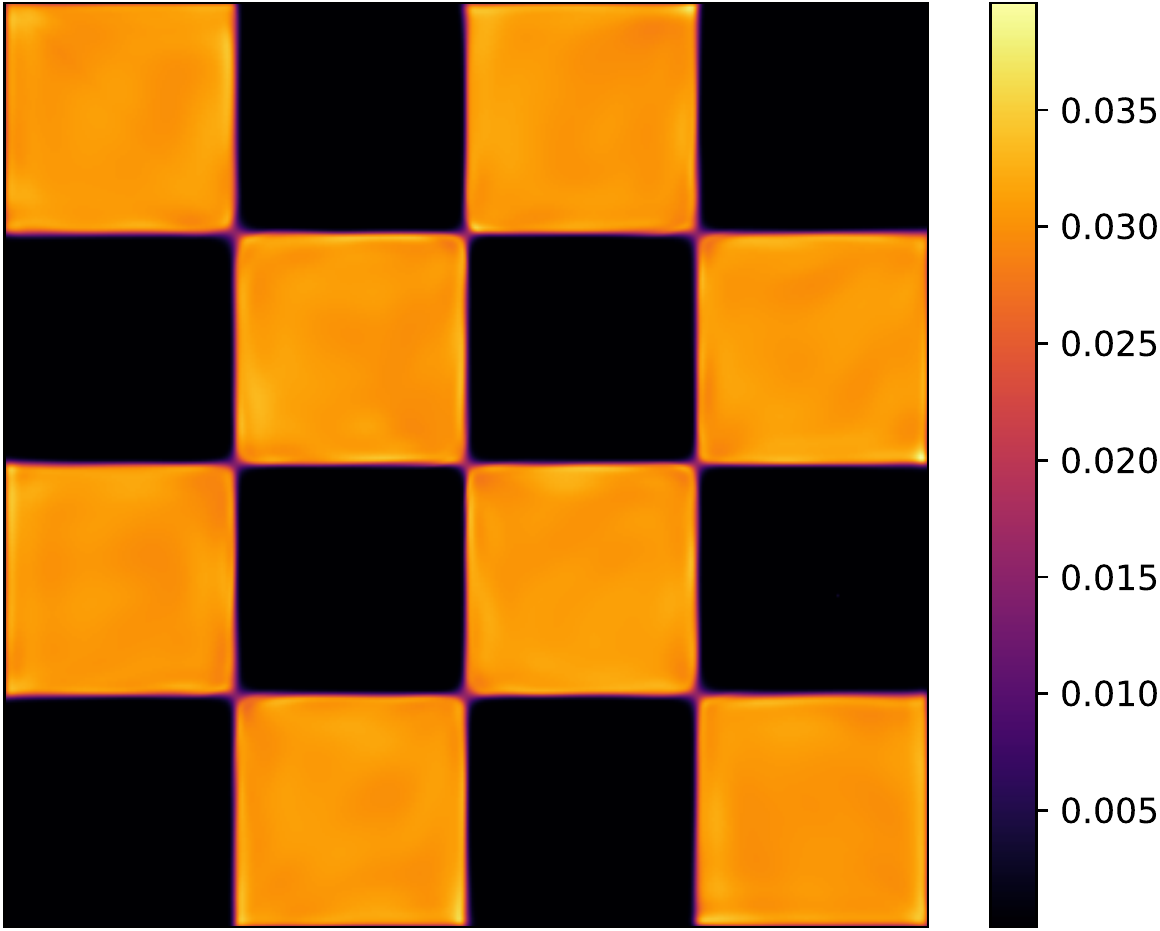}\\
\small{(c) Third-order score matching}
   \end{minipage}
   \vspace{-.05in}
	\caption{Model density of {\ode}s (VE type) on 2-D checkerboard data.\label{fig:checker_density}}
	\vspace{-.15in}
\end{figure*}

\subsection{Example: 1-D Mixture of Gaussians}
\vspace{-.1cm}
\label{sec:mog_example}

We take an example to demonstrate how the high-order score matching training impacts the weighted Fisher divergence $\Jc_{\FISH}(\theta)$ and the model density of {\ode}s. Let $q_0$ be a 1-D Gaussian mixture distribution as shown in Fig.~\ref{fig:mog_intro} (a). We train a score model of VE type by minimizing the first-order score matching objective $\Jc_{\SM}(\theta)$, and use the corresponding {\ode} to evaluate the model density, as shown in Fig.~\ref{fig:mog_intro} (b). The model density achieved by only first-order score matching is quite different from the data distribution at some data points. However, by third-order score matching, the model density in Fig.~\ref{fig:mog_intro} (c) is quite similar to the data distribution, showing the effectiveness of our proposed algorithm.

We further analyze the difference of $\Jc_{\FISH}(\theta)$ between the first, second, third-order score matching training. As $q_0$ is a Gaussian mixture distribution, we can prove that each $q_t$ is also a Gaussian mixture distribution, and $\nabla_{\x}\log q_t$ can be analytically computed. Denote
\begin{align*}
    &\ell_{\FISH}(t)\!\coloneqq\! \frac{1}{2}g(t)^2\fisher{q_t}{p_t^{\ODE}},\\
    &\ell_{\SM}(t)\!\coloneqq\! \frac{1}{2}g(t)^2\E_{q_t(\x_t)}\|\s_\theta(\x_t,t)\!-\!\nabla_{\x}\!\log q_t(\x_t)\|_2^2,
\end{align*}
which are the integrands of each time $t$ for $\Jc_{\SM}(\theta)$ and $\Jc_{\FISH}(\theta)$. We use the method in Appendix.~\ref{appendix:change_of_score} to compute $\fisher{q_t}{p_t^{\ODE}}$, and evaluate $\ell_{\FISH}(t)$ and $\ell_{\SM}(t)$ from $t=\epsilon$ to $t=T$ with $100$ points of the {\ode}s trained by first, second, third-order denoising score matching objectives. As shown in Fig.~\ref{fig:mog_plot}, although $\ell_{\SM}(t)$ is small when minimizing first-order score matching objective, $\ell_{\FISH}(t)$ is rather large for most $t$, which means $\Jc_{\FISH}(\theta)$ is also large. However, the second and third score matching gradually reduce $\ell_{\FISH}(t)$, which shows that high-order score matching training can reduce $\ell_{\FISH}(t)$ and then reduce $\Jc_{\FISH}(\theta)$. Such results empirically validate that controlling the high-order score matching error can control the Fisher divergence, as proposed in Theorem~\ref{thrm:fisher_bound}.

\vspace{-.1cm}
\subsection{Density Modeling on 2-D Checkerboard Data}
\vspace{-.1cm}

We then train {\ode}s (VE type) by the first, second, third-order score matchings on the \textit{checkerboard} data whose density is multi-modal, as shown in Fig.~\ref{fig:checker_density}. We use a simple MLP neural network with Swish activations~\cite{ramachandran2017searching}, and the detailed settings are in Appendix.~\ref{appendix:checkerboard}.

For the second and third-order score matchings, we do not use any pre-trained lower-order score models. Instead, we train the score model and its first, second-order derivatives from scratch, using the objective in Eqn.~\eqref{eqn:final_objective}. The model density by the first-order score matching is rather poor, because it cannot control the value of $\nabla_{\x}\log p^{\ODE}_t$ in $\Jc_{\DIFF}(\theta)$ in Eqn.~\eqref{eqn:J_diff}. The second-order score matching can control part of $\nabla_{\x}\log p^{\ODE}_t$ and improve the model density, while the third-order score matching can reduce the Fisher divergence $\fisher{q_t}{p_t^{\ODE}}$ and achieve excellent model density. Such results show the effectiveness of our error-bounded high-order denoising score matching methods, and accord with our theoretical analysis in Theorem~\ref{thrm:kl-ode} and \ref{thrm:fisher_bound}.

\vspace{-.1cm}
\subsection{Density Modeling on Image Datasets}
\vspace{-.1cm}

We also train {\ode}s (VE type) on both the CIFAR-10 dataset~\cite{Krizhevsky09learningmultiple} and the ImageNet 32x32 dataset~\cite{deng2009imagenet}, which are two of the most popular datasets for generative modeling and likelihood evaluation. We use the estimated training objectives by trace estimators, which are detailed in Appendix.~\ref{appendix:ubiased_estimation}. We use the same neural networks and hyperparameters as the NCSN++ cont. model in~\cite{song2020score} (denoted as VE) and the NCSN++ cont. deep model in~\cite{song2020score} (denoted as VE (deep)), respectively (see detailed settings in Appendix.~\ref{appendix:cifar10} and Appendix.~\ref{appendix:imagenet32}). We evaluate the likelihood by the {\ode} $p^{\ODE}_{\epsilon}$, and sample from the {\sde} by the PC sampler~\cite{song2020score}. As shown in Table~\ref{tab:cifar10}, our proposed high-order score matching can improve the likelihood performance of {\ode}s, while retraining the high sample quality of the corresponding {\sde}s (see Appendix.~\ref{appendix:additional_samples} for samples). Moreover, the computation costs of the high-order DSM objectives are acceptable because of the trace estimators and the efficient ``Jacobian-vector-product'' computation in JAX. We list the detailed computation costs in Appendix.~\ref{appendix:cifar10}. 

We also compare the VE, VP and subVP types of ScoreODEs trained by the first-order DSM and the third-order DSM, and the detailed results are listed in Appendix.~\ref{appendix:additional_samples}.

\begin{table}
    \vspace{-.1in}
    \centering
    \caption{Negative log-likelihood (NLL) in bits/dim (bpd) and sample quality (FID scores) on CIFAR-10 and ImageNet 32x32.}
    \vskip 0.1in
    \begin{small}
    \resizebox{0.48\textwidth}{!}{%
    \begin{tabular}{lccc}
    \toprule
    \multirow{3}{*}{Model} & \multicolumn{2}{c}{\multirow{2}{*}{CIFAR-10}} & ImageNet \Bstrut\\
    & & & 32x32\\
    \cline{2-4}
    & NLL $\downarrow$ & FID $\downarrow$ & NLL $\downarrow$ \Tstrut\\
    \midrule
        VE~\cite{song2020score} & 3.66 & 2.42 & 4.21\\
        VE (second)~(\bf{ours}) & 3.44 & \textbf{2.37} & 4.06\\
        VE (third)~(\bf{ours}) & \bf{3.38} & 2.95 & \textbf{4.04}\Bstrut\\
    \hline
        VE (deep)~\cite{song2020score} & 3.45 & \textbf{2.19} & 4.21\Tstrut \\
        VE (deep, second)~(\bf{ours}) & 3.35 & 2.43 & 4.05\\
        VE (deep, third)~(\bf{ours}) & \bf{3.27} & 2.61 & \textbf{4.03}\\
    \bottomrule
    \end{tabular}%
    }
    \end{small}
    \label{tab:cifar10}
    \vspace{-0.15in}
\end{table}

%% file: 0-appendix.tex
\section{Assumptions}
\label{appendix:assumptions}
We follow the regularity assumptions in~\cite{song2021maximum} to ensure the existence of reverse-time SDEs and probability flow ODEs and the correctness of the ``integration by parts'' tricks for the computation of KL divergence. And to ensure the existence of the third score function, we change the assumptions of differentiability. For completeness, we list all these assumptions in this section.

For simplicity, in the Appendix sections, we use $\nabla(\cdot)$ to denote $\nabla_{\x}(\cdot)$ and omit the subscript $\x$. And we denote $\nabla\cdot\hv(\x)\coloneqq\tr(\nabla\hv(\x))$ as the divergence of a function $\hv(\x):\R^d\rightarrow\R^d$.
\begin{assumption}
\label{assumption:all}
We make the following assumptions, most of which are presented in~\cite{song2021maximum}:
\begin{enumerate}
    \item $q_0(\x)\in \mathcal{C}^3$ and $\E_{q_0(\x)}[\|\x\|_2^2]<\infty$.
    \item $\forall t\in [0,T]:\fv(\cdot,t)\in \mathcal{C}^2$. And $\exists C>0$, $\forall \x\in\R^d,t\in [0,T]:\|\fv(\x,t)\|_2\leq C(1+\|\x\|_2)$.
    \item $\exists C>0,\forall \x,\y\in\R^d:\|\fv(\x,t)-\fv(\y,t)\|_2\leq C\|\x-\y\|_2$.
    \item $g\in \mathcal{C}$ and $\forall t\in [0,T],|g(t)|>0$.
    \item For any open bounded set $\mathcal{O}$, $\int_0^T\int_{\mathcal{O}}\|q_t(\x)\|_2^2+d\cdot g(t)^2\|\nabla q_t(\x)\|_2^2 \dxv\dt<\infty$.
    \item $\exists C>0,\forall \x\in\R^d,t\in [0,T]:\|\nabla q_t(\x)\|_2^2\leq C(1+\|\x\|_2)$.
    \item $\exists C>0,\forall \x,\y\in\R^d:\|\nabla\log q_t(\x)-\nabla\log q_t(\y)\|_2\leq C\|\x-\y\|_2$.
    \item $\exists C>0,\forall \x\in\R^d,t\in [0,T]:\|\s_{\theta}(\x,t)\|_2\leq C(1+\|\x\|_2)$.
    \item $\exists C>0,\forall \x,\y\in\R^d:\|\s_{\theta}(\x,t)-\s_{\theta}(\y,t)\|_2\leq C\|\x-\y\|_2$.
    \item Novikov’s condition: $\E\left[\exp\left(\frac{1}{2}\int_0^T\|\nabla\log q_t(\x)-\s_{\theta}(\x,t)\|_2^2\dt \right) \right]<\infty$.
    \item $\forall t\in [0,T],\exists k>0:q_t(\x)=O(e^{-\|\x\|_2^k})$, $p_t^{\SDE}(\x)=O(e^{-\|\x\|_2^k})$, $p_t^{\ODE}(\x)=O(e^{-\|\x\|_2^k})$ as $\|\x\|_2\rightarrow\infty$.
\end{enumerate}
\end{assumption}

\section{Distribution gap between score-based diffusion SDEs and ODEs}
\label{appendix:distribution_gap}
\begin{table}
    \centering
    \caption{Three distributions: $q_t(\x_t)$, $p_t^{\SDE}(\x_t)$ and $p_t^{\ODE}(\x_t)$ and their corresponding equivalent SDEs / ODEs.}
    \vskip 0.10in
    \begin{tabular}{c|c|c}
    \toprule
    Distribution & Dynamics type & Formulation\\
    \midrule
    \multirow{3}{*}{$q_t(\x_t)$} & Forward SDE & $\dxv_t=\fv(\x_t,t)\dt+g(t)\dwv_t$\\
    &Reverse SDE & $\dxv_t=[\fv(\x_t,t)-g(t)^2\nabla_{\x}\log q_t(\x_t)]\dt+g(t)\dm\bar{\wv}_t$\\
    &Probability flow ODE & $\frac{\dxv_t}{\dt}=\fv(\x_t,t)-\frac{1}{2}g(t)^2\nabla_{\x}\log q_t(\x_t)$\Bstrut\\
    \hline
    \multirow{3}{*}{$p_t^{\SDE}(\x_t)$} & Forward SDE & $\dxv_t=[\fv(\x_t,t)+g(t)^2(\nabla\log p_t^{\SDE}(\x_t)-\s_\theta(\x_t,t))]\dt+g(t)\dm\wv_t$\Tstrut\\
    &Reverse SDE & $\dxv_t=[\fv(\x_t,t)-g(t)^2\s_\theta(\x_t,t)]\dt+g(t)\dm\bar{\wv}_t$\\
    &Probability flow ODE & $\frac{\dxv_t}{\dt}=\fv(\x_t,t)-g(t)^2\sv_\theta(\x_t,t)+\frac{1}{2}g(t)^2\nabla_{\x}\log p_t^{\SDE}(\x_t)$\Bstrut\\
    \hline
    \multirow{3}{*}{$p_t^{\ODE}(\x_t)$} & Forward SDE & $\backslash$\Tstrut\\
    &Reverse SDE & $\backslash$\\
    &Probability flow ODE & $\frac{\dxv_t}{\dt}=\fv(\x_t,t)-\frac{1}{2}g(t)^2\sv_\theta(\x_t,t)$\Bstrut\\
    \bottomrule
    \end{tabular}
    \label{tab:SDE_and_ODE}
\end{table}
We list the corresponding equivalent SDEs and ODEs of $q_t(\x_t)$, $p_t^{\SDE}(\x_t)$ and $p_t^{\ODE}(\x_t)$ in Table~\ref{tab:SDE_and_ODE}. In most cases, $q_t(\x_t)$, $p_t^{\SDE}(\x_t)$ and $p_t^{\ODE}(\x_t)$ are different distributions, which we will prove in this section.

According to the reverse SDE and probability flow ODE listed in the table, it is obvious that when $\s_\theta(\x_t,t)\neq \nabla\log q_t(\x_t)$, $q_t(\x_t)$ and $p_t^{\SDE}$ are different, and $q_t(\x_t)$, $p_t^{\ODE}$ are different. Below we show that in most cases, $p_t^{\SDE}$ and $p_t^{\ODE}$ are also different.

Firstly, we should notice that the probability flow ODE of the {\sde} in Eqn.~\eqref{eqn:reverse_sde_p} is
\begin{equation}
\label{eqn:reverse_sde_ode}
    \frac{\dxv_t}{\dt}=\fv(\x_t,t)-g(t)^2\sv_\theta(\x_t,t)+\frac{1}{2}g(t)^2\nabla_{\x}\log p_t^{\SDE}(\x_t),
\end{equation}
where the distribution of $\x_t$ during the trajectory is also $p_t^{\SDE}(\x_t)$. Below we show that for the commonly-used SGMs, in most cases, the probability flow ODE in Eqn.~\eqref{eqn:reverse_sde_ode} of the {\sde} is different from the {\ode} in Eqn.~\eqref{eqn:score_ode}.
\begin{proposition}
\label{prop:sde_gaussian}
Assume $\fv(\x_t,t)=\alpha(t)\x_t$ is a linear function of $\x_t$ for all $\x_t\in\R^d$ and $t\in[0,T]$, where $\alpha(t)\in\R$. If for all $\x_t\in\R^d$ and $t\in[0,T]$, $\sv_\theta(\x_t,t)\equiv\nabla_{\x}\log p^{\SDE}_t(\x_t)$, then $p_t^{\SDE}$ is a Gaussian distribution for all $t\in[0,T]$.
\end{proposition}
\begin{proof}
If $\s_\theta(\cdot,t)\equiv p^{\SDE}_t$, then for any $\x\in\R^d$, by Fokker-Planck equation, we have
\begin{equation}
\label{eqn:fokker_planck_sde_ode}
    \pdv{p_t^{\SDE}(\x)}{t}=\nabla_{\x}\cdot \left(\left(\fv(\x,t)-\frac{1}{2}g(t)^2\nabla_{\x}\log p_t^{\SDE}(\x)\right)p_t^{\SDE}(\x)\right).
\end{equation}
On the other hand, if we start the forward process in Eqn.~\eqref{eqn:forward_sde} with $\x_0\sim p^{\SDE}_0$, and the distribution of $\x_t$ at time $t$ during its trajectory follows the same equation as Eqn.~\eqref{eqn:fokker_planck_sde_ode}. Therefore, the distribution of $\x_t$ is also $p_t^{\SDE}$.

As $\fv(\x_t,t)=\alpha(t)\x_t$ is linear to $\x_t$ and $g(t)$ is independent of $\x_t$, there exists a function $\mu(t):\R\rightarrow\R$ and a function $\tilde\sigma(t):\R\rightarrow\R_{>0}$, such that $p^{\SDE}(\x_t|\x_0)=\Nc(\x_t|\mu(t)\x_0,\tilde\sigma(t)^2\Iv)$. Therefore, for any $\x_T\sim p_T^{\SDE}(\x_T)$, there exists a random variable $\epsilonv_T$ which is independent of $\x_0$, such that $\epsilonv\sim\N(\vect{0},\Iv)$ and $\x_T=\mu(t)\x_0+\sigma(t)\epsilonv_T$. As $\x_T\sim\N(\vect{0},\sigma_T^2\Iv)$, by Cramér's decomposition theorem, because the random $\x_T$ is normally distributed and admits a decomposition as a sum of two independent random variables, we can conclude that $\mu(t)\x_0$ also follows a Gaussian distribution. As $\mu(t)$ is a scalar function independent of $\x_t$, we can further conclude that $\x_0$ follows a Gaussian distribution. As $p^{\SDE}(\x_t|\x_0)=\Nc(\x_t|\mu(t)\x_0,\tilde\sigma(t)^2\Iv)$ is a Gaussian distribution whose mean is linear to $\x_0$ and covariance is independent of $\x_0$, we have $p_t^{\SDE}(\x_t)$ is a Gaussian distribution for all $t$.
\end{proof}

The assumption for $\fv(\x_t,t)$ is true for the common SGMs, because for VPSDE, VESDE and subVPSDE, $\fv(\x_t,t)$ are all linear to $\x_t$, which enables fast sampling for denoising score matching~\cite{song2020score}. In most cases, $p_0^{\SDE}$ is not a Gaussian distribution because we want to minimize $\kl{q_0}{p_0^{\SDE}}$. Therefore, there exists $t\in[0,T]$ such that $\s_\theta(\cdot,t)\neq \nabla_{\x}\log p_t^{\SDE}$, and then the probability flow ODE in Eqn.~\eqref{eqn:reverse_sde_ode} is different from the {\ode} in Eqn.~\eqref{eqn:score_ode}. In conclusion, in most cases, $p^{\SDE}_0\neq p^{\ODE}_0$, and minimizing $\Jc_{\SM}(\theta)$ is minimizing $\kl{q_0}{p_0^{\SDE}}$, but is not necessarily minimizing $\kl{q_0}{p_0^{\ODE}}$.

\begin{remark}
The "variational gap" of {\sde}s in~\cite{huang2021variational} is the gap between the ``joint distribution" KL divergence of $p^{\SDE}_{0:T}$ and ``marginal distribution" KL divergence of $p_0^{\SDE}$, which does not include $p^{\ODE}_t$. We refer to Appendix~\ref{appendix:variational_gap} for further discussions.
\end{remark}
\begin{remark}
When $\Jc_{\SM}(\theta)=0$, we have $\s_\theta(\cdot,t)\equiv\nabla_{\x}\log q_t$. In this case, the {\sde} in Eqn.~\eqref{eqn:reverse_sde_p} becomes the reverse diffusion SDE in Eqn.~\eqref{eqn:reverse_sde_q}, and the {\ode} in Eqn.~\eqref{eqn:score_ode} becomes the probability flow ODE in Eqn.~\eqref{eqn:forward_ode}. However, if $\fv(\x_t,t)$ is linear to $\x_t$ and $q_0(\x_0)$ is not Gaussian, we have $q_T$ is not Gaussian, so $q_T\neq p_T^{\SDE}$ and $q_T\neq p_T^{\ODE}$. Therefore, in this case, we still have $\s_\theta(\cdot,T)=\nabla_{\x}\log q_T\neq \nabla_{\x}\log p_T^{\SDE}$, so the probability flow in Eqn.~\eqref{eqn:reverse_sde_ode} of the {\sde} is still different from the {\ode} in Eqn.~\eqref{eqn:score_ode}, leading to the fact that $p_0^{\SDE}\neq p_0^{\ODE}$. (But the difference is extremely small, because they are both very similar to $q_0$).
\end{remark}

\section{KL divergence and variational gap of {\sde}s}
\label{appendix:variational_gap}
~\cite{song2021maximum} give an upper bound of marginal KL divergence $\kl{q_0}{p_0^{\SDE}}$ of {\sde} by calculating the joint KL divergence $\kl{q_{0:T}}{p_{0:T}^{\SDE}}$ (path measure) using Girsanov Theorem and Itô integrals. The upper bound is:
\begin{equation}
\label{eqn:joint_kl_appendix}
    \kl{q_0}{p_0^{\SDE}}\leq \kl{q_{0:T}}{p_{0:T}^{\SDE}}=\kl{q_T}{p_T^{\SDE}}+\frac{1}{2}\int_0^T g(t)^2\E_{q_t(\x_t)}\left[\|\s_\theta(\x_t,t)-\nabla_{\x}\log q_t(\x_t) \|_2^2\right] \dt
\end{equation}
In this section, we derive the KL divergence $\kl{q_0}{p_0^{\SDE}}$ and the variational gap of the {\sde}. We first propose the KL divergence $\kl{q_0}{p_0^{\SDE}}$ in the following proposition.
\begin{proposition}
\label{prop:kl_sde}
The KL divergence $\kl{q_0}{p_0^{\SDE}}$ of the {\sde} is
\small{
\begin{equation}
\label{variational_gap:marginal_kl}
    \kl{q_0}{p_0^{\SDE}}=\kl{q_T}{p_T^{\SDE}}+\frac{1}{2}\int_0^T g(t)^2\E_{q_t(\x_t)}\left[\|\s_\theta(\x_t,t)-\nabla_{\x}\log q_t(\x_t) \|_2^2-\|\s_\theta(\x_t,t)-\nabla_{\x}\log p_t^{\SDE}(\x_t) \|_2^2\right] \dt
\end{equation}
}
\end{proposition}
\begin{proof}
By Eqn.~\eqref{eqn:reverse_sde_ode}, we denote the probability flow ODE of {\sde} as
\begin{equation}
    \frac{\dxv_t}{\dt}=\hv_{p^{\SDE}}(\x_t,t)\coloneqq \fv(\x_t,t)-g(t)^2\sv_\theta(\x_t,t)+\frac{1}{2}g(t)^2\nabla_{\x}\log p_t^{\SDE}(\x_t),
\end{equation}
First we rewrite the KL divergence from $q_0$ to $p_0^{\SDE}$ in an integral form
\begin{align}
    \nonumber
    \kl{q_0}{p_0^{\SDE}}&=\kl{q_T}{p_T^{\SDE}}-\kl{q_T}{p_T^{\SDE}}+\kl{q_0}{p_0^{\SDE}}\\
    &=\kl{q_T}{p_T^{\SDE}}-\int_0^T\frac{\partial \kl{q_t}{p_t^{\SDE}}}{\partial t}\dt\label{thrm_kl_sde:integral:form}
\end{align}
Given a fixed $\x$, by the special case of Fokker-Planck equation with zero diffusion term, we can derive the time-evolution of ODE's associated probability density function by:
\begin{align}
    \frac{\partial q_t(\x)}{\partial t}=-\nabla_{\x}\cdot(\hv_q(\x,t)q_t(\x)),\quad \frac{\partial p_t^{\SDE}(\x)}{\partial t}=-\nabla_{\x}\cdot(\hv_{p^{\SDE}}(\x,t)p_t^{\SDE}(\x))
\end{align}
Then we can expand the time-derivative of $\kl{q_t}{p_t^{\SDE}}$ as
\begin{align}
    \frac{\partial \kl{q_t}{p_t^{\SDE}}}{\partial t}&=\frac{\partial}{\partial t}\int q_t(\x)\log\frac{q_t(\x)}{p_t^{\SDE}(\x)}\dxv \nonumber\\
    &=\int\frac{\partial q_t(\x)}{\partial t}\log\frac{q_t(\x)}{p_t^{\SDE}(\x)}\dxv+\underbrace{\int\frac{\partial q_t(\x)}{\partial t}\dxv}_{=0}-\int\frac{q_t(\x)}{p_t^{\SDE}(\x)}\frac{\partial p_t^{\SDE}(\x)}{\partial t}\dxv \nonumber\\
    &=-\int\nabla_{\x}\cdot(\hv_q(\x,t)q_t(\x))\log\frac{q_t(\x)}{p_t^{\SDE}(\x)}\dxv+\int\frac{q_t(\x)}{p_t^{\SDE}(\x)}\nabla_{\x}\cdot(\hv_{p^{\SDE}}(\x,t)p_t^{\SDE}(\x))\dxv \nonumber\\
    &=\int(\hv_q(\x,t)q_t(\x))^\top\nabla_{\x}\log\frac{q_t(\x)}{p_t^{\SDE}(\x)}\dxv-\int(\hv_{p^{\SDE}}(\x,t)p_t^{\SDE}(\x))^\top\nabla_{\x}\frac{q_t(\x)}{p_t^{\SDE}(\x)}\dxv\label{thrm_kl_sde:parts}\\
    &=\int q_t(\x)\left[\hv_q^\top(\x,t)-\hv_{p^{\SDE}}^\top(\x,t)\right]\left[\nabla_{\x}\log q_t(\x)-\nabla_{\x}\log p_t^{\SDE}(\x)\right]\dxv\nonumber \\
    &=-\frac{1}{2}g(t)^2\E_{q_t(\x)}\Big[\big(\nabla_{\x}\log q_t(\x)+\nabla_{\x}\log p_t^{\SDE}(\x)-2\s_\theta(\x,t)\big)^\top\!\big(\nabla_{\x}\log q_t(\x)-\nabla_{\x}\log p_t^{\SDE}(\x)\big)\Big] \nonumber\\
    &=-\frac{1}{2}g(t)^2\E_{q_t(\x)}\left[\|\s_\theta(\x,t)-\nabla_{\x}\log q_t(\x) \|_2^2-\|\s_\theta(\x,t)-\nabla_{\x}\log p_t^{\SDE}(\x) \|_2^2\right] \nonumber
\end{align}
where Eqn.~\eqref{thrm_kl_sde:parts} is due to integration by parts under the Assumption.~\ref{assumption:all}(11), which shows that $\lim_{\x\rightarrow\infty}\hv_q(\x,t)q_t(\x)=0$ and $\lim_{\x\rightarrow\infty}\hv_{p^{\SDE}}(\x,t)p_t^{\SDE}(\x)=0$ for all $t\in[0,T]$. Combining with Eqn.~\eqref{thrm_kl_sde:integral:form}, we can finish the proof.
\end{proof}

Thus by Eqn.~\eqref{eqn:joint_kl_appendix} and Proposition~\ref{prop:kl_sde}, the expectation of the variational gap for $p_0^{\SDE}$ under data distribution $q_0$ is
\begin{align}
    \nonumber
    \mathrm{Gap}(q_0,p^{\SDE}_0)&\coloneqq \kl{q_{0:T}}{p_{0:T}^{\SDE}}-\kl{q_0}{p_0^{\SDE}}\\
    &=\frac{1}{2}\int_0^T g(t)^2\E_{q_t(\x_t)}\left[\|\s_\theta(\x_t,t)-\nabla_{\x}\log p_t^{\SDE}(\x_t) \|_2^2\right] \dt
\end{align}
Note that this is equivalent to the expectation of the variational gap in~\cite{huang2021variational}. This expression tells us that
\begin{align}
    \mathrm{Gap}(q_0,p_0^{\SDE})=0\Leftrightarrow \forall \x_t\in\R^d, t\in[0,T], \s_\theta(\x_t,t)\equiv\nabla_{\x}\log p_t^{\SDE}(\x_t).
\end{align}
Therefore, in practice, we always have $\mathrm{Gap}(q_0,p_0^{\SDE})>0$ since according to Proposition~\ref{prop:sde_gaussian}, $\mathrm{Gap}(q_0,p_0^{\SDE})=0$ means that $p_t^{\SDE}$ is a Gaussian distribution for all $t\in[0,T]$, which doesn't match real data.

Besides, we should notice that the variational gap $\mathrm{Gap}(q_0,p_0^{\SDE})$ is not necessarily related to the KL divergence of {\ode}, because in practice $p_0^{\SDE}\neq p_0^{\ODE}$.

\section{Instantaneous change of score function of ODEs}
\label{appendix:change_of_score}
\begin{theorem}
\label{thrm:change_of_score}
(Instantaneous Change of Score Function). Let $\z(t)\in\R^{d}$ be a finite continuous random variable with probability $p(\z(t),t)$ describing by an ordinary differential equation $\frac{\dzv(t)}{\dt}=\fv(\z(t),t)$. Assume that $\fv\in C^3(\R^{d+1})$, then for each time $t$, the score function $\nabla_{\z}\log p(\z(t),t)$ follows a differential equation:
\begin{equation}
\begin{split}
    \frac{\mathrm{d}\nabla_{\z}\log p(\z(t),t)}{\dt}=&-\nabla_{\z}\left(\tr(\nabla_{\z}\fv(\z(t),t))\right)\\
    &-\left(\nabla_{\z}\fv(\z(t),t)\right)^\top\!\nabla_{\z}\log p(\z(t),t)
\end{split}
\end{equation}
\end{theorem}
\begin{proof}
Given a fixed $\x$, by the special case of Fokker-Planck equation with zero diffusion term, we have
\begin{align}
    \frac{\partial p(\z,t)}{\partial t}&=-\nabla_{\z}\cdot(\fv(\z,t)p(\z,t))
\end{align}
So
\begin{align}
    \frac{\partial\log p(\z,t)}{\partial t}&=\frac{1}{p(\z,t)}\frac{\partial p(\z,t)}{\partial t}\\
    &=-\nabla_{\z}\cdot \fv(\z,t)-\fv(\z,t)^\top\nabla_{\z}\log p(\z,t)
\end{align}
and
\begin{align}
    \frac{\partial\nabla_{\z}\log p(\z,t)}{\partial t}&=\nabla_{\z}\frac{\partial\log p(\z,t)}{\partial t}\\
    &=-\nabla_{\z}\left(\nabla_{\z}\cdot \fv(\z,t)\right)-\nabla_{\z}\left(\fv(\z,t)^\top\nabla_{\z}\log p(\z,t)\right)
\end{align}
Then assume $\z(t)$ follows the trajectory of ODE $\frac{\dm\z(t)}{\dt}=\fv(\z(t),t)$, the total derivative of score function $\nabla_{\z}\log p(\z(t),t)$ w.r.t. $t$ is
\begin{align}
    \frac{\dm \nabla_{\z}\log p(\z(t),t)}{\dm t}&=\frac{\partial \nabla_{\z}\log p(\z(t),t)}{\partial \z}\frac{\dm\z(t)}{\dm t}+\frac{\partial \nabla_{\z}\log p(\z(t),t)}{\partial t}\\
    &=\nabla^2_{\z}\log p(\z(t),t)\fv(\z(t),t)-\nabla_{\z}\left(\nabla_{\z}\cdot \fv(\z(t),t)\right)-\nabla_{\z}\Big(\fv(\z(t),t)^\top\nabla_{\z}\log p(\z(t),t)\Big)\\
    &=-\nabla_{\z}\left(\nabla_{\z}\cdot \fv(\z(t),t)\right)-\left(\nabla_{\z}\fv(\z(t),t)\right)^\top\nabla_{\z}\log p(\z(t),t)
\end{align}
\end{proof}

By Theorem~\ref{thrm:change_of_score}, we can compute $\nabla_{\x}\log p^{\ODE}_0(\x_0)$ by firstly solving the forward ODE $\frac{\dxv_t}{\dt}=\hv_p(\x_t,t)$ from $0$ to $T$ to get $\x_T$, then solving a reverse ODE of $(\x_t, \nabla_{\x}\log p^{\ODE}_t(\x_t))$ with the initial value $(\x_T,\nabla_{\x}\log p^{\ODE}_T(\x_T))$ from $T$ to $0$.

\section{KL divergence of {\ode}s}
In this section, we propose the proofs for Sec.~\ref{sec:ode_likelihood}.
\subsection{Proof of Theorem~\ref{thrm:kl-ode}}
\label{appendix:thrm_kl_ode}
\begin{proof}
First we rewrite the KL divergence from $q_0$ to $p_0^{\ODE}$ in an integral form
\begin{align}
    \kl{q_0}{p_0^{\ODE}}&=\kl{q_T}{p_T^{\ODE}}-\kl{q_T}{p_T^{\ODE}}+\kl{q_0}{p_0^{\ODE}}\\
    &=\kl{q_T}{p_T^{\ODE}}-\int_0^T\frac{\partial \kl{q_t}{p_t^{\ODE}}}{\partial t}\dt.\label{thrm_kl_ode:integral:form}
\end{align}
Given a fixed $\x$, by the special case of Fokker-Planck equation with zero diffusion term, we can derive the time-evolution of ODE's associated probability density function by:
\begin{align}
    \frac{\partial q_t(\x)}{\partial t}=-\nabla_{\x}\cdot(\hv_q(\x,t)q_t(\x)),\quad \frac{\partial p_t^{\ODE}(\x)}{\partial t}=-\nabla_{\x}\cdot(\hv_p(\x,t)p_t^{\ODE}(\x)).
\end{align}
Then we can expand the time-derivative of $\kl{q_t}{p_t^{\ODE}}$ as
\begin{align}
    \frac{\partial \kl{q_t}{p_t^{\ODE}}}{\partial t}&=\frac{\partial}{\partial t}\int q_t(\x)\log\frac{q_t(\x)}{p_t^{\ODE}(\x)}\dxv\\
    &=\int\frac{\partial q_t(\x)}{\partial t}\log\frac{q_t(\x)}{p_t^{\ODE}(\x)}\dxv+\underbrace{\int\frac{\partial q_t(\x)}{\partial t}\dxv}_{=0}-\int\frac{q_t(\x)}{p_t^{\ODE}(\x)}\frac{\partial p_t^{\ODE}(\x)}{\partial t}\dxv\\
    &=-\int\nabla_{\x}\cdot(\hv_q(\x,t)q_t(\x))\log\frac{q_t(\x)}{p_t^{\ODE}(\x)}\dxv+\int\frac{q_t(\x)}{p_t^{\ODE}(\x)}\nabla_{\x}\cdot(\hv_p(\x,t)p_t^{\ODE}(\x))\dxv\\
    &=\int(\hv_q(\x,t)q_t(\x))^\top\nabla_{\x}\log\frac{q_t(\x)}{p_t^{\ODE}(\x)}\dxv-\int(\hv_p(\x,t)p_t^{\ODE}(\x))^\top\nabla_{\x}\frac{q_t(\x)}{p_t^{\ODE}(\x)}\dxv\label{thrm_kl_ode:parts}\\
    &=\int q_t(\x)\left[\hv_q^\top(\x,t)-\hv_p^\top(\x,t)\right]\left[\nabla_{\x}\log q_t(\x)-\nabla_{\x}\log p_t^{\ODE}(\x)\right]\dxv\\
    &=-\frac{1}{2}g(t)^2\E_{q_t(\x)}\left[\left(\s_\theta(\x,t)-\nabla_{\x}\log q_t(\x)\right)^\top\left(\nabla_{\x}\log p_t^{\ODE}(\x)-\nabla_{\x}\log q_t(\x)\right)\right],
\end{align}
where Eqn.~\eqref{thrm_kl_ode:parts} is due to integration by parts under the Assumption.~\ref{assumption:all}(11), which shows that $\lim_{\x\rightarrow\infty}\hv_q(\x,t)q_t(\x)=0$ and $\lim_{\x\rightarrow\infty}\hv_{p}(\x,t)p_t^{\ODE}(\x)=0$ for all $t\in[0,T]$. Combining with Eqn.~\eqref{thrm_kl_ode:integral:form}, we can conclude that
\begin{equation}
\begin{aligned}
    &\kl{q_0}{p_0^{\ODE}}\\
    &=\kl{q_T}{p_T^{\ODE}}+\frac{1}{2}\int_0^T g(t)^2\E_{q_t(\x_t)}\left[\left(\s_\theta(\x_t,t)-\nabla\log q_t(\x_t)\right)^\top\left(\nabla\log p_t^{\ODE}(\x_t)-\nabla\log q_t(\x_t)\right)\right]\dt\\
    &=\kl{q_T}{p_T^{\ODE}}+\Jc_{\ODE}(\theta)\\
    &=\kl{q_T}{p^{\ODE}_T}+\Jc_{\SM}(\theta)+\Jc_{\DIFF}(\theta).
\end{aligned}
\end{equation}
\end{proof}

\subsection{Proof of Theorem~\ref{thrm:fisher_bound}}
\label{appendix:thrm_fisher_bound}
Firstly, we propose a Lemma for computing $\nabla\log p_t^{\ODE}(\x_t)-\nabla\log q_t(\x_t)$ with  the trajectory of $\frac{\dxv_t}{\dt}=\hv_q(\x_t,t)$.
\begin{lemma}
\label{lemma:ode_score_another_trajectory}
Assume $\frac{\dxv_t}{\dt}=\hv_q(\x_t,t)$, $p_t^{\ODE}$ and $q_t$ are defined as the same in Sec.~\ref{sec:background}. We have
\begin{equation}
\begin{aligned}
    \frac{\dm(\nabla\log p^{\ODE}_t-\nabla\log q_t)}{\dt}&=-\left(\nabla\tr(\nabla\hv_{p}(\x_t,t))-\nabla\tr(\nabla\hv_{q}(\x_t,t))\right)\\
    &\quad-\left(\nabla\hv_p(\x_t,t)^\top\nabla\log p^{\ODE}_t(\x_t)-\nabla\hv_q(\x_t,t)^\top\nabla\log q_t(\x_t)\right)\\
    &\quad-\nabla^2\log p^{\ODE}_t(\x_t)\big(\hv_p(\x_t,t)-\hv_q(\x_t,t)\big)
\end{aligned}
\end{equation}
\end{lemma}
\begin{proof}
(Proof of Lemma~\ref{lemma:ode_score_another_trajectory})

Given a fixed $\x$, by the special case of Fokker-Planck equation with zero diffusion term, we have
\begin{align}
    \frac{\partial p_t^{\ODE}(\x)}{\partial t}&=-\nabla\cdot(\hv_p(\x_t,t)p_t^{\ODE}(\x))
\end{align}
So
\begin{align}
    \frac{\partial\log p_t^{\ODE}(\x)}{\partial t}&=\frac{1}{p_t^{\ODE}(\x)}\frac{\partial p_t^{\ODE}(\x)}{\partial t}\\
    &=-\nabla\cdot \hv_p(\x_t,t)-\hv_p(\x_t,t)^\top\nabla\log p_t^{\ODE}(\x)
\end{align}
and
\begin{align}
    \frac{\partial\nabla\log p_t^{\ODE}(\x)}{\partial t}&=\nabla\frac{\partial\log p_t^{\ODE}(\x)}{\partial t}\\
    &=-\nabla\left(\nabla\cdot \hv_p(\x_t,t)\right)-\nabla\left(\hv_p(\x_t,t)^\top\nabla\log p_t^{\ODE}(\x)\right)
\end{align}
As $\x_t$ follows the trajectory of $\frac{\dm\x_t}{\dt}=\hv_q(\x_t,t)$, the total derivative of score function $\nabla\log p_t^{\ODE}(\x_t)$ w.r.t. $t$ is
\begin{align}
    \frac{\dm \nabla\log p_t^{\ODE}(\x_t)}{\dm t}&=\frac{\partial \nabla\log p_t^{\ODE}(\x_t)}{\partial \x_t}\frac{\dm\x_t}{\dm t}+\frac{\partial \nabla\log p_t^{\ODE}(\x_t)}{\partial t}\\
    &=\nabla^2\log p_t^{\ODE}(\x_t)\hv_q(\x_t,t)-\nabla\left(\nabla\cdot \hv_p(\x_t,t)\right)-\nabla\left(\hv_p(\x_t,t)^\top\nabla\log p_t^{\ODE}(\x)\right)\\
    &=-\nabla\left(\nabla\cdot \hv_p(\x_t,t)\right)-\nabla\hv_p(\x_t,t)^\top\nabla\log p^{\ODE}_t(\x_t)-\nabla^2\log p^{\ODE}_t(\x_t)\left(\hv_p(\x_t,t)-\hv_q(\x_t,t)\right)
\end{align}
Similarly, for $\nabla\log q_t(\x_t)$ we have
\begin{align}
    \frac{\dm \nabla\log q_t(\x_t)}{\dm t}&=-\nabla\left(\nabla\cdot \hv_q(\x_t,t)\right)-\nabla\hv_q(\x_t,t)^\top\nabla\log q_t(\x_t)
\end{align}
Therefore, by combining $\frac{\dm \nabla\log p_t^{\ODE}(\x_t)}{\dm t}$ and $\frac{\dm \nabla\log q_t(\x_t)}{\dm t}$, we can derive the conclusion.
\end{proof}

Then we prove Theorem~\ref{thrm:fisher_bound} below.
\begin{proof}
(Proof of Theorem~\ref{thrm:fisher_bound}.)

Firstly, we have
\begin{align}
    \hv_p(\x_t,t)-\hv_q(\x_t,t)=-\frac{1}{2}g(t)^2\big(\s_\theta(\x_t,t)-\nabla\log q_t(\x_t)\big)
\end{align}
So we have
\begin{align}
    &\|\hv_p(\x_t,t)-\hv_q(\x_t,t)\|_2\leq \frac{1}{2}g(t)^2\delta_1,\\
    &\|\nabla\hv_p(\x_t,t)-\nabla\hv_q(\x_t,t)\|_2\leq \|\nabla\hv_p(\x_t,t)-\nabla\hv_q(\x_t,t)\|_F\leq\frac{1}{2}g(t)^2\delta_2,\\
    &\|\nabla\tr(\nabla\hv_p(\x_t,t))-\nabla\tr(\nabla\hv_q(\x_t,t))\|_2\leq \frac{1}{2}g(t)^2\delta_3,
\end{align}
where $\|\cdot\|_2$ is the $L_2$-norm for vectors and the induced $2$-norm for matrices.

Assume that $\frac{\dxv_t}{\dt}=\hv_q(\x_t,t)$, by Lemma~\ref{lemma:ode_score_another_trajectory}, we have
\begin{align}
    \nabla\log p_t^{\ODE}(\x_t)-&\nabla\log q_t(\x_t)=\nabla\log p^{\ODE}_T(\x_T)-\nabla\log q_T(\x_T)\\
    &\ +\int_t^T \Big(\nabla\tr(\nabla\hv_{p}(\x_s,s))-\nabla\tr(\nabla\hv_{q}(\x_s,s))\Big)\dm s\\
    &\ +\int_t^T \Big(\nabla\hv_p(\x_s,s)^\top\nabla\log p^{\ODE}_s(\x_s)-\nabla\hv_q(\x_s,s)^\top\nabla\log q_s(\x_s)\Big)\dm s\\
    &\ +\int_t^T \nabla^2\log p^{\ODE}_s(\x_s)\big(\hv_p(\x_s,s)-\hv_q(\x_s,s)\big)\dm s
\end{align}
For simplicity, we denote $\hv_p(s)\coloneqq \hv_p(\x_s,s)$, $\hv_q(s)\coloneqq \hv_q(\x_s,s)$, $p_s\coloneqq p_s^{\ODE}(\x_s)$ and $q_s\coloneqq q_s(\x_s)$.
We use $\|\cdot\|$ to denote the 2-norm for vectors and matrices. As
\begin{align}
    \nabla\hv_p(s)^\top\nabla\log p_s-&\nabla\hv_q(s)^\top\nabla\log q_s=\big(\nabla\hv_p(s)-\nabla\hv_q(s)\big)^\top\big(\nabla\log p_s-\nabla\log q_s\big)\\
    &+\nabla\hv_q(s)^\top\big(\nabla\log p_s-\nabla\log q_s\big)+\big(\nabla\hv_p(s)-\nabla\hv_q(\s)\big)^\top\nabla\log q_s
\end{align}
We have
\begin{align}
    \|\nabla\log p_t-&\nabla\log q_t\|\leq \|\nabla\log p_T-\nabla\log q_T\|+\int_t^T \big\|\nabla\tr(\nabla\hv_{p}(s))-\nabla\tr(\nabla\hv_{q}(s))\big\|\dm s\\
    &\ +\int_t^T \left(\big\|\nabla\hv_p(s)-\nabla\hv_q(s)\big\|+\big\|\nabla\hv_q(s)\big\|\right)\cdot\big\|\nabla\log p_s-\nabla\log q_s\big\|\dm s\\
    &\ +\int_t^T \big\|\nabla\hv_p(s)-\nabla\hv_q(s)\big\|\cdot\big\|\nabla\log q_s\big\|\dm s+\int_t^T \|\nabla^2\log p_s\|\cdot\left\|\hv_p(s)-\hv_q(s)\right\|\dm s\\
    \leq\ &\|\nabla\log p_T-\nabla\log q_T\|+\frac{1}{2}\int_t^T g(s)^2\Big(\delta_3+\delta_2\|\nabla\log q_s\|+\delta_1 C\Big)\dm s\\
    &\ +\int_t^T \left(\frac{\delta_2}{2}g(s)^2+\|\nabla\hv_q(s)\|\right)\|\nabla\log p_s-\nabla\log q_s\|\dm s
\end{align}
Denote
\begin{align}
    u(t)&\coloneqq \|\nabla\log p_t-\nabla\log q_t\|\\
    \alpha(t)&\coloneqq \|\nabla\log p_T-\nabla\log q_T\|+\frac{1}{2}\int_t^T g(s)^2\Big(\delta_3+\delta_2\|\nabla\log q_s\|+\delta_1 C\Big)\dm s\\
    \beta(t)&\coloneqq\frac{\delta_2}{2}g(t)^2+\|\nabla\hv_q(t)\|
\end{align}
Then $\alpha(t)\geq 0, \beta(t)\geq 0$ are independent of $\theta$, and we have
\begin{align}
    u(t)\leq \alpha(t)+\int_t^T\beta(s)u(s)\dm s.
\end{align}
By Grönwall's inequality, we have
\begin{align}
    u(t)\leq \alpha(t)+\int_t^T\alpha(s)\beta(s)\exp\left(\int_t^s\beta(r)\dm r\right)\dm s,
\end{align}
and therefore,
\begin{align}
    \fisher{q_t}{p_t}=\E_{q_t}[u(t)^2]\leq \E_{q_t}\left[\left(\alpha(t)+\int_t^T\alpha(s)\beta(s)\exp\left(\int_t^s\beta(r)\dm r\right)\dm s\right)^2\right],
\end{align}
where the r.h.s. is only dependent on $\delta_1,\delta_2,\delta_3$ and $C$, with some constants that are only dependent on the forward process $q$. Furthermore, as r.h.s. is an increasing function of $\alpha(t)\geq 0$ and $\beta(t)\geq 0$, and $\alpha(t),\beta(t)$ are increasing functions of $\delta_1,\delta_2,\delta_3,C$, so we can minimize $\delta_1,\delta_2,\delta_3$ to minimize an upper bound of $\fisher{q_t}{p_t}$.%
\end{proof}

\subsection{Maximum likelihood training of {\ode} by ODE solvers}
\label{appendix:mle_scoreode_by_odesolver}
By Eqn.~\eqref{eqn:change_of_variable_score_ode}, we have
\begin{equation}
    \log p_0^{\ODE}(\x_0)=\log p_T^{\ODE}(\x_T)+\int_0^T\tr(\nabla_{\x}\hv_p(\x_t,t))\dt,
\end{equation}
where $\frac{\dxv_t}{\dt}=\hv_p(\x_t,t)$.
So minimizing $\kl{q_0}{p_0^{\ODE}}$ is equivalent to maximizing
\begin{align}
    \E_{q_0(\x_0)}\left[\log p_0^{\ODE}(\x_0)\right]=\E_{q_0(\x_0)}\left[\log p_T^{\ODE}(\x_T)+\int_0^T\tr(\nabla_{\x}\hv_p(\x_t,t))\dt\right].
\end{align}
We can directly call ODE solvers to compute $\log p_0^{\ODE}(\x_0)$ for a given data point $\x_0$, and thus do maximum likelihood training~\cite{chen2018neural, grathwohl2018ffjord}. However, this needs to be done at every optimization step, which is hard to scale up for the large neural networks used in SGMs. For example, it takes $2\sim 3$ minutes for evaluating $\log p^{\ODE}_0$ for a single batch of the {\ode} used in~\cite{song2020score}. Besides, directly maximum likelihood training for $p^{\ODE}_0$ cannot make sure that the optimal solution for $\s_\theta(\x_t,t)$ is the data score function $\nabla_{\x}\log q_t(\x_t)$, because the directly MLE for $p_0^{\ODE}$ cannot ensure that the model distribution $p_t^{\ODE}$ is also similar to the data distribution $q_t$ at each time $t\in(0,T)$, thus cannot ensure the model ODE function $\hv_p(\x_t,t)$ is similar to the data ODE function $\hv_q(\x_t,t)$. In fact, {\ode} is a special formulation of Neural ODEs~\cite{chen2018neural}. Empirically, \citet{finlay2020train} find that directly maximum likelihood training of Neural ODEs may cause rather complex dynamics ($\hv_p(\x_t,t)$ here), which indicates that directly maximum likelihood training by ODE solvers cannot ensure $\hv_p(\x_t,t)$ be similar to $\hv_q(\x_t,t)$, and thus cannot ensure the score model $\s_\theta(\x_t,t)$ be similar to the data score function $\nabla_{\x}\log q_t(\x_t)$.

\section{Error-bounded High-Order Denoising Score Matching}
\label{appendix:proof_high_order_DSM}
In this section, we present all the lemmas and proofs for our proposed error-bounded high-order DSM algorithm.

Below we propose an expectation formulation of the first-order score function.
\begin{lemma}
\label{lemma:first}
Assume $(\x_t,\x_0)\sim q(\x_t,\x_0)$, we have
\begin{align}
    \nabla_{\x_t}\log q_t(\x_t)=\E_{q_{t0}(\x_0|\x_t)}\left[\nabla_{\x_t}\log q_{0t}(\x_t|\x_0)\right]
\end{align}
\end{lemma}
\begin{proof}
(Proof of Lemma~\ref{lemma:first})

We use $\nabla(\cdot)$ to denote the derivative of $\x_t$, namely $\nabla_{\x_t}(\cdot)$.
\begin{align}
    \nabla\log q_t(\x_t)&=\frac{\nabla q_t(\x_t)}{q_t(\x_t)}\\
    &=\frac{\nabla\int q_{0t}(\x_t|\x_0)q_0(\x_0)\dxv_0}{q_t(\x_t)}\\
    &=\int\frac{q_0(\x_0)q_{0t}(\x_t|\x_0)}{q_t(\x_t)}\frac{\nabla q_{0t}(\x_t|\x_0)}{q_{0t}(\x_t|\x_0)}\dxv_0\\
    &=\int q_{t0}(\x_0|\x_t)\nabla\log q_{0t}(\x_t|\x_0)\dxv_0\\
    &=\E_{q_{t0}(\x_0|\x_t)}\left[\nabla\log q_{0t}(\x_t|\x_0)\right]
\end{align}
\end{proof}
And below we propose an expectation formulation for the second-order score function.
\begin{lemma}
\label{lemma:second}
Assume $(\x_t,\x_0)\sim q(\x_t,\x_0)$, we have
\begin{equation}
\begin{aligned}
    &\nabla^2_{\x_t}\log q_t(\x_t)\\
    &=\E_{q_{t0}(\x_0|\x_t)}\Big[\nabla^2_{\x_t}\log q_{0t}(\x_t|\x_0)+\big(\nabla_{\x_t}\log q_{0t}(\x_t|\x_0)-\nabla_{\x_t}\log q_t(\x_t)\big)\big(\nabla_{\x_t}\log q_{0t}(\x_t|\x_0)-\nabla_{\x_t}\log q_t(\x_t)\big)^\top\Big]
\end{aligned}
\end{equation}
and
\begin{align}
    \tr(\nabla^2_{\x_t}\log q_t(\x_t))&=\E_{q_{t0}(\x_0|\x_t)}\Big[\tr(\nabla_{\x_t}^2\log q_{0t}(\x_t|\x_0))+\big\|\nabla_{\x_t}\log q_{0t}(\x_t|\x_0)-\nabla_{\x_t}\log q_t(\x_t)\big\|_2^2\Big]
\end{align}
\end{lemma}
\begin{proof}
(Proof of Lemma~\ref{lemma:second})

We use $\nabla(\cdot)$ to denote the derivative of $\x_t$, namely $\nabla_{\x_t}(\cdot)$. Firstly, the gradient of $q_{t0}$ w.r.t. $\x_t$ can be calculated as
\begin{align}
    \nabla q_{t0}(\x_0|\x_t)&=\nabla\frac{q_0(\x_0)q_{0t}(\x_t|\x_0)}{q_t(\x_t)}\\
    &=q_0(\x_0)\frac{q_t(\x_t)\nabla q_{0t}(\x_t|\x_0)-q_{0t}(\x_t|\x_0)\nabla q_t(\x_t)}{q_t(\x_t)^2}\\
    &=\frac{q_0(\x_0)q_{0t}(\x_t|\x_0)}{q_t(\x_t)}\left(\nabla\log q_{0t}(\x_t|\x_0)-\nabla\log q_t(\x_t)\right)\\
    &=q_{t0}(\x_0|\x_t)\left(\nabla\log q_{0t}(\x_t|\x_0)-\nabla\log q_t(\x_t)\right)\label{lemma:first:mid}
\end{align}
Then using Lemma~\ref{lemma:first} and the product rule, we have
\begin{align}
    \nabla^2\log q_t(\x_t)&=\nabla \E_{q_{t0}(\x_0|\x_t)}\left[\nabla\log q_{0t}(\x_t|\x_0)\right]\\
    &=\int q_{t0}(\x_0|\x_t)\nabla^2\log q_{0t}(\x_t|\x_0)+\nabla q_{t0}(\x_0|\x_t)\nabla\log q_{0t}(\x_t|\x_0)^\top\dxv_0\\
    &=\E_{q_{t0}(\x_0|\x_t)}\left[\nabla^2\log q_{0t}(\x_t|\x_0)+\left(\nabla\log q_{0t}(\x_t|\x_0)-\nabla\log q_t(\x_t)\right)\nabla\log q_{0t}(\x_t|\x_0)^\top\right]\label{lemma:second:mid}
\end{align}
From Lemma~\ref{lemma:first}, we also have $\E_{q_{t0}}\left[\nabla\log q_{0t}-\nabla\log q_t\right]=\nabla\log q_t-\nabla\log q_t=\vect{0}$, so
\begin{align}
    \E_{q_{t0}}\left[(\nabla\log q_{0t}-\nabla\log q_t)\nabla\log q_t^\top\right]&=\E_{q_{t0}}\left[\nabla\log q_{0t}-\nabla\log q_t\right]\nabla\log q_t^\top\\
    &=\vect{0}
\end{align}
Thus Eqn.~\eqref{lemma:second:mid} can be further transformed by subtracting $\vect{0}$ as
\begin{align}
&\nabla^2\log q_t(\x_t)\\
&=\E_{q_{t0}(\x_0|\x_t)}\left[\nabla^2\log q_{0t}(\x_t|\x_0)+\left(\nabla\log q_{0t}(\x_t|\x_0)-\nabla\log q_t(\x_t)\right)\nabla\log q_{0t}(\x_t|\x_0)^\top\right]\\
&\quad- \E_{q_{t0}(\x_0|\x_t)}\left[(\nabla\log q_{0t}(\x_t|\x_0)-\nabla\log q_t(\x_t))\nabla\log q_t(\x_t)^\top\right]\\
&=\E_{q_{t0}(\x_0|\x_t)}\left[\nabla^2\log q_{0t}(\x_t|\x_0)
    +\left(\nabla\log q_{0t}(\x_t|\x_0)-\nabla\log q_t(\x_t)\right)\left(\nabla\log q_{0t}(\x_t|\x_0)-\nabla\log q_t(\x_t)\right)^\top\right]
\end{align}
which completes the proof of second-order score matrix and its trace.
\end{proof}
Below we propose a corollary of an expectation formulation of the sum of the first-order and the second-order score functions, which can be used to design the third-order denoising score matching method.
\begin{corollary}
\label{coro:second_conditional_score}
Assume $(\x_t,\x_0)\sim q(\x_t,\x_0)$, we have
\begin{equation}
\begin{aligned}
    &\E_{q_{t0}(\x_0|\x_t)}\Big[\nabla_{\x_t}\log q_{0t}(\x_t|\x_0)\nabla_{\x_t}\log q_{0t}(\x_t|\x_0)^\top+\nabla_{\x_t}^2\log q_{0t}(\x_t|\x_0)\Big]\\
    &=\nabla_{\x_t}\log q_t(\x_t)\nabla_{\x_t}\log q_t(\x_t)^\top+\nabla_{\x_t}^2\log q_t(\x_t)
\end{aligned}
\end{equation}
and
\begin{align}
    \E_{q_{t0}(\x_0|\x_t)}\Big[\big\|\nabla_{\x_t}\log q_{0t}(\x_t|\x_0)\big\|_2^2+\tr(\nabla_{\x_t}^2\log q_{0t}(\x_t|\x_0))\Big]&=\big\|\nabla_{\x_t}\log q_t(\x_t)\big\|_2^2+\tr(\nabla_{\x_t}^2\log q_t(\x_t))
\end{align}
\end{corollary}
\begin{proof}
(Proof of Corollary~\ref{coro:second_conditional_score})

This is the direct corollary of Eqn.~\eqref{lemma:second:mid} by adding $\nabla_{\x_t}\log q_t(\x_t)\nabla_{\x_t}\log q_t(\x_t)^\top$ to both sides and using Lemma~\ref{lemma:first}.
\end{proof}
And below we propose an expectation formulation for the third-order score function.
\begin{lemma}
\label{lemma:third}
Assume $(\x_t,\x_0)\sim q(\x_t,\x_0)$ and $\nabla_{\x_t}^3\log q_{0t}(\x_t|\x_0)=\vect{0}$, we have
\begin{align}
    \nabla_{\x_t}\tr(\nabla_{\x_t}^2\log q_t(\x_t))=\E_{q_{t0}(\x_0|\x_t)}\Big[\big\|\nabla_{\x_t}\log q_{0t}(\x_t|\x_0)-\nabla_{\x_t}\log q_t(\x_t)\big\|_2^2\big(\nabla_{\x_t}\log q_{0t}(\x_t|\x_0)-\nabla_{\x_t}\log q_t(\x_t)\big)\Big]
\end{align}
\end{lemma}
\begin{proof}
(Proof of Lemma~\ref{lemma:third})

We use $\nabla(\cdot)$ to denote the derivative of $\x_t$, namely $\nabla_{\x_t}(\cdot)$. Firstly, according to the second-order score trace given by Lemma~\ref{lemma:second}, we have
\begin{equation}
\begin{aligned}
&\nabla\tr(\nabla^2\log q_t(\x_t))\\
&=\underbrace{\nabla\int q_{t0}(\x_0|\x_t)\tr\left(\nabla^2\log q_{0t}(\x_t|\x_0)\right)\dxv_0}_{(1)}+\underbrace{\nabla\int q_{t0}(\x_0|\x_t)\|\nabla\log q_{0t}(\x_t|\x_0)-\nabla\log q_t(\x_t)\|_2^2\dxv_0}_{(2)}
\end{aligned}
\end{equation}
Before simplifying them, we should notice two simple tricks. Firstly, due to the assumption $\nabla^3\log q_{0t}(\x_t|\x_0)=\vect{0}$, we know that $\nabla^2\log q_{0t}$ and $\tr\left(\nabla^2\log q_{0t}\right)$ are constants, so they can be drag out of the integral w.r.t $\x_0$ and the derivative w.r.t. $\x_t$. Secondly, using Lemma~\ref{lemma:first}, we have $\E_{q_{t0}}\left[\nabla\log q_{0t}-\nabla\log q_t\right]=\vect{0}$. Also, $\nabla q_{t0}(\x_0|\x_t)$ can be represented by Eqn.~\eqref{lemma:first:mid}. Thus
\begin{align}
    (1)&=\tr\left(\nabla^2\log q_{0t}(\x_t|\x_0)\right)\nabla\int q_{t0}(\x_0|\x_t)\dxv_0\\
    &=\tr\left(\nabla^2\log q_{0t}(\x_t|\x_0)\right)\nabla1\\
    &=\vect{0}
\end{align}
\begin{align}
    (2)&=\underbrace{\int \nabla q_{t0}(\x_0|\x_t)\|\nabla\log q_{0t}(\x_t|\x_0)\!-\!\nabla\log q_t(\x_t)\|_2^2\dxv_0}_{(3)}+\underbrace{\int q_{t0}(\x_0|\x_t)\nabla\|\nabla\log q_{0t}(\x_t|\x_0)-\nabla\log q_t(\x_t)\|_2^2\dxv_0}_{(4)}
\end{align}
where
\begin{align}
    (4)&=2\int q_{t0}(\x_0|\x_t)\left(\nabla^2\log q_{0t}(\x_t|\x_0)-\nabla^2\log q_t(\x_t)\right)\left(\nabla\log q_{0t}(\x_t|\x_0)-\nabla\log q_t(\x_t)\right)\dxv_0\\
    &=2\left(\nabla^2\log q_{0t}(\x_t|\x_0)-\nabla^2\log q_t(\x_t)\right)\int q_{t0}(\x_0|\x_t)\left(\nabla\log q_{0t}(\x_t|\x_0)-\nabla\log q_t(\x_t)\right)\dxv_0\\
    &=\vect{0}
\end{align}
Combining equations above, we have
\begin{align}
\nabla\tr(\nabla^2\log q_t(\x_t))&=(3)\\
&=\int q_{t0}(\x_0|\x_t)\|\nabla\log q_{0t}(\x_t|\x_0)-\nabla\log q_t(\x_t)\|_2^2\left(\nabla\log q_{0t}(\x_t|\x_0)-\nabla\log q_t(\x_t)\right) \dxv_0\\
&=\E_{q_{t0}(\x_0|\x_t)}\left[\|\nabla\log q_{0t}(\x_t|\x_0)-\nabla\log q_t(\x_t)\|_2^2\left(\nabla\log q_{0t}(\x_t|\x_0)-\nabla\log q_t(\x_t)\right)\right]
\end{align}
\end{proof}

\begin{lemma}
\label{lemma:third-with-vector}
Assume $(\x_t,\x_0)\sim q(\x_t,\x_0)$ and $\nabla_{\x_t}^3\log q_{0t}(\x_t|\x_0)=\vect{0}$, we have
\begin{equation}
\begin{aligned}
    &\nabla_{\x_t}\left(\vv^\top\nabla_{\x_t}^2\log q_t(\x_t)\vv\right)\\
    &=\E_{q_{t0}(\x_0|\x_t)}\Big[\left(\big(\nabla_{\x_t}\log q_{0t}(\x_t|\x_0)-\nabla_{\x_t}\log q_t(\x_t)\big)^\top\vv\right)^2\big(\nabla_{\x_t}\log q_{0t}(\x_t|\x_0)-\nabla_{\x_t}\log q_t(\x_t)\big)\Big]
\end{aligned}
\end{equation}
\end{lemma}
\begin{proof}(Proof of Lemma~\ref{lemma:third-with-vector})

We use $\nabla(\cdot)$ to denote the derivative of $\x_t$, namely $\nabla_{\x_t}(\cdot)$. Firstly, according to the second-order score given by Lemma~\ref{lemma:second}, we have
\begin{align}
\nabla\left(\vv^\top\nabla^2\log q_t(\x_t)\vv\right)&=\underbrace{\nabla\int q_{t0}(\x_0|\x_t)\left(\vv^\top\nabla^2\log q_{0t}(\x_t|\x_0)\vv\right)\dxv_0}_{(1)}\\
&\quad +\underbrace{\nabla\int q_{t0}(\x_0|\x_t)\left(\left(\nabla\log q_{0t}(\x_t|\x_0)-\nabla\log q_t(\x_t)\right)^\top\vv\right)^2\dxv_0}_{(2)}
\end{align}
Similar to the proof of Lemma~\ref{lemma:third}, we have
\begin{equation}
    (1)=\left(\vv^\top\nabla^2\log q_{0t}(\x_t|\x_0)\vv\right)\nabla\vect{1}=\vect{0},
\end{equation}
and
\begin{align}
    (2)&=\underbrace{\int \nabla q_{t0}(\x_0|\x_t)\left(\left(\nabla\log q_{0t}(\x_t|\x_0)-\nabla\log q_t(\x_t)\right)^\top\vv\right)^2\dxv_0}_{(3)}\\
    &\quad+\underbrace{\int q_{t0}(\x_0|\x_t)\nabla\left(\left(\nabla\log q_{0t}(\x_t|\x_0)-\nabla\log q_t(\x_t)\right)^\top\vv\right)^2\dxv_0}_{(4)}
\end{align}
Also similar to the proof of Lemma~\ref{lemma:third}, we have
\begin{align}
    (4)&=2\int q_{t0}(\x_0|\x_t)\left(\left((\nabla^2\log q_{0t}(\x_t|\x_0)-\nabla^2\log q_t(\x_t)\right)^\top\vv\right)\left(\left(\nabla\log q_{0t}(\x_t|\x_0)-\nabla\log q_t(\x_t)\right)^\top\vv\right)\dxv_0\\
    &=2\left(\left((\nabla^2\log q_{0t}(\x_t|\x_0)-\nabla^2\log q_t(\x_t)\right)^\top\vv\right)\int q_{t0}(\x_0|\x_t)\left(\left(\nabla\log q_{0t}(\x_t|\x_0)-\nabla\log q_t(\x_t)\right)^\top\vv\right)\dxv_0\\
    &=\vect{0},
\end{align}
so
\begin{align}
    &\nabla\left(\vv^\top\nabla^2\log q_t(\x_t)\vv\right)\\
    &=(3)\\
    &=\int q_{t0}(\x_0|\x_t)\left(\left(\nabla\log q_{0t}(\x_t|\x_0)-\nabla\log q_t(\x_t)\right)^\top\vv\right)^2\left(\nabla\log q_{0t}(\x_t|\x_0)-\nabla\log q_t(\x_t)\right)\dxv_0\\
    &=\E_{q_{t0}(\x_0|\x_t)}\Big[\left(\big(\nabla\log q_{0t}(\x_t|\x_0)-\nabla\log q_t(\x_t)\big)^\top\vv\right)^2\big(\nabla\log q_{0t}(\x_t|\x_0)-\nabla\log q_t(\x_t)\big)\Big]
\end{align}

\end{proof}

\subsection{Proof of Theorem~\ref{thrm:second}}
\label{appendix:thrm_second}
\begin{proof}
For simplicity, we denote $q_{0t}\coloneqq q_{0t}(\x_t|\x_0)$, $q_{t0}\coloneqq q_{t0}(\x_0|\x_t)$, $q_t\coloneqq q_t(\x_t)$, $q_0\coloneqq q_0(\x_0)$, $\hat\s_1\coloneqq\hat\s_1(\x_t,t)$, $\s_2(\theta)\coloneqq\s_2(\x_t,t;\theta)$.

As $\nabla\log q_{0t}=-\frac{\epsilonv}{\sigma_t}$ and $\nabla^2\log q_{0t}=-\frac{1}{\sigma_t^2}\Iv$, by rewriting the objective in Eqn.~\eqref{eqn:dsm-2-obj}, the optimization is equivalent to
\begin{align}
    \theta^*=\argmin_\theta\E_{q_t}\E_{q_{t0}}\left[\left\|\s_2(\theta)-\nabla^2\log q_{0t}-(\nabla\log q_{0t}-\hat\s_1)(\nabla\log q_{0t}-\hat\s_1)^\top\right\|_F^2\right]
\end{align}
For fixed $t$ and $\x_t$, minimizing the inner expectation is a minimum mean square error problem for $\s_2(\theta)$, so the optimal $\theta^*$ satisfies
\begin{align}
    \s_2(\theta^*)&=\E_{q_{t0}}\left[\nabla^2\log q_{0t}+(\nabla\log q_{0t}-\hat\s_1)(\nabla\log q_{0t}-\hat\s_1)^\top\right]
\end{align}
By Lemma~\ref{lemma:first} and Lemma~\ref{lemma:second}, we have
\begin{align}
    &\s_2(\theta^*)-\nabla^2\log q_t\\
    &=\E_{q_{t0}}\left[\hat\s_1\hat\s_1^\top-\hat\s_1\nabla\log q_{0t}^\top-\nabla\log q_{0t}\hat\s_1^\top-\nabla\log q_t\nabla\log q_t^\top+\nabla\log q_{0t}\nabla\log q_t^\top+\nabla\log q_t\nabla\log q_{0t}^\top\right]\\
    &=(\hat\s_1-\nabla\log q_t)(\hat\s_1-\nabla\log q_t)^\top
\end{align}
Therefore, we have
\begin{align}
    \|\s_2(\theta)-\nabla^2\log q_t\|_F&\leq \|\s_2(\theta)-\s_2(\theta^*)\|_F+\|\s_2(\theta^*)-\nabla^2\log q_t\|_F\\
    &=\|\s_2(\theta)-\s_2(\theta^*)\|_F+\|\hat\s_1-\nabla\log q_t\|_2^2
\end{align}
Moreover, by leveraging the property of minimum mean square error, we should notice that the training objective can be rewritten to
\begin{align}
    \theta^*=\argmin_{\theta}\E_{q_t(\x_t)}\|\s_2(\x_t,t;\theta)-\s_2(\x_t,t;\theta^*)\|_F^2,
\end{align}
which shows that $\|\s_2(\x_t,t;\theta)-\s_2(\x_t,t;\theta^*)\|_F$ can be viewed as the training error.
\end{proof}

\subsection{Proof of Corollary~\ref{coro:second_trace}}
\label{appendix:thrm_second_trace}
\begin{proof}
For simplicity, we denote $q_{0t}\coloneqq q_{0t}(\x_t|\x_0)$, $q_{t0}\coloneqq q_{t0}(\x_0|\x_t)$, $q_t\coloneqq q_t(\x_t)$, $q_0\coloneqq q_0(\x_0)$, $\hat\s_1\coloneqq\hat\s_1(\x_t,t)$, $\s_2^{\text{trace}}(\theta)\coloneqq\s_2^{\text{trace}}(\x_t,t;\theta)$.

As $\nabla\log q_{0t}=-\frac{\epsilonv}{\sigma_t}$ and $\nabla^2\log q_{0t}=-\frac{1}{\sigma_t^2}\Iv$, by rewriting the objective in Eqn.~\eqref{eqn:dsm-2-trace-obj}, the optimization is equivalent to
\begin{align}
    \theta^*=\argmin_\theta\E_{q_t}\E_{q_{t0}}\left[\left|\s_2^{\text{trace}}(\theta)-\tr(\nabla^2\log q_{0t})-\|\nabla\log q_{0t}-\hat\s_1\|_2^2\right|^2\right]
\end{align}
For fixed $t$ and $\x_t$, minimizing the inner expectation is a minimum mean square error problem for $\s_2^{\text{trace}}(\theta)$, so the optimal $\theta^*$ satisfies
\begin{align}
    \s_2^{\text{trace}}(\theta^*)&=\E_{q_{t0}}\left[\tr(\nabla^2\log q_{0t})+\|\nabla\log q_{0t}-\hat\s_1\|_2^2\right]
\end{align}
By Lemma~\ref{lemma:first} and Lemma~\ref{lemma:second}, similarly we have
\begin{align}
    \s_2^{\text{trace}}(\theta^*)-\tr(\nabla^2\log q_t)&=\|\hat\s_1-\nabla\log q_t\|_2^2
\end{align}
Therefore, we have
\begin{align}
    |\s_2^{\text{trace}}(\theta)-\tr(\nabla^2\log q_t)|&\leq |\s_2^{\text{trace}}(\theta)-\s_2^{\text{trace}}(\theta^*)|+|\s_2^{\text{trace}}(\theta^*)-\tr(\nabla^2\log q_t)|\\
    &=|\s_2^{\text{trace}}(\theta)-\s_2^{\text{trace}}(\theta^*)|+\|\hat\s_1-\nabla\log q_t\|_2^2
\end{align}
Moreover, by leveraging the property of minimum mean square error, we should notice that the training objective can be rewritten to
\begin{align}
    \theta^*=\argmin_{\theta}\E_{q_t(\x_t)}|\s_2^{\text{trace}}(\x_t,t;\theta)-\s_2^{\text{trace}}(\x_t,t;\theta^*)|^2,
\end{align}
which shows that $|\s_2^{\text{trace}}(\x_t,t;\theta)-\s_2^{\text{trace}}(\x_t,t;\theta^*)|$ can be viewed as the training error.
\end{proof}
\subsection{Proof of Theorem~\ref{thrm:third}}
\label{appendix:thrm_third}
\begin{proof}
For simplicity, we denote $q_{0t}\coloneqq q_{0t}(\x_t|\x_0)$, $q_{t0}\coloneqq q_{t0}(\x_0|\x_t)$, $q_t\coloneqq q_t(\x_t)$, $q_0\coloneqq q_0(\x_0)$, $\hat\s_1\coloneqq\hat\s_1(\x_t,t)$, $\hat\s_2\coloneqq\hat\s_2(\x_t,t)$, $\s_3(\theta)\coloneqq\s_3(\x_t,t;\theta)$.

As $\nabla\log q_{0t}=-\frac{\epsilonv}{\sigma_t}$ and $\nabla^2\log q_{0t}=-\frac{1}{\sigma_t^2}\Iv$, by rewriting the objective in Eqn.~\eqref{eqn:dsm-3-obj}, the optimization is equivalent to
\begin{equation}
\begin{split}
    \theta^*=\argmin_{\theta}&\E_{q_t(\x_t)}\E_{q_{t0}(\x_0|\x_t)}\left[\Bigg\|\s_3(\theta)-\left\|\nabla\log q_{0t}-\hat\s_1\right\|_2^2\left(\nabla\log q_{0t}-\hat\s_1\right)\right.\\
    &\left.+\Big(\left(\tr(\hat\s_2)-\tr(\nabla^2\log q_{0t})\right)\Iv+2\left(\hat\s_2-\nabla^2\log q_{0t}\right)\Big)\Big(\nabla\log q_{0t}-\hat\s_1\Big)\Bigg\|^2_2\right]
\end{split}
\end{equation}
For fixed $t$ and $\x_t$, minimizing the inner expectation is a minimum mean square error problem for $\s_3(\theta)$, so the optimal $\theta^*$ satisfies
\begin{align}
    \s_3(\theta^*)&=\E_{q_{t0}}\Big[\big\|\nabla\log q_{0t}-\hat\s_1\big\|_2^2\big(\nabla\log q_{0t}-\hat\s_1\big)\Big]\\
    &-\E_{q_{t0}}\bigg[\Big(\big(\tr(\hat\s_2)-\tr(\nabla^2\log q_{0t})\big)\Iv+2\big(\hat\s_2-\nabla^2\log q_{0t}\big)\Big)\Big(\nabla\log q_{0t}-\hat\s_1\Big)\bigg]
\end{align}
As $\nabla^2\log q_{0t}=-\frac{1}{\sigma_t^2}\Iv$ is constant w.r.t. $\x_0$, by Lemma~\ref{lemma:first}, we have
\begin{align}
    \s_3(\theta^*)&=\E_{q_{t0}}\Big[\big\|\nabla\log q_{0t}-\hat\s_1\big\|_2^2\big(\nabla\log q_{0t}-\hat\s_1\big)\Big]\\
    &-\Big(\big(\tr(\hat\s_2)-\tr(\nabla^2\log q_{0t})\big)\Iv+2\big(\hat\s_2-\nabla^2\log q_{0t}\big)\Big)\Big(\nabla\log q_t-\hat\s_1\Big)
\end{align}
By Lemma~\ref{lemma:third}, we have
\begin{align}
    &\s_3(\theta^*)-\nabla\tr(\nabla^2\log q_t)\\
    =\ &\E_{q_{t0}}\left[\Big(2\nabla\log q_{0t}\nabla\log q_{0t}^\top+2\nabla^2\log q_{0t}-2\hat\s_2\Big)\big(\nabla\log q_t-\hat\s_1\big)\right]\\
    &+\E_{q_{t0}}\left[\Big(\left\|\nabla\log q_{0t}\right\|_2^2+\tr(\nabla^2\log q_{0t})-\tr(\hat\s_2)\Big)\big(\nabla\log q_t-\hat\s_1\big)\right]\\
    &+(\|\hat\s_1\|_2^2-\|\nabla\log q_t\|_2^2)\E_{q_{t0}}[\nabla\log q_{0t}]\\
    &+2\E_{q_{t0}}\Big[(\nabla\log q_{0t}^\top\hat\s_1)\hat\s_1-(\nabla\log q_{0t}^\top\nabla\log q_t)\nabla\log q_t\Big]\\
    &+\|\nabla\log q_t\|_2^2\nabla\log q_t-\|\hat\s_1\|_2^2\hat\s_1
\end{align}
By Lemma~\ref{lemma:first} and Corollary~\ref{coro:second_conditional_score}, we have
\begin{align}
    \nonumber
    &\s_3(\theta^*)-\nabla\tr(\nabla^2\log q_t)\\
    \nonumber
    =\ &2\Big(\nabla\log q_{t}\nabla\log q_{t}^\top+\nabla^2\log q_{t}-\hat\s_2\Big)\big(\nabla\log q_t-\hat\s_1\big)\\
    \nonumber
    &+\Big(\left\|\nabla\log q_{t}\right\|_2^2+\tr(\nabla^2\log q_{t})-\tr(\hat\s_2)\Big)\big(\nabla\log q_t-\hat\s_1\big)\\
    \nonumber
    &+(\|\hat\s_1\|_2^2-\|\nabla\log q_t\|_2^2)\nabla\log q_t\\
    \nonumber
    &+2\left((\nabla\log q_{t}^\top\hat\s_1)\hat\s_1-\|\nabla\log q_t\|_2^2\nabla\log q_t\right)\\
    \nonumber
    &+\|\nabla\log q_t\|_2^2\nabla\log q_t-\|\hat\s_1\|_2^2\hat\s_1\\
    \nonumber
    =\ &\Big(2\big(\nabla^2\log q_t-\hat\s_2\big)+\big(\tr(\nabla^2\log q_t)-\tr(\hat\s_2)\big)\Big)(\nabla\log q_t-\hat\s_1)\\
    &+\|\nabla\log q_t-\hat\s_1\|_2^2(\nabla\log q_t-\hat\s_1)\label{eqn:appendix-third-original-optimal}
\end{align}
Therefore, we have
\begin{align}
    &\|\s_3(\theta)-\nabla\tr(\nabla^2\log q_t)\|_2\leq \|\s_3(\theta)-\s_3(\theta^*)\|_2+\|\s_3(\theta^*)-\nabla\tr(\nabla^2\log q_t)\|_2\\
    \leq\ &\|\s_3(\theta)-\s_3(\theta^*)\|_2+\|\nabla\log q_t-\hat\s_1\|_2^3\\
    &+\Big(2\big\|\nabla^2\log q_t-\hat\s_2\big\|_F+\big|\tr(\nabla^2\log q_t)-\tr(\hat\s_2)\big|\Big)\|\nabla\log q_t-\hat\s_1\|_2
\end{align}
Moreover, by leveraging the property of minimum mean square error, we should notice that the training objective can be rewritten to
\begin{align}
    \theta^*=\argmin_{\theta}\E_{q_t(\x_t)}\|\s_3(\x_t,t;\theta)-\s_3(\x_t,t;\theta^*)\|_2^2,
\end{align}
which shows that $\|\s_3(\x_t,t;\theta)-\s_3(\x_t,t;\theta^*)\|_2$ can be viewed as the training error.
\end{proof}

\subsection{Difference between error-bounded high-order denoising score matching and previous high-order score matching in \citet{meng2021estimating}}
\label{appendix:differ_high_order}
In this section, we analyze the DSM objective in \citet{meng2021estimating}, and show that the objective in \citet{meng2021estimating} has the unbounded-error property, which means even if the training error is zero and the first-order score matching error is any small, the second-order score matching error may be arbitrarily large.

~\cite{meng2021estimating} proposed an objective for estimating second-order score
\begin{align}
    \theta^*=\argmin_\theta\E_{q_t}\E_{q_{t0}}\left[\left\|\s_2(\theta)-\nabla^2\log q_{0t}-\nabla\log q_{0t}\nabla\log q_{0t}^\top +\hat\s_1\hat\s_1^\top\right\|_F^2\right]
\end{align}
For fixed $t$ and $\x_t$, minimizing the inner expectation is a minimum mean square error problem for $\s_2(\theta)$, so the optimal $\theta^*$ satisfies
\begin{align}
    \s_2(\theta^*)&=\E_{q_{t0}}\left[\nabla^2\log q_{0t}+\nabla\log q_{0t}\nabla\log q_{0t}^\top -\hat\s_1\hat\s_1^\top\right]
\end{align}
By Lemma~\ref{lemma:first} and Lemma~\ref{lemma:second}, we have
\begin{align}
    \s_2(\theta^*)-\nabla^2\log q_t&=\E_{q_{t0}}\left[-\hat\s_1\hat\s_1^\top-\nabla\log q_t\nabla\log q_t^\top+\nabla\log q_{0t}\nabla\log q_t^\top+\nabla\log q_t\nabla\log q_{0t}^\top\right]\\
    &=\nabla\log q_t\nabla\log q_t^\top-\hat\s_1\hat\s_1^\top
\end{align}
However, below we show that even if $\theta$ achieves its optimal solution $\theta^*$ and the first-order score matching error $\|\hat\s_1-\nabla\log q_t\|$ is small, the second-order score matching error $\s_2(\theta^*)-\nabla^2\log q_t$ may still be rather large.

We construct an example to demonstrate this point. Suppose the first-order score matching error $\hat\s_1-\nabla\log q_t=\delta_1\cdot\vect{1}$, where $\delta_1>0$ is small, then the first-order score matching error $\|\hat\s_1-\nabla\log q_t\|_2=\delta_1\sqrt{d}$, where $d$ is the dimension of $\x_t$. We have
\begin{align}
    \s_2(\theta^*)-\nabla^2\log q_t&=\nabla\log q_t\nabla\log q_t^\top - (\nabla\log q_t+\delta_1\cdot\vect{1})(\nabla\log q_t+\delta_1\cdot\vect{1})^\top\\
    &=-\delta_1\nabla\log q_t\vect{1}^\top-\delta_1\vect{1}\nabla\log q_t^\top - \delta_1^2\vect{1}\vect{1}^\top
\end{align}
We consider the unbounded score case~\cite{song2020score}, if the first-order score function $\nabla\log q_t$ is unbounded, i.e. for any $\delta_1>0$ and any $C>0$, there exists $\x_t\in\R^d$ such that $\|\nabla\log q_t\|_2>\frac{C+\delta_1^2d}{2\delta_1}$, then we have
\begin{align}
    \|\s_2(\theta^*)-\nabla\log q_t\|_F&\geq \delta_1\|\nabla\log q_t\vect{1}^\top+\vect{1}\nabla\log q_t^\top\|_F-\delta_1^2\|\vect{1}\vect{1}^\top\|_F\\
    &\geq \delta_1\|\diag(\nabla\log q_t\vect{1}^\top+\vect{1}\nabla\log q_t^\top)\|_2-\delta_1^2\|\vect{1}\vect{1}^\top\|_F\\
    &=\delta_1\|2\nabla\log q_t\|_2-\delta_1^2d\\
    &> C,
\end{align}
where the second inequality is because $\|A\|_F\geq \|\diag(A)\|_2$, where $\diag(A)$ means the diagonal vector of the matrix $A$. Therefore, for any small $\delta_1$, the second-order score estimation error may be arbitrarily large.

In practice, \citet{meng2021estimating} do not stop gradients for $\hat\s_1$, which makes the optimization of the second-order score model $\s_2(\theta)$ affecting the first-order score model $\s_1(\theta)$, and thus cannot theoretically guarantee the convergence. Empirically, we implement the second-order DSM objective in \citet{meng2021estimating}, and we find that if we stop gradients for $\hat\s_1$, the model quickly diverges and cannot work. We argue that this is because of the unbounded error of their method, as shown above.

Instead, our proposed high-order DSM method has the error-bounded property, which shows that $\|\s_2(\theta^*)-\nabla\log q_t\|_F=\|\hat\s_1-\nabla\log q_t\|_2^2$. So our method does not have the problem mentioned above.

\section{Estimated objectives of high-order DSM for high-dimensional data}
\label{appendix:ubiased_estimation}
In this section, we propose the estimated high-order DSM objectives for high-dimensional data.

The second-order DSM objective in Eqn.~\eqref{eqn:dsm-2-obj} requires computing the Frobenius norm of the full Jacobian of the score model, i.e. $\nabla_{\x}\s_\theta(\x,t)$, which typically has $\Oc(d^2)$ time complexity and is unacceptable for high dimensional real data. Moreover, the high-order DSM objectives in Eqn.~\eqref{eqn:dsm-2-trace-obj} and Eqn.~\eqref{eqn:dsm-3-obj} include computing the divergence of score network $\nabla_{\x}\cdot\s_\theta(\x,t)$ i.e. the trace of Jacobian $\tr(\nabla_{\x}\s_\theta(\x,t))$, and it also has $\Oc(d^2)$ time complexity. Similar to~\citet{grathwohl2018ffjord} and \citet{finlay2020train}, the cost can be reduced to $\Oc(d)$ using Hutchinson’s trace estimator and automatic diffentiation provided by general deep learning frameworks, which needs one-time backpropagation only.

For a $d$-by-$d$ matrix $A$, its trace can be unbiasedly estimated by~\cite{hutchinson1989stochastic}:
\begin{align}
    \tr(A)=\E_{p(\vv)}\left[\vv^\top A\vv\right]
\end{align}
where $p(\vv)$ is a $d$-dimensional distribution such that $\E\left[\vv\right]=\vect{0}$ and $\Cov\left[\vv\right]=\Iv$. Typical choices of $p(\vv)$ are a standard Gaussian or Rademacher distribution. Moreover, the Frobenius norm can also be unbiasedly estimated, since
\begin{align}
    \|A\|_F^2=\tr(A^\top A)=\E_{p(\vv)}\left[\vv^\top A^\top A\vv\right]=\E_{p(\vv)}\left[\|A\vv\|_2^2\right]
\end{align}
Let $A=\nabla_{\x}\s_\theta(\x,t)$, The Jacobian-vector-product $\nabla_{\x}\s_\theta(\x,t)\vv$ can be efficiently computed by using once forward-mode automatic differentiation in JAX, making the evaluating of trace and Frobenius norm approximately the same cost as evaluating $\s_\theta(\x,t)$. And we show the time costs for the high-order DSM training in Appendix.~\ref{appendix:cifar10}.

By leveraging the unbiased estimator, our final objectives for second-order and third-order DSM are:
\begin{align}
    \Jc_{\DSM,\text{estimation}}^{(2)}(\theta)&=\E_{t,\x_0,\epsilonv}\E_{p(\vv)}\left[\left\|\sigma_t^2\s_{jvp}+\vv-(\sigma_t\hat\s_1\cdot\vv+\epsilonv\cdot\vv)(\sigma_t\hat\s_1+\epsilonv)\right\|_2^2\right]\label{unbiased:dsm2},\\
    \Jc_{\DSM,\text{estimation}}^{(2,\text{tr})}(\theta)&=\E_{t,\x_0,\epsilonv}\E_{p(\vv)}\bigg[\Big|\sigma_t^2\vv^\top\s_{jvp}+\|\vv\|_2^2-|\sigma_t\hat\s_1\cdot\vv+\epsilonv\cdot\vv|^2\Big|^2\bigg]\label{unbiased:dsm2tr},\\
    \nonumber
    \Jc_{\DSM,\text{estimation}}^{(3)}(\theta)&=\E_{t,\x_0,\epsilonv}\E_{p(\vv)}\bigg[\Big\|\sigma_t^3\vv^\top \nabla_{\x}\s_{jvp}+|\sigma_t \hat\s_1\cdot\vv+\epsilonv\cdot\vv|^2(\sigma_t \hat\s_1+\epsilonv)-(\sigma_t^2\vv^\top\hat\s_{jvp}+\|\vv\|_2^2)(\sigma_t \hat\s_1+\epsilonv)\\
    &\quad\quad\quad\quad\quad\quad\quad-2(\sigma_t \hat\s_1\cdot\vv+\epsilonv\cdot\vv)(\sigma_t^2 \hat\s_{jvp}+\vv)\Big\|_2^2\bigg]\label{unbiased:dsm3},
\end{align}
where $\s_{jvp}=\nabla_{\x}\s_\theta(\x,t)\vv$ is the Jacobian-vector-product as mentioned above, $\hat\s_1=\texttt{stop\_gradient}(\s_\theta)$, $\hat\s_{jvp}=\texttt{stop\_gradient}(\s_{jvp})$, and $\av\cdot \bv$ denotes the vector-inner-product for vectors $\av$ and $\bv$. We compute $\vv^\top\nabla_{\x}\s_{jvp}$ by using the auto-gradient functions in JAX (which can compute the "vector-Jacobian-product" by one-time backpropagation").

\subsection{Relationship between the estimated objectives and the original objectives for high-order DSM}
In this section, we show that the proposed estimated objectives are actually equivalent to or can upper bound the original objectives in Sec.~\ref{sec:variance_reduction} for high-order DSM. Specifically, we have
\begin{align}
    \Jc_{\DSM}^{(2)}(\theta)&=\Jc_{\DSM,\text{estimation}}^{(2)}(\theta),\\
    \Jc_{\DSM}^{(2,\text{tr})}(\theta)&\leq\Jc_{\DSM,\text{estimation}}^{(2,\text{tr})}(\theta),\\
    \Jc_{\DSM}^{(3)}(\theta)&\leq\Jc_{\DSM,\text{estimation}}^{(3)}(\theta).
\end{align}

Firstly, the estimated second-order DSM objective in Eqn.~\eqref{unbiased:dsm2} is equivalent to the original objective in Sec.~\ref{sec:variance_reduction}, because
\begin{align}
    \Jc_{\DSM}^{(2)}(\theta)&= \E_{t,\x_0,\epsilonv}\left[\left\|\sigma_t^2\nabla_{\x}\s_\theta(\x_t,t)+\Iv-(\sigma_t \hat\s_1+\epsilonv)(\sigma_t \hat\s_1+\epsilonv)^\top\right\|_F^2\right]\\
    &=\E_{t,\x_0,\epsilonv}\E_{p(\vv)}\left[\left\|\left(\sigma_t^2\nabla_{\x}\s_\theta(\x_t,t)+\Iv-(\sigma_t \hat\s_1+\epsilonv)(\sigma_t \hat\s_1+\epsilonv)^\top
    \right)\vv\right\|_2^2\right]\\
    &=\E_{t,\x_0,\epsilonv}\E_{p(\vv)}\left[\left\|\sigma_t^2\nabla_{\x}\s_\theta(\x_t,t)\vv+\vv-(\sigma_t\hat\s_1\cdot\vv+\epsilonv\cdot\vv)(\sigma_t\hat\s_1+\epsilonv)\right\|_2^2\right]\\
    &=\Jc_{\DSM,\text{estimation}}^{(2)}(\theta).
\end{align}
And the estimated trace of second-order DSM objective in Eqn.~\eqref{unbiased:dsm2tr} can upper bound the original objective in Sec.~\ref{sec:variance_reduction}, because
\begin{align}
    \Jc_{\DSM}^{(2,\text{tr})}(\theta)&=\E_{t,\x_0,\epsilonv}\left[\left|\sigma_t^2\tr(\nabla_{\x}\s_\theta(\x_t,t))+d-\|\sigma_t \hat\s_1+\epsilonv\|_2^2\right|^2\right]\\
    &=\E_{t,\x_0,\epsilonv}\left[\left|\tr\left(\sigma_t^2\nabla_{\x}\s_\theta(\x_t,t)+\Iv-(\sigma_t \hat\s_1+\epsilonv)(\sigma_t \hat\s_1+\epsilonv)^\top\right)\right|^2\right]\\
    &=\E_{t,\x_0,\epsilonv}\left[\left|\E_{p(\vv)}\left[\vv^\top\left(\sigma_t^2\nabla_{\x}\s_\theta(\x_t,t)+\Iv-(\sigma_t \hat\s_1+\epsilonv)(\sigma_t \hat\s_1+\epsilonv)^\top\right)\vv\right] \right|^2\right]\\
    &\leq \E_{t,\x_0,\epsilonv}\E_{p(\vv)}\left[\left|\vv^\top\left(\sigma_t^2\nabla_{\x}\s_\theta(\x_t,t)+\Iv-(\sigma_t \hat\s_1+\epsilonv)(\sigma_t \hat\s_1+\epsilonv)^\top\right)\vv\right|^2\right]\\
    &=\E_{t,\x_0,\epsilonv}\E_{p(\vv)}\left[\left|\sigma_t^2\vv^\top\nabla_{\x}\s_\theta(\x_t,t)\vv+\|\vv\|_2^2-|\sigma_t\hat\s_1\cdot\vv+\epsilonv\cdot\vv|^2\right|^2\right]\\
    &=\Jc_{\DSM,\text{estimation}}^{(2,\text{tr})}(\theta).
\end{align}
And the estimated third-order DSM objective in Eqn.~\eqref{unbiased:dsm3} can also upper bound the original objective in Sec.~\ref{sec:variance_reduction}. To prove that, we firstly propose the following theorem, which presents an equivalent form of the third-order DSM objective in Eqn.~\eqref{eqn:dsm-3-obj} for score models and the corresponding derivatives.
\begin{theorem}
\label{thrm:third-equiv}
(Error-Bounded Third-Order DSM by score models with trace estimators) Suppose that $\hat\s_1(\x_t,t)$ is an estimation for $\nabla_{\x}\log q_t(\x_t)$, and its derivative $\hat\s_2(\x_t,t)\coloneqq\nabla_{\x}\hat\s_1(\x_t,t)$ is an estimation for $\nabla_{\x}^2\log q_t(\x_t)$. Denote the Jacobian-vector-product of $\hat\s_1$ as $\hat\s_{jvp}(\x_t,t,\vv)\coloneqq \nabla_{\x}\hat\s_1(\x_t,t)\vv$. Assume we have a neural network $\s_\theta(\cdot,t):\R^d\rightarrow\R^d$ parameterized by $\theta$, then we can learn a third-order score model $\nabla_{\x}\tr(\nabla_{\x}\s_\theta(\x_t,t)): \R^{d}\rightarrow\R^{d}$ which minimizes
\begin{equation*}
    \E_{q_t(\x_t)}\left[\left\|\nabla_{\x}\tr(\nabla_{\x}\s_\theta(\x_t,t))-\nabla_{\x}\tr(\nabla_{\x}^2\log q_t(\x_t))\right\|_2^2\right],
\end{equation*}
by optimizing
\begin{equation}
\label{eqn:dsm-3-obj-by-estimation}
\begin{split}
    &\theta^*=\argmin_{\theta}\E_{\x_0,\epsilonv}\left[\frac{1}{\sigma_t^6}\big\|\sigma_t^3\nabla_{\x}\tr(\nabla_{\x}\s_\theta(\x_t,t))+\E_{p(\vv)}\left[\ellv_3(\x_t,t,\epsilonv,\vv)\right]\big\|_2^2\right]
\end{split}
\end{equation}
where
\begin{equation}
\label{eqn:ell_3_equiv}
\begin{split}
    &\ellv_3(\x_t,t,\epsilonv,\vv)\coloneqq|\sigma_t \hat\s_1\cdot\vv+\epsilonv\cdot\vv|^2(\sigma_t \hat\s_1+\epsilonv)-(\sigma_t^2\vv^\top\hat\s_{jvp}+\|\vv\|_2^2)(\sigma_t \hat\s_1+\epsilonv)-2(\sigma_t \hat\s_1\cdot\vv+\epsilonv\cdot\vv)(\sigma_t^2 \hat\s_{jvp}+\vv)\\
    &\x_t=\alpha_t\x_0+\sigma_t\epsilonv,\quad \epsilonv\sim\N(\vect{0},\Iv),\\
    &\vv\sim p(\vv),\ \E_{p(\vv)}[\vv]=\vect{0},\ \Cov_{p(\vv)}[\vv]=\Iv.
\end{split}
\end{equation}
Denote the first-order score matching error as $\delta_1(\x_t,t)\coloneqq\|\hat\s_1(\x_t,t)-\nabla_{\x}\log q_t(\x_t)\|_2$ and the second-order score matching errors as $\delta_2(\x_t,t)\coloneqq\|\nabla_{\x}\hat\s_1(\x_t,t)-\nabla_{\x}^2\log q_t(\x_t)\|_F$ and $\delta_{2,\text{tr}}(\x_t,t)\coloneqq|\tr(\nabla_{\x}\hat\s_1(\x_t,t))-\tr(\nabla_{\x}^2\log q_t(\x_t))|$. Then $\forall \x_t, \theta$, the score matching error for $\nabla_{\x}\tr(\nabla_{\x}\s_\theta(\x_t,t))$ can be bounded by:
\begin{equation*}
\begin{split}
    &\left\|\nabla_{\x}\tr(\nabla_{\x}\s_\theta(\x_t,t))-\nabla_{\x}\tr(\nabla_{\x}^2\log q_t(\x_t))\right\|_2\\
    \leq\ &\left\|\nabla_{\x}\tr(\nabla_{\x}\s_\theta(\x_t,t))-\nabla_{\x}\tr(\nabla_{\x}\s_{\theta^*}(\x_t,t))\right\|_2+\big(\delta_1^2+\delta_{2,\text{tr}}+2\delta_2\big)\delta_1^2
\end{split}
\end{equation*}
\end{theorem}
\begin{proof}
For simplicity, we denote $q_{0t}\coloneqq q_{0t}(\x_t|\x_0)$, $q_{t0}\coloneqq q_{t0}(\x_0|\x_t)$, $q_t\coloneqq q_t(\x_t)$, $q_0\coloneqq q_0(\x_0)$, $\hat\s_1\coloneqq\hat\s_1(\x_t,t)$, $\hat\s_{jvp}\coloneqq\hat\s_{jvp}(\x_t,t)$, $\nabla\tr(\nabla\s(\theta))\coloneqq\nabla_{\x}\tr(\nabla_{\x}\s_\theta(\x_t,t))$.

As $\nabla\log q_{0t}=-\frac{\epsilonv}{\sigma_t}$ and $\nabla^2\log q_{0t}=-\frac{1}{\sigma_t^2}\Iv$, by rewriting the objective in Eqn.~\eqref{eqn:dsm-3-obj-by-estimation}, the optimization is equivalent to:
\begin{equation}
\begin{split}
    \theta^*&=\argmin_{\theta} \E_{q_t}\E_{q_{t0}}\left[\Bigg\|\nabla\tr(\nabla\s(\theta))-\E_{p(\vv)}\bigg[\left|(\nabla\log q_{0t}-\hat\s_1)^\top\vv\right|^2\left(\nabla\log q_{0t}-\hat\s_1\right)\right.\\
    &\left.-\left(\vv^\top(\nabla\hat\s_1-\nabla^2\log q_{0t})\vv\right)(\nabla\log q_{0t}-\hat\s_1)-2\left((\nabla\log q_{0t}-\hat\s_1)^\top\vv\right)(\nabla\hat\s_1-\nabla^2\log q_{0t})\vv\bigg]\Bigg\|^2_2\right]
\end{split}
\end{equation}
For fixed $t$ and $\x_t$, minimizing the inner expectation is a minimum mean square error problem for $\nabla\tr(\nabla\s(\theta))$, so the optimal $\theta^*$ satisfies
\begin{align}
    \nabla\tr(\nabla\s(\theta^*))&=\E_{q_{t0}}\E_{p(\vv)}\Big[\left|(\nabla\log q_{0t}-\hat\s_1)^\top\vv\right|^2\left(\nabla\log q_{0t}-\hat\s_1\right)\\
    &\quad-\left(\vv^\top(\hat\s_2-\nabla^2\log q_{0t})\vv\right)(\nabla\log q_{0t}-\hat\s_1)-2\left((\nabla\log q_{0t}-\hat\s_1)^\top\vv\right)(\hat\s_2-\nabla^2\log q_{0t})\vv\Big].
\end{align}
Denote
\begin{align}
    \s_3(\vv,\theta^*)&\coloneqq\E_{q_{t0}}\Big[\left|(\nabla\log q_{0t}-\hat\s_1)^\top\vv\right|^2\left(\nabla\log q_{0t}-\hat\s_1\right)\\
    &\quad-\left(\vv^\top(\hat\s_2-\nabla^2\log q_{0t})\vv\right)(\nabla\log q_{0t}-\hat\s_1)-2\left((\nabla\log q_{0t}-\hat\s_1)^\top\vv\right)(\hat\s_2-\nabla^2\log q_{0t})\vv\Big],
\end{align}
then we have $\E_{p(\vv)}[\s_3(\vv,\theta^*)]=\nabla\tr(\nabla\s(\theta^*))$. Similar to the proof in Appendix.~\ref{appendix:thrm_third}, combining with Lemma~\ref{lemma:third-with-vector}, we have
\begin{equation}
\begin{aligned}
    \s_3(\vv,\theta^*)-\nabla(\vv^\top\nabla^2\log q_t\vv)&=\s_3(\vv,\theta^*)-\E_{q_{t0}}\left[|(\nabla\log q_{0t}-\nabla\log q_t)^\top\vv|^2(\nabla\log q_{0t}-\nabla\log q_t)\right]\\
    &=2\left((\nabla\log q_t-\hat\s_1)^\top\vv\right)(\nabla\log q_t\nabla\log q_t^\top+\nabla^2\log q_t-\hat\s_2)\vv\\
    &\quad+\vv^\top(\nabla\log q_t\nabla\log q_t^\top+\nabla^2\log q_t-\hat\s_2)\vv(\nabla\log q_t-\hat\s_1)\\
    &\quad+\vv^\top(\hat\s_1\hat\s_1^\top-\nabla\log q_t\nabla\log q_t^\top)\vv\nabla\log q_t\\
    &\quad+2\left((\hat\s_1^\top\vv)(\nabla\log q_t^\top\vv)\hat\s_1-(\nabla\log q_t^\top\vv)^2\nabla\log q_t\right)\\
    &\quad+(\nabla\log q_t^\top\vv)^2\nabla\log q_t-(\hat\s_1^\top\vv)^2\hat\s_1\\
    &=\left(2(\nabla^2\log q_t-\hat\s_2)\vv\vv^\top+\vv^\top(\nabla^2\log q_t-\hat\s_2)\vv\right)(\nabla\log q_t-\hat\s_1)\\
    &\quad+\left|(\nabla\log q_t-\hat\s_1)^\top\vv\right|^2(\nabla\log q_t-\hat\s_1)
\end{aligned}
\end{equation}
As $\E_{p(\vv)}\left[\vv^\top A\vv\right]=\tr(A)$ and $\E_{p(\vv)}\left[\vv\vv^\top\right]=\E\left[\vv\right]\E\left[\vv\right]^\top+\Cov\left[\vv\right]=\Iv$, we have
\begin{align}
    \nabla\tr(\nabla\s(\theta^*))-\nabla\tr(\nabla^2\log q_t)&=\E_{p(\vv)}\left[\s_3(\vv,\theta^*)-\nabla(\vv^\top\nabla^2\log q_t\vv)\right]\\
    \nonumber
    &=\left(2\left(\nabla^2\log q_t-\hat\s_2\right)+\left(\tr(\nabla^2\log q_t)-\tr(\hat\s_2)\right)\right)(\nabla\log q_t-\hat\s_1)\\
    &\quad+\|\nabla\log q_t-\hat\s_1\|_2^2(\nabla\log q_t-\hat\s_1)\label{eqn:appendix-third-order-optimal}
\end{align}
Therefore, we have
\begin{align}
    \|\nabla\tr(\nabla\s(\theta))-\nabla\tr(\nabla^2\log q_t)\|_2&\leq \|\nabla\tr(\nabla\s(\theta))-\nabla\tr(\nabla\s(\theta^*))\|_2+\|\nabla\tr(\nabla\s(\theta^*))-\nabla\tr(\nabla^2\log q_t)\|_2\\
    &\leq\|\nabla\tr(\nabla\s(\theta))-\nabla\tr(\nabla\s(\theta^*))\|_2+\|\nabla\log q_t-\hat\s_1\|_2^3\\
    &\quad+\Big(2\big\|\nabla^2\log q_t-\hat\s_2\big\|_F+\big|\tr(\nabla^2\log q_t)-\tr(\hat\s_2)\big|\Big)\|\nabla\log q_t-\hat\s_1\|_2,
\end{align}
which completes the proof.
\end{proof}

Below we show that the objective in Eqn.~\eqref{eqn:dsm-3-obj} in Theorem~\ref{thrm:third} is equivalent to the objective in Eqn.~\eqref{eqn:dsm-3-obj-by-estimation} in Theorem~\ref{thrm:third-equiv}.
\begin{corollary}
\label{coro:third-equiv}
Suppose that $\hat\s_1(\x_t,t)$ is an estimation for $\nabla_{\x}\log q_t(\x_t)$, and its derivative $\hat\s_2(\x_t,t)\coloneqq\nabla_{\x}\hat\s_1(\x_t,t)$ is an estimation for $\nabla_{\x}^2\log q_t(\x_t)$. Denote the Jacobian-vector-product of $\hat\s_1$ as $\hat\s_{jvp}(\x_t,t,\vv)\coloneqq \nabla_{\x}\hat\s_1(\x_t,t)\vv$. Assume we have a neural network $\s_\theta(\cdot,t):\R^d\rightarrow\R^d$ parameterized by $\theta$, and we learn a third-order score model $\nabla_{\x}\tr(\nabla_{\x}\s_\theta(\x_t,t)): \R^{d}\rightarrow\R^{d}$ by the third-order DSM objectives. Then the objective in Eqn.~\eqref{eqn:dsm-3-obj-by-estimation} is equivalent to the objective in Eqn.~\eqref{eqn:dsm-3-obj} w.r.t. the optimization for $\theta$.
\end{corollary}
\begin{proof}
For simplicity, we denote $q_{0t}\coloneqq q_{0t}(\x_t|\x_0)$, $q_{t0}\coloneqq q_{t0}(\x_0|\x_t)$, $q_t\coloneqq q_t(\x_t)$, $q_0\coloneqq q_0(\x_0)$, $\hat\s_1\coloneqq\hat\s_1(\x_t,t)$, $\hat\s_{jvp}\coloneqq\hat\s_{jvp}(\x_t,t)$, $\nabla\tr(\nabla\s(\theta))\coloneqq\nabla_{\x}\tr(\nabla_{\x}\s_\theta(\x_t,t))$.

On the one hand, by Eqn.~\eqref{eqn:appendix-third-order-optimal}, the optimal solution of the objective in Eqn.~\eqref{eqn:dsm-3-obj-by-estimation} is:
\begin{align}
    \nonumber
    \nabla\tr(\nabla\s(\theta^*))&=\nabla\tr(\nabla^2\log q_t)+\left(2\left(\nabla^2\log q_t-\hat\s_2\right)+\left(\tr(\nabla^2\log q_t)-\tr(\hat\s_2)\right)\right)(\nabla\log q_t-\hat\s_1)\\
    &\quad+\|\nabla\log q_t-\hat\s_1\|_2^2(\nabla\log q_t-\hat\s_1),
\end{align}
so by the property of least mean square error, the objective in Eqn.~\eqref{eqn:dsm-3-obj-by-estimation} w.r.t. the optimization of $\theta$ is equivalent to
\begin{equation}
\begin{split}
    \min_{\theta}\E_{\x_0,\epsilonv}\left[\big\|\nabla_{\x}\tr(\nabla_{\x}\s_\theta(\x_t,t))-\nabla_{\x}\tr(\nabla_{\x}\s_{\theta^*}(\x_t,t))\big\|_2^2\right].
\end{split}
\end{equation}
On the other hand, by Eqn.~\eqref{eqn:appendix-third-original-optimal}, the optimal solution of the objective in Eqn.~\eqref{eqn:dsm-3-obj} is also:
\begin{align}
    \nonumber
    \nabla\tr(\nabla\s(\theta^*))&=\nabla\tr(\nabla^2\log q_t)+\left(2\left(\nabla^2\log q_t-\hat\s_2\right)+\left(\tr(\nabla^2\log q_t)-\tr(\hat\s_2)\right)\right)(\nabla\log q_t-\hat\s_1)\\
    &\quad+\|\nabla\log q_t-\hat\s_1\|_2^2(\nabla\log q_t-\hat\s_1),
\end{align}
so by the property of least mean square error, the objective in Eqn.~\eqref{eqn:dsm-3-obj} w.r.t. the optimization of $\theta$ is also equivalent to
\begin{equation}
\begin{split}
    \min_{\theta}\E_{\x_0,\epsilonv}\left[\big\|\nabla_{\x}\tr(\nabla_{\x}\s_\theta(\x_t,t))-\nabla_{\x}\tr(\nabla_{\x}\s_{\theta^*}(\x_t,t))\big\|_2^2\right].
\end{split}
\end{equation}
Therefore, the two objectives in Eqn.~\eqref{eqn:dsm-3-obj} and Eqn.~\eqref{eqn:dsm-3-obj-by-estimation} are equivalent w.r.t. $\theta$.
\end{proof}

Therefore, by \cref{coro:third-equiv}, we can derive an equivalent formulation of $\Jc_{\DSM}^{(3)}(\theta)$ by Theorem.~\ref{thrm:third-equiv}:
\begin{equation}
    \Jc_{\DSM}^{(3)}(\theta)= \E_{t,\x_0,\epsilonv}\Big[\big\|\sigma_t^3\nabla_{\x}\!\tr(\nabla_{\x}\s_\theta(\x_t,t))+\E_{p(\vv)}\left[\ellv_3(\x_t,t,\epsilonv,\vv)\right]\big\|_2^2\Big],
\end{equation}
where $\ellv_3(\x_t,t,\epsilonv,\vv)$ is defined in Eqn.~\eqref{eqn:ell_3_equiv}. Thus, we have
\begin{equation}
\begin{aligned}
    \Jc_{\DSM}^{(3)}(\theta)&= \E_{t,\x_0,\epsilonv}\Big[\big\|\sigma_t^3\E_{p(\vv)}\left[\nabla_{\x}\left(\vv^\top\nabla_{\x}\s_\theta(\x_t,t)\vv\right)\right]+\E_{p(\vv)}\left[\ellv_3(\x_t,t,\epsilonv,\vv)\right]\big\|_2^2\Big]\\
    &\leq \E_{t,\x_0,\epsilonv}\E_{p(\vv)}\Big[\big\|\sigma_t^3\nabla_{\x}\left(\vv^\top\nabla_{\x}\s_\theta(\x_t,t)\vv\right)+\ellv_3(\x_t,t,\epsilonv,\vv)\big\|_2^2\Big]\\
    &= \Jc_{\DSM,\text{estimation}}^{(3)}(\theta)
\end{aligned}
\end{equation}

\section{Training algorithm}
In this section, we propose our training algorithm for high-dimensional data, based on the high-order DSM objectives in Appendix.~\ref{appendix:ubiased_estimation}.

Let `JVP' and `VJP' denote forward-mode Jacobian-vector-product and reverse-mode vector-Jacobian-product with auxiliary data. More specifically, suppose $\fv$ is a function $\R^d\rightarrow\R^d$, and fn($\x$) = ($\fv(\x)$, aux), where aux is auxiliary data, then JVP(fn, $\x_0$, $\vv$) = ($\fv(\x_0)$, $\nabla\fv|_{\x=\x_0}\vv$, aux) and VJP(fn, $\x_0$, $\vv$) = ($\fv(\x_0)$, $\vv^\top\nabla\fv|_{\x=\x_0}$, aux).
\label{appendix:training_algorithm}
\begin{algorithm}[ht!]
\caption{Training of high-order denoising score matching}
\label{algorithm:1}
\textbf{Require:} score network $\s_\theta$, hyperparameters $\lambda_1,\lambda_2$, smallest time $\epsilon$, forward SDE, proposal distribution $p(t)$\\
\textbf{Input:} sample $\x_0$ from data distribution\\
\textbf{Output:} denoising score matching loss $\Jc_{\DSM}$

\quad Sample $\epsilonv$ from $\N(\vect{0},\Iv)$

\quad Sample $\vv$ from standard Gaussian or Rademacher distribution

\quad Sample $t$ from proposal distribution $p(t)$ ($\Uc(\epsilon,1)$ for VE SDEs)

\quad Get mean $\alpha_t$ and std $\sigma_t$ at time $t$ from the forward SDE.

\quad $\x_t\leftarrow \alpha_t\x_0+\sigma_t\epsilonv$
  
\quad \textbf{def} grad\_div\_fn($\x_t$, $t$, $\vv$):

\quad \quad \textbf{def} score\_jvp\_fn($\x_t$):

\quad \quad \quad $\s_\theta\leftarrow$ lambda $\x$: $\s_\theta(\x, t)$

\quad \quad \quad $\s_\theta(\x_t,t), \s_{jvp}\leftarrow$ JVP($\s_\theta$, $\x_t$, $\vv$)
    
\quad \quad \quad \textbf{return} $\s_{jvp}$, $\s_\theta(\x_t,t)$

\quad \quad $\s_{jvp}$, $\vv^\top\nabla\s_{jvp}$, $\s_\theta(\x_t,t)\leftarrow$ VJP(score\_jvp\_fn, $\x_t$, $\vv$)

\quad \quad \textbf{return} $\s_{jvp}$, $\vv^\top\nabla\s_{jvp}$, $\s_\theta(\x_t,t)$

\quad $\s_{jvp}$, $\vv^\top\nabla\s_{jvp}$, $\s_\theta(\x_t,t)\leftarrow$ grad\_div\_fn($\x_t$, $t$, $\vv$)

\quad $\hat\s_1\leftarrow$ stop\_gradient($\s_\theta(\x_t,t)$)

\quad $\hat\s_{jvp}\leftarrow$ stop\_gradient($\s_{jvp}$)

\quad Calculate $\Jc_{\DSM}^{(1)}(\theta),\Jc_{\DSM}^{(2)}(\theta),\Jc_{\DSM}^{(2,\text{tr})}(\theta),\Jc_{\DSM}^{(3)}(\theta)$ from Eqn.~\eqref{eqn:dsm-1-final}~\eqref{unbiased:dsm2}~\eqref{unbiased:dsm2tr}~\eqref{unbiased:dsm3}

\quad \textbf{return} $\Jc_{\DSM}^{(1)}(\theta)+\lambda_1\left(\Jc_{\DSM}^{(2)}(\theta)+\Jc_{\DSM}^{(2,\text{tr})}(\theta)\right)+\lambda_2\Jc_{\DSM}^{(3)}(\theta)$

\end{algorithm}
\section{Experiment details}
\label{appendix:experiments_details}

\subsection{Choosing of $\lambda_1,\lambda_2$}
\label{appendix:choose_lambda}
As we use Monto-Carlo method to unbiasedly estimate the expectations, we empirically find that we need to ensure the mean values of $\Jc_{\DSM}^{(1)}(\theta),\lambda_1\left(\Jc_{\DSM}^{(2)}(\theta)+\Jc_{\DSM}^{(2,\text{tr})}(\theta)\right)$ and $\lambda_2\Jc_{\DSM}^{(3)}(\theta)$ are in the same order of magnitude.
Therefore, for synthesis data experiments, we choose $\lambda_1=0.5$ and $\lambda_2=0.1$ for our final objectives. And for CIFAR-10 experiments, we simply choose $\lambda_1=\lambda_2=1$ with no further tuning.

Specifically, for synthesis data, we choose $\lambda_1=0.5$, $\lambda_2=0$ as the second-order score matching objective, and $\lambda_1=0.5$, $\lambda_2=0.1$ as the third-order score matching objective. For CIFAR-10, we simply choose $\lambda_1=1$, $\lambda_2=0$ as the second-order score matching objective, and $\lambda_1=\lambda_2=1$ as the third-order score matching objective.

In practice, the implementation of the first-order DSM objective in~\citep{song2020score, song2021maximum} divides the loss function by the data dimension (i.e. use $\frac{\|\cdot\|_2^2}{d}$ instead of $\|\cdot\|_2^2$). We also divide the second-order DSM and the third-order DSM objectives by the data dimension for the image data to balance the magnitudes of the first, second and third-order objectives. Please refer to the implementation in our released code for details.

\subsection{1-D mixture of Gaussians}
\label{appendix:mog_experiments}
We exactly compute the high-order score matching objectives in Sec.~\ref{sec:variance_reduction}, and choose the starting time $\epsilon=10^{-5}$.

The density function of the data distribution is
\begin{equation}
    q_0(\x_0)=0.4*\Nc(\x_0|-\frac{2}{9},\frac{1}{9^2})+0.4*\Nc(\x_0|-\frac{2}{3},\frac{1}{9^2})+0.2*\Nc(\x_0|\frac{4}{9},\frac{2}{9^2})
\end{equation}
We use the ``noise-prediction'' type model~\cite{kingma2021variational}, i.e. we use a neural network $\epsilonv_{\theta}(\x_t,t)$ to model $\sigma_t\s_\theta(\x_t,t)$. We use the time-embedding in~\citet{song2020score}. For the score model, we use a two-layer MLP to encode $t$, and a two-layer MLP to encode the input $\x_t$, then concatenate them together to another two-layer MLP network to output the predicted noise $\epsilonv_\theta(\x_t,t)$.

We use Adam~\cite{kingma2014adam} optimizer with the default settings in JAX. The batch size is $5000$. We train the model for $50$k iterations by one NVIDIA GeForece RTX 2080 Ti GPU card.

As our proposed high-order DSM algorithm has the property of bounded errors, we directly train our model from the default initialized neural network, without any pre-training for the lower-order models.

\subsection{2-D checkerboard data}
\label{appendix:checkerboard}
We exactly compute the high-order score matching objectives in Sec.~\ref{sec:variance_reduction}, and choose the starting time $\epsilon=10^{-5}$.

We use the same neural network and hyperparameters in Appendix.~\ref{appendix:mog_experiments}, and we train for $100$k iterations for all experiments by one NVIDIA GeForece RTX 2080 Ti GPU card.

Similar to Appendix.~\ref{appendix:mog_experiments}, we directly train our model from the default initialized neural network, without any pre-training for the lower-order models.

\subsection{CIFAR-10 experiments}
\label{appendix:cifar10}
Our code of CIFAR-10 experiments are based on the released code of \citet{song2020score} and \citet{song2021maximum}. We choose the start time $\epsilon=10^{-5}$ for both training and evaluation.

We train the score model by the proposed high-order DSM objectives both from $0$ iteration and from the pre-trained checkpoints in \citet{song2020score} to a fixed ending iteration, and we empirically find that the model performance by these two training procedure are nearly the same. Therefore, to save time, we simply use the pre-trained checkpoints to further train the second-order and third-order models by a few iteration steps and report the results of the fine-tuning models. Moreover, we find that further train the pre-trained checkpoints by first-order DSM cannot improve the likelihood of {\ode}, so we simply use the pre-trained checkpoints to evaluate the first-order models.

\paragraph{Model architectures}The model architectures are the same as~\cite{song2020score}. For VE, we use NCSN++ cont. which has 4 residual blocks per resolution. For VE (deep), we use NCSN++ cont. deep which has 8 residual blocks per resolution. Also, our network is the ``noise-prediction'' type, same as the implementation in \citet{song2020score}.

\paragraph{Training}We follow the same training procedure and default settings for score-based models as~\cite{song2020score}, and set the exponential moving average (EMA) rate to 0.999 as advised. Also as in~\cite{song2020score}, the input images are pre-processed to be normalized to $[0,1]$ for VESDEs. For all experiments, we set the ``\texttt{n\_jitted\_steps}=1'' in JAX code.

For the experiments of the VE model, we use 8 GPU cards of NVIDIA GeForece RTX 2080 Ti. We use the pre-trained checkpoints of $1200$k iterations of \citet{song2020score}, and further train $100$k iterations (for about half a day) and the model quickly converges. We use a batchsize of 128 of the second-order training, and a batchsize of 48 of the third-order training.

For the experiments of the VE(deep) model, we use 8 GPU cards of Tesla P100-SXM2-16GB. We use the pre-trained checkpoints of $600$k iterations of \citet{song2020score}, and further train $100$k iterations (for about half a day) and the model quickly converges. We use a batchsize of 128 of the second-order training, and a batchsize of 48 of the third-order training.

\paragraph{Likelihood and sample quality} We use the uniform dequantization for likelihood evaluation. For likelihood, we report the bpd on the test dataset with 5 times repeating (to reduce the variance of the trace estimator). For sampling, we find the ode sampler often produce low quality images and instead use the PC sampler~\cite{song2020score} discretized at 1000 time steps to generate 50k samples and report the FIDs on them. We use the released pre-trained checkpoints of the VESDE in \citet{song2020score} to evaluate the likelihood and sample quality of the first-order score matching models.

\paragraph{Computation time, iteration numbers and memory consumption} We list the computation time and memory consumption of the VE model (shallow model) on 8 NVIDIA GeForce RTX 2080 Ti GPU cards with batch size 48 in Table~\ref{tab:cifar10:cost}. The \texttt{n\_jitted\_steps} is set to 1. When training by second-order DSM, we only use \texttt{score\_jvp\_fn} in the training algorithm and remove \texttt{grad\_div\_fn}. The computation time is averaged over 10k iterations. We use  \texttt{jax.profiler} to trace the GPU memory usage during training, and report the peak total memory on 8 GPUs.

We find that the costs of the second-order DSM training are close to the first-order DSM training, and the third-order DSM training costs less than twice of the first-order training. Nevertheless, our method can scale up to high-dimensional data and improve the likelihood of {\ode}s.

\begin{table}
    \centering
    \caption{Benchmark of computation time, iteration numbers and memory consumption on CIFAR-10.}
    \vskip 0.15in
    \begin{small}
    \begin{tabular}{lcccc}
    \toprule
    Model & Time per iteration (s) & Memory in total (GiB) & Training iterations \\
    \midrule
        VE~\cite{song2020score} & 0.28 & 25.84 & 1200k\\
        VE (second)~(\bf{ours}) & 0.33 & 28.93 & 1200k(checkpoints)\ +\ 100k\\
        VE (third)~(\bf{ours}) & 0.44 & 49.12 & 1200k(checkpoints)\ +\ 100k\\
    \bottomrule
    \end{tabular}
    \end{small}
    \label{tab:cifar10:cost}
\end{table}

\subsection{ImageNet 32x32 experiments}
\label{appendix:imagenet32}
We adopt the same start time and model architecture for both VE and VE (deep) as CIFAR-10 experiments. Note that the released code of \citet{song2020score} and \citet{song2021maximum} provides no pretrained checkpoint of VE type for ImageNet 32x32 dataset, so we use their training of VP type as a reference, and train the first-order VE models from scratch. Specifically, we train the VE baseline for 1200k iterations, and the VE (deep) baseline for 950k iterations.

For the high-order experiments of both VE and VE (deep), we further train 100k iterations, using a batchsize of 128 for the second-order training and a batchsize of 48 for the third-order training. We use 8 GPU cards of NVIDIA GeForece RTX 2080 Ti for VE, and 8 GPU cards of NVIDIA A40 for VE (deep). We report the average bpd on the test dataset with 5 times repeating.

\section{Additional results for VE, VP and subVP types}
\label{appendix:additional_samples}

\begin{figure*}[t]
	\centering
	\begin{minipage}{.4\linewidth}
	\centering
	\includegraphics[width=.95\linewidth]{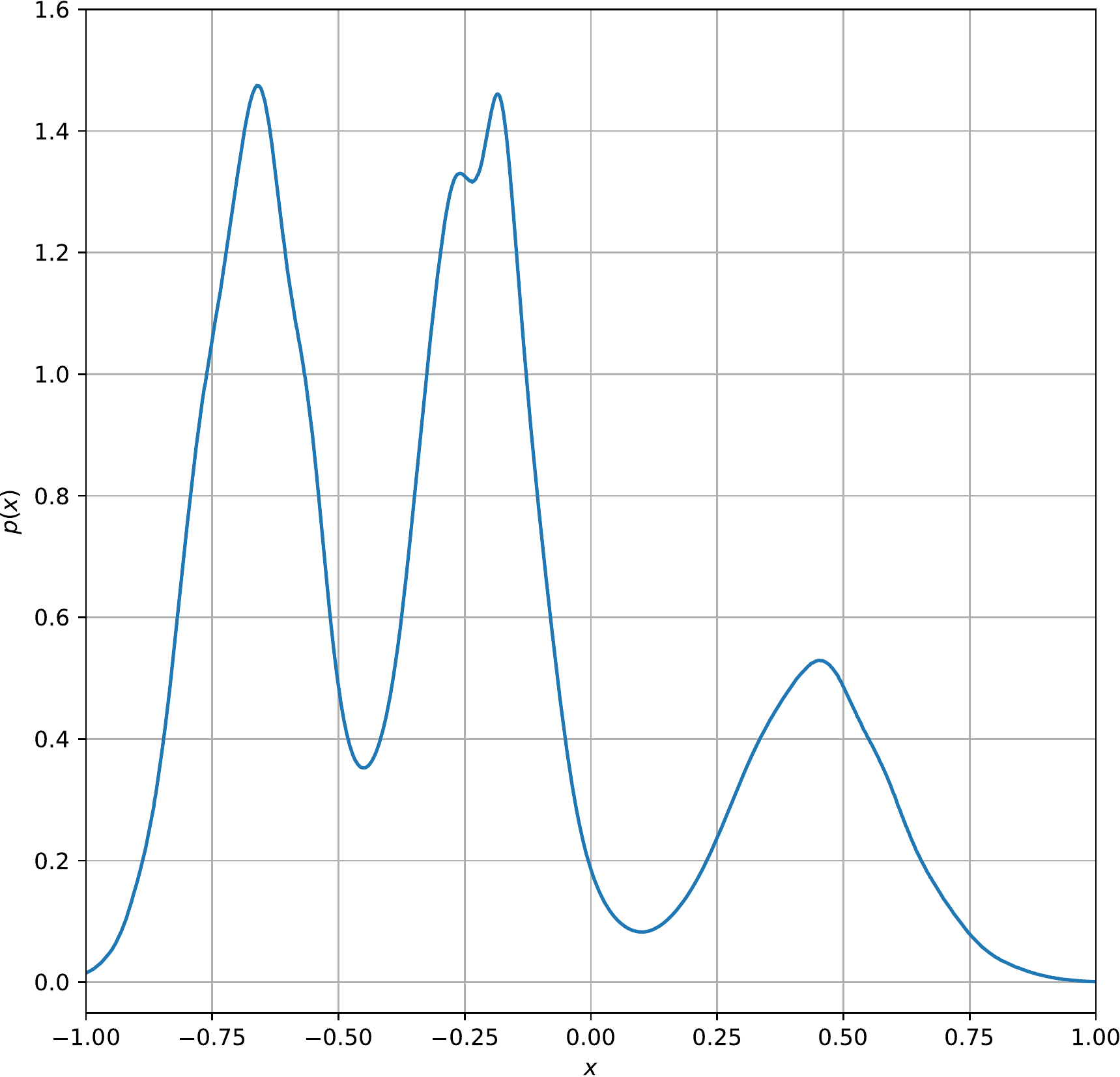}\\
\small{(e) First-order DSM, VP}
	\end{minipage}
	\begin{minipage}{.4\linewidth}
	\centering
	\includegraphics[width=.95\linewidth]{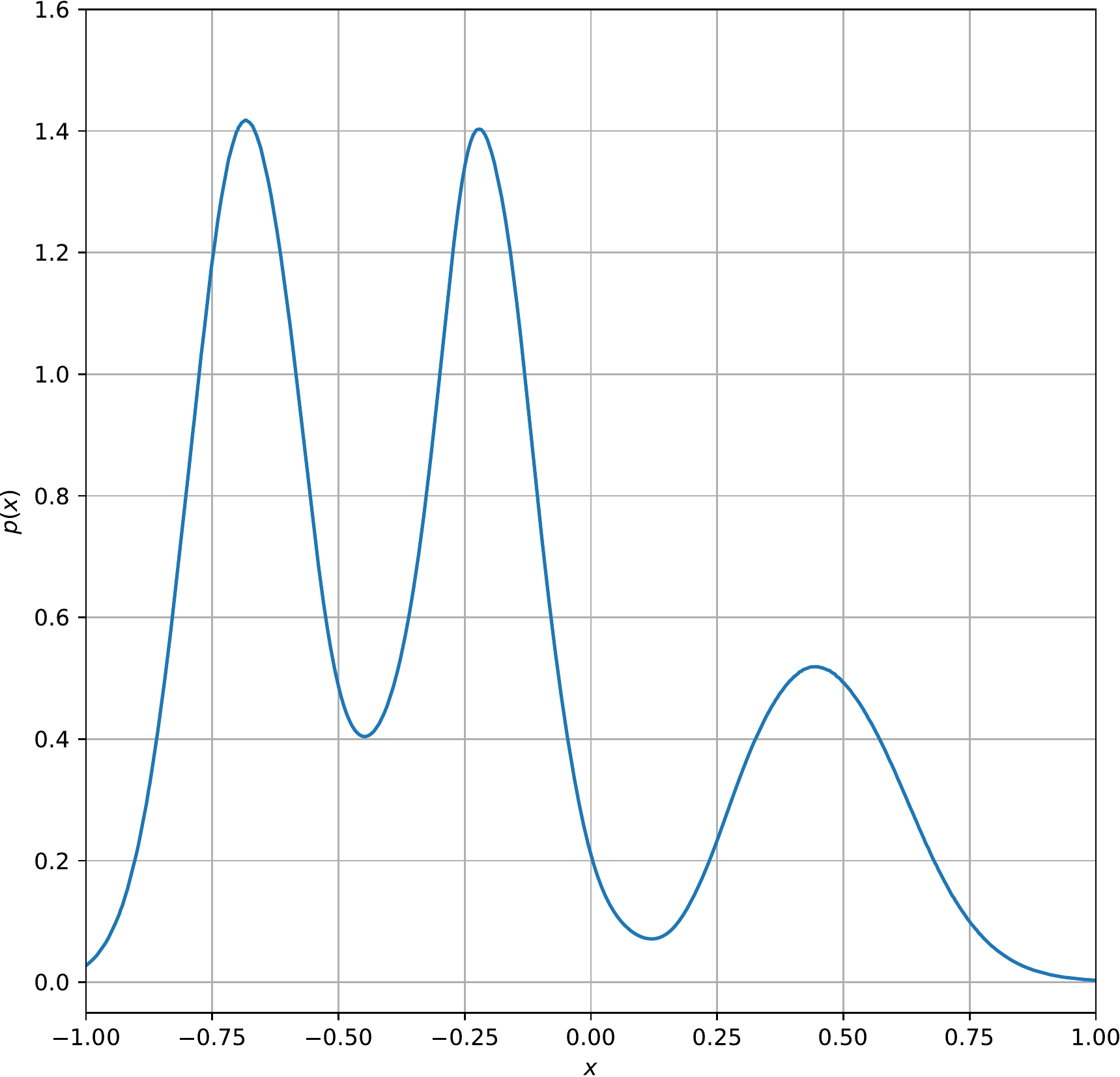}\\
\small{(f) Third-order DSM, VP}
\end{minipage}
	\caption{Model density of {\ode}s (VP type) on 1-D mixture-of-Gaussians data.\label{fig:mog_density_vp}}
\end{figure*}

\begin{figure*}[t]
	\centering
	\begin{minipage}{.48\linewidth}
		\centering		
 	\includegraphics[width=.95\linewidth]{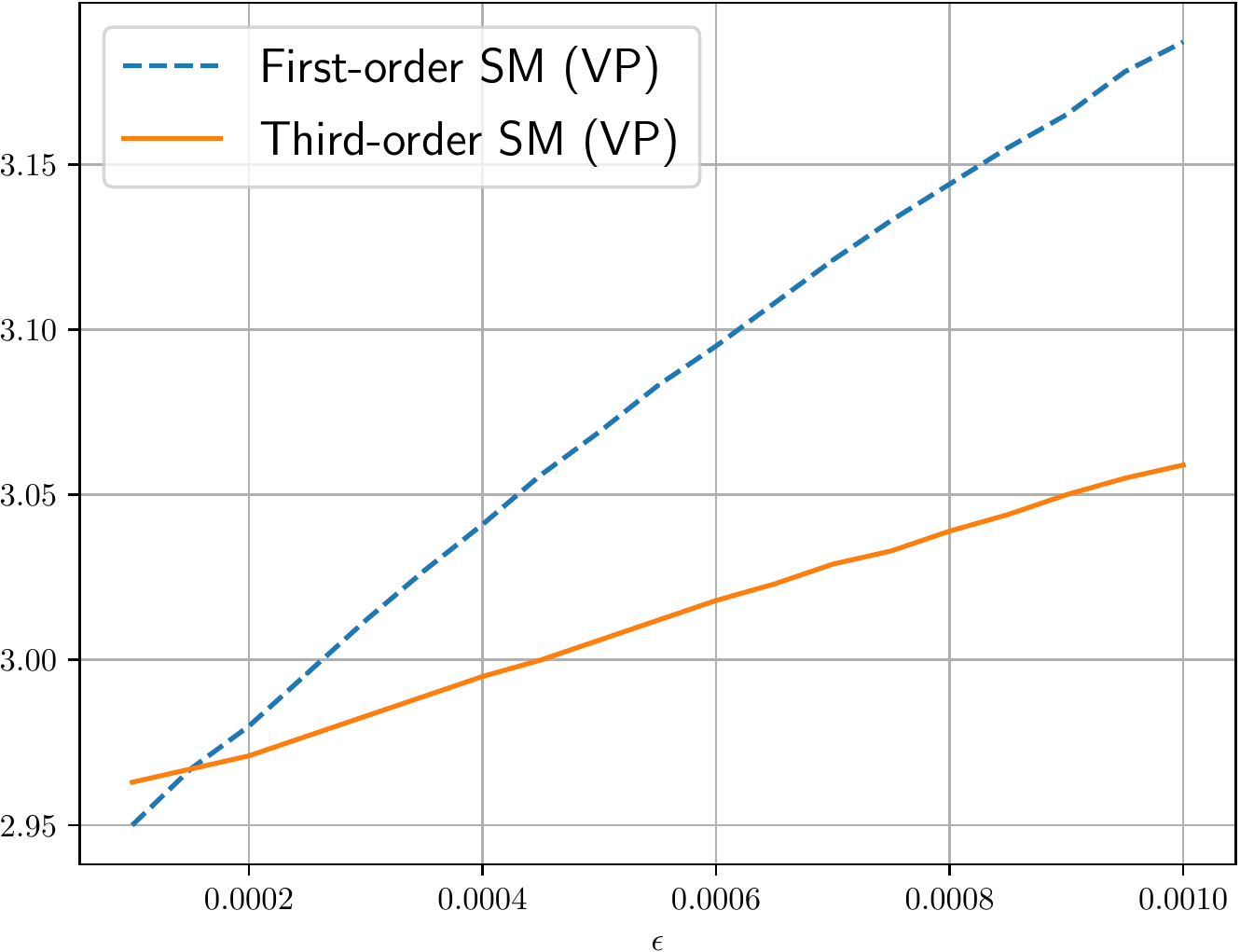}\\
\small{(a) CIFAR-10 bpd (VP)}
  \end{minipage}
	\begin{minipage}{.48\linewidth}
		\centering		
 	\includegraphics[width=.95\linewidth]{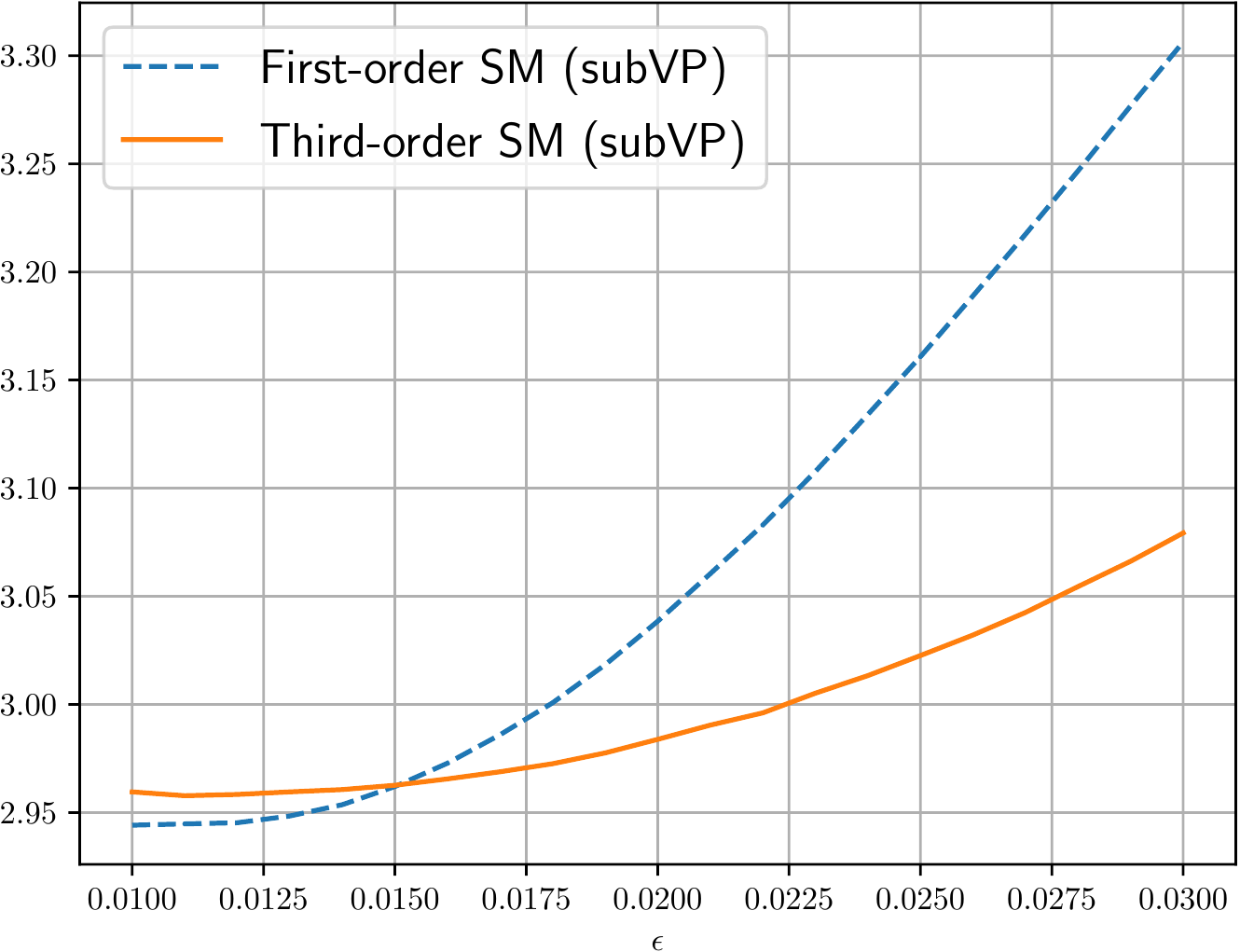}\\
\small{(b) CIFAR-10 bpd (subVP)}
  \end{minipage}
	\caption{Model NLL (negative-log-likelihood) in bits/dim (bpd) of {\ode}s for both VP type and subVP type trained by first-order DSM and our proposed third-order DSM on CIFAR-10 with the likelihood weighting functions and the importance sampling in~\citet{song2021maximum}, varying the start evaluation time $\epsilon$.\label{fig:cifar10_density_vp}}
\end{figure*}

\begin{figure*}[t]
	\centering
	\begin{minipage}{.48\linewidth}
		\centering
			\includegraphics[width=.91\linewidth]{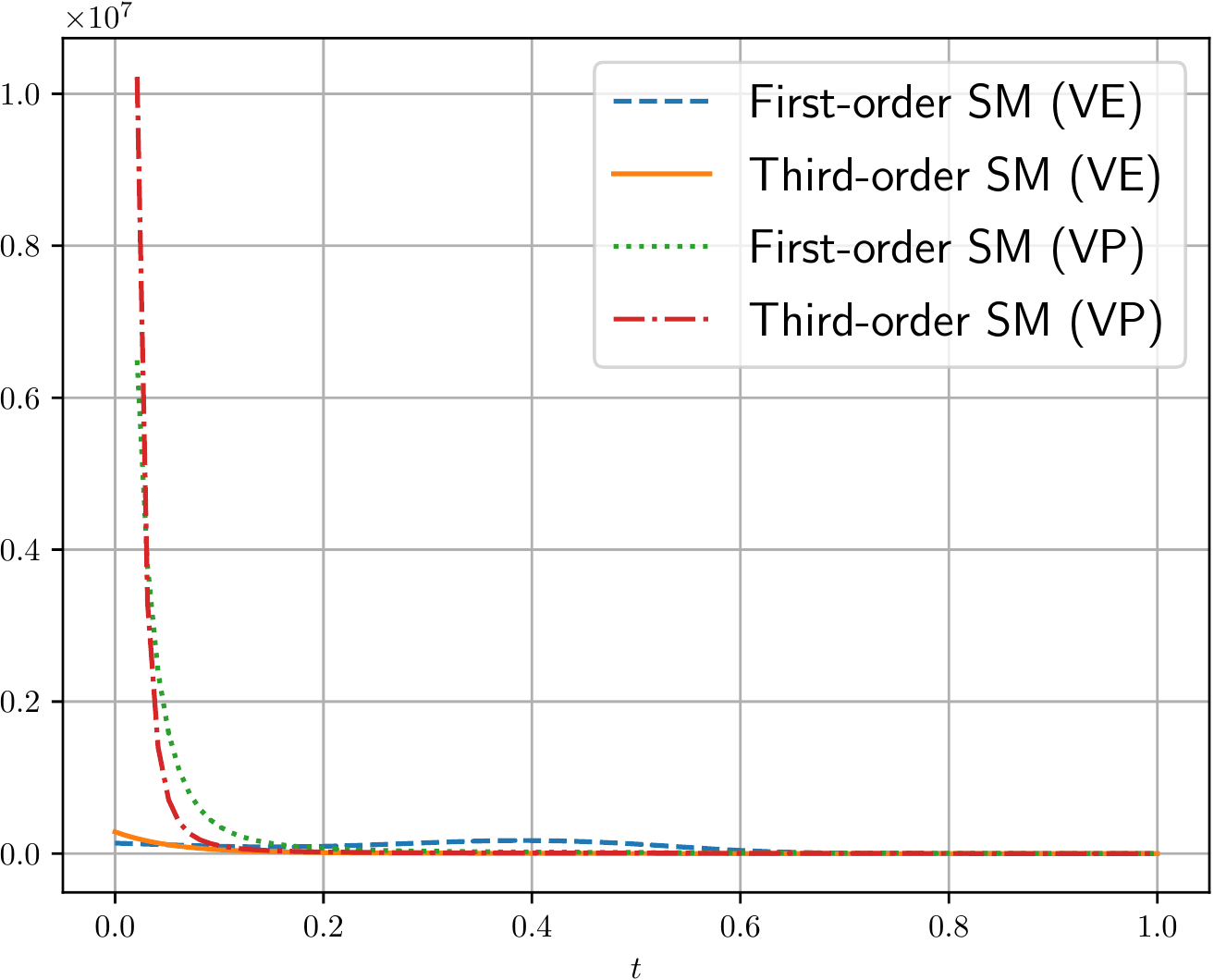}\\
\small{(a) $\ell_{\DIFF}(t)$.}
	\end{minipage}
	\begin{minipage}{.48\linewidth}
	\centering
	\includegraphics[width=\linewidth]{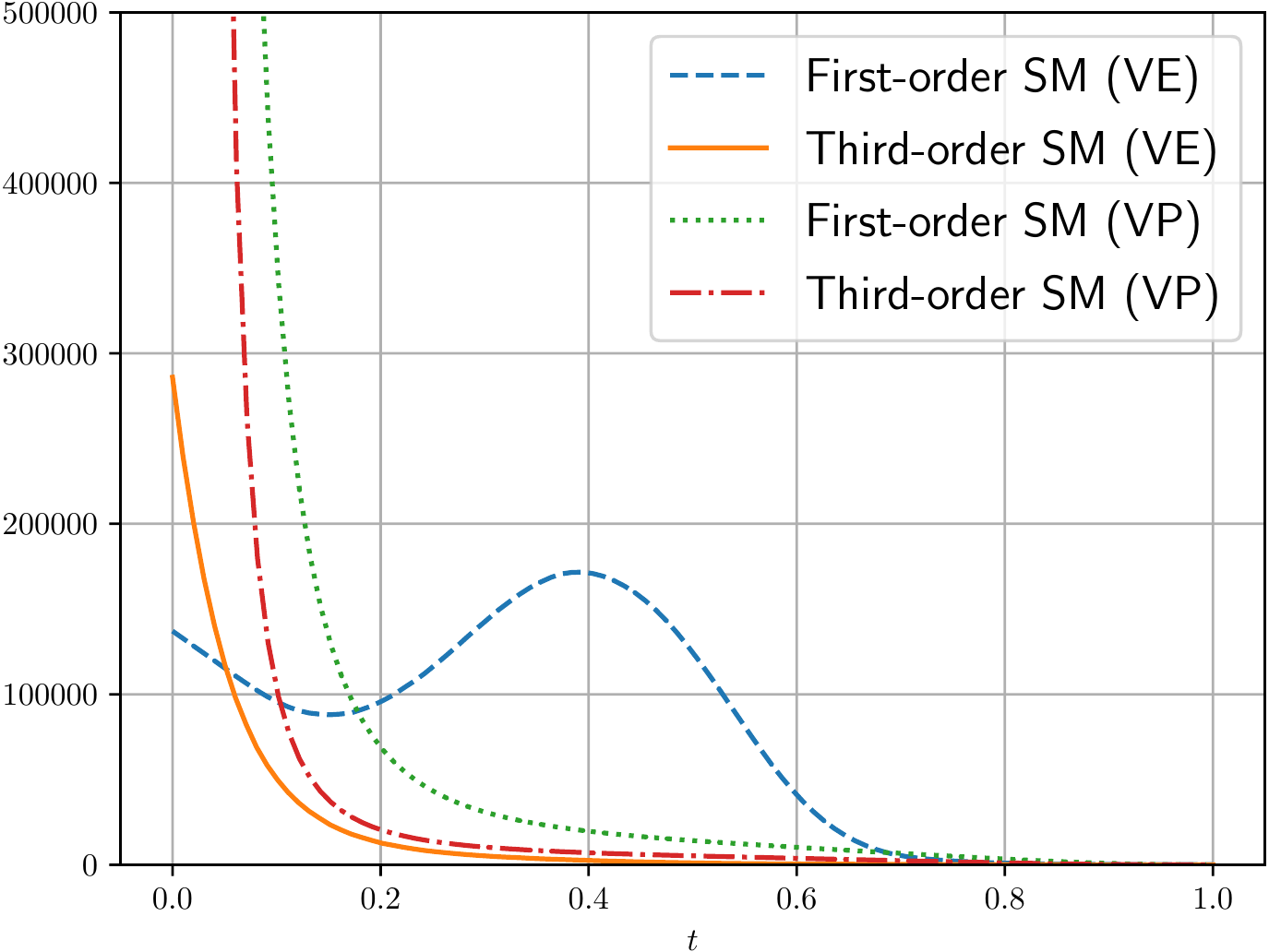}\\
\small{	(b) $\ell_{\DIFF}(t)$ (clipped by y-axis).}
\end{minipage}
	\caption{$\ell_{\DIFF}(t)$ of VP type and VE type on CIFAR-10, trained by first-order DSM with the likelihood weighting functions and importance sampling in~\citet{song2021maximum} and our proposed third-order DSM.\label{fig:s_ode_cifar10}}
\end{figure*}

\begin{figure*}[t]
	\centering
	\begin{minipage}{.28\linewidth}
		\centering
			\includegraphics[width=.96\linewidth]{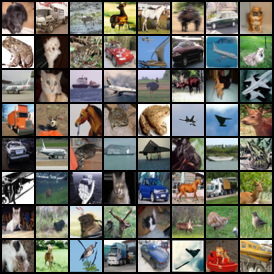}\\
\small{(a) First-order score matching}
	\end{minipage}
	\begin{minipage}{.28\linewidth}
	\centering
	\includegraphics[width=.96\linewidth]{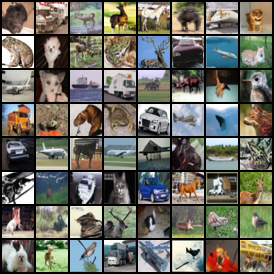}\\
\small{	(b) Second-order score matching}
\end{minipage}
	\begin{minipage}{.28\linewidth}
		\centering		
 	\includegraphics[width=.96\linewidth]{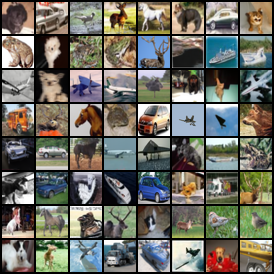}\\
\small{(c) Third-order score matching}
   \end{minipage}
	\caption{Random samples of SGMs (VE type) by PC sampler, trained by different orders of DSM.\label{fig:cifar10_samples}}
\end{figure*}

We also train VP and subVP types of ScoreODEs by the proposed high-order DSM method, with the maximum likelihood weighting function for ScoreSDE in~\citep{song2021maximum} (the weighting functions are detailed in \citep[Table 1]{song2021maximum}). Note that for the VE type, the likelihood weighting is exactly the DSM weighting used in our experiments.

We firstly show that for the ScoreODE of VP type, even on the simple 1-D mixture-of-Gaussians, the model density trained by the first-order DSM is not good enough and can be improved by the third-order DSM, as shown in Fig.~\ref{fig:mog_density_vp}.

We then train ScoreODE of VP and subVP types on the CIFAR-10 dataset. We use the pretrained checkpoint by first-order DSM in~\citep{song2021maximum} (the checkpoint of ``Baseline+LW+IS''), and further train 100k iterations. We use the same experiment settings as the VE experiments. We vary the start evaluation time $\epsilon$ and evaluate the model likelihood at each $\epsilon$, as shown in Fig.~\ref{fig:cifar10_density_vp}. For the ScoreODE trained by the first-order DSM, the model likelihood is poor when $\epsilon$ is slightly large. Note that \citet{song2020score,song2021maximum} uses $\epsilon=10^{-3}$ for sampling of VP type, and the corresponding likelihood is poor. Moreover, our proposed method can improve the likelihood for $\epsilon$ larger than a certain value (e.g. our method can improve the likelihood for the VP type with $\epsilon=10^{-3}$).
However, for very small $\epsilon$, our method cannot improve the likelihood for VP and subVP. We suspect that it is because of the ``unbounded score'' problem~\citep{dockhorn2021score}. The first-order score function suffers numerical issues for $t$ near $0$ and the first-order SM error is so large that it cannot provide useful information for the higher-order SM.

As the data score function $\nabla_{\x}\log q_t(\x_t)$ is unknown for CIFAR-10 experiments, we cannot evaluate $\Jc_{\FISH}(\theta)$. To show the effectiveness of our method, we evaluate the difference between $\s_\theta(\x_t,t)$ and $\nabla_{\x}\log p_t^{\ODE}(\x_t)$ in the $\Jc_{\DIFF}(\theta)$. Denote
\begin{equation}
    \ell_{\DIFF}(t)\coloneqq g(t)^2\E_{q_t(\x_t)}\|\s_\theta(\x_t,t)-\nabla_{\x}\log p_t^{\ODE}(\x_t)\|_2^2,
\end{equation}
we take $100$ time steps between $\epsilon=10^{-5}$ and $T=1$ to evaluate $\ell_{\DIFF}(t)$, where the ODE score function $\nabla_{\x}\log p_t^{\ODE}(\x_t)$ is computed by the method in Appendix.~\ref{appendix:change_of_score}, and the expectation w.r.t. $q_t(\x_t)$ are computed by the Monte-Carlo method, namely we use the test dataset for $q_0(\x_0)$ and then sample $\x_t$ from $q(\x_t|\x_0)$. We evaluate the $\ell_{\DIFF}(t)$ for the first-order DSM training (baseline) and the third-order DSM training for both the VP and VE types, as shown in Fig.~\ref{fig:s_ode_cifar10}. We can find that our proposed high-order DSM method can reduce $\ell_{\DIFF}(t)$ for {\ode}s for most $t\in[\epsilon,T]$. As $\s_\theta(\x_t,t)$ is an estimation of the true data score function $\nabla_{\x}\log q_t(\x_t)$, the results indicates that our proposed method can reduce $\ell_{\FISH}(t)$ of {\ode}s and then further reduce $\Jc_{\FISH}(\theta)$.
Also, we can find that the ``unbounded score'' problem of VE is much milder than that of VP, which can explain why our method can greatly improve the likelihood of VE type even for very small $\epsilon$.

Moreover, we randomly select a batch of generated samples by the PC sampler~\cite{song2020score} of the same random seed by the VE model of first-order, second-order and third-order DSM training, as shown in Fig.~\ref{fig:cifar10_samples}. The samples are very close for human eyes, which shows that after our proposed training method, the score model can still be used for the sample methods of SGMs to generate high-quality samples.

%% file: main.bbl
\begin{thebibliography}{31}
\providecommand{\natexlab}[1]{#1}
\providecommand{\url}[1]{\texttt{#1}}
\expandafter\ifx\csname urlstyle\endcsname\relax
  \providecommand{\doi}[1]{doi: #1}\else
  \providecommand{\doi}{doi: \begingroup \urlstyle{rm}\Url}\fi

\bibitem[Anderson(1982)]{anderson1982reverse}
Anderson, B.~D.
\newblock Reverse-time diffusion equation models.
\newblock \emph{Stochastic Processes and their Applications}, 12\penalty0
  (3):\penalty0 313--326, 1982.

\bibitem[Bao et~al.(2022)Bao, Li, Sun, Zhu, and Zhang]{bao2022estimating}
Bao, F., Li, C., Sun, J., Zhu, J., and Zhang, B.
\newblock Estimating the optimal covariance with imperfect mean in diffusion
  probabilistic models.
\newblock \emph{arXiv preprint arXiv:2206.07309}, 2022.

\bibitem[Bradbury et~al.(2018)Bradbury, Frostig, Hawkins, Johnson, Leary,
  Maclaurin, Necula, Paszke, Vander{P}las, Wanderman-{M}ilne, and
  Zhang]{jax2018github}
Bradbury, J., Frostig, R., Hawkins, P., Johnson, M.~J., Leary, C., Maclaurin,
  D., Necula, G., Paszke, A., Vander{P}las, J., Wanderman-{M}ilne, S., and
  Zhang, Q.
\newblock {JAX}: Composable transformations of {P}ython+{N}um{P}y programs,
  2018.
\newblock URL \url{http://github.com/google/jax}.

\bibitem[Chen et~al.(2020)Chen, Zhang, Zen, Weiss, Norouzi, and
  Chan]{chen2020wavegrad}
Chen, N., Zhang, Y., Zen, H., Weiss, R.~J., Norouzi, M., and Chan, W.
\newblock Wavegrad: Estimating gradients for waveform generation.
\newblock \emph{arXiv preprint arXiv:2009.00713}, 2020.

\bibitem[Chen et~al.(2018)Chen, Rubanova, Bettencourt, and
  Duvenaud]{chen2018neural}
Chen, R.~T., Rubanova, Y., Bettencourt, J., and Duvenaud, D.
\newblock Neural ordinary differential equations.
\newblock \emph{arXiv preprint arXiv:1806.07366}, 2018.

\bibitem[Deng et~al.(2009)Deng, Dong, Socher, Li, Li, and
  Fei{-}Fei]{deng2009imagenet}
Deng, J., Dong, W., Socher, R., Li, L., Li, K., and Fei{-}Fei, L.
\newblock Image{N}et: {A} large-scale hierarchical image database.
\newblock In \emph{2009 IEEE Conference on Computer Vision and Pattern
  Recognition}, pp.\  248--255. IEEE, 2009.

\bibitem[Dhariwal \& Nichol(2021)Dhariwal and Nichol]{dhariwal2021diffusion}
Dhariwal, P. and Nichol, A.
\newblock Diffusion models beat {GAN}s on image synthesis.
\newblock \emph{arXiv preprint arXiv:2105.05233}, 2021.

\bibitem[Dockhorn et~al.(2021)Dockhorn, Vahdat, and Kreis]{dockhorn2021score}
Dockhorn, T., Vahdat, A., and Kreis, K.
\newblock Score-based generative modeling with critically-damped langevin
  diffusion.
\newblock \emph{arXiv preprint arXiv:2112.07068}, 2021.

\bibitem[Finlay et~al.(2020)Finlay, Jacobsen, Nurbekyan, and
  Oberman]{finlay2020train}
Finlay, C., Jacobsen, J.-H., Nurbekyan, L., and Oberman, A.
\newblock How to train your {N}eural {ODE}: the world of {J}acobian and kinetic
  regularization.
\newblock In \emph{International Conference on Machine Learning}, pp.\
  3154--3164. PMLR, 2020.

\bibitem[Grathwohl et~al.(2018)Grathwohl, Chen, Bettencourt, Sutskever, and
  Duvenaud]{grathwohl2018ffjord}
Grathwohl, W., Chen, R.~T., Bettencourt, J., Sutskever, I., and Duvenaud, D.
\newblock {FFJORD}: Free-form continuous dynamics for scalable reversible
  generative models.
\newblock \emph{arXiv preprint arXiv:1810.01367}, 2018.

\bibitem[Ho et~al.(2020)Ho, Jain, and Abbeel]{ho2020denoising}
Ho, J., Jain, A., and Abbeel, P.
\newblock Denoising diffusion probabilistic models.
\newblock \emph{arXiv preprint arXiv:2006.11239}, 2020.

\bibitem[Huang et~al.(2021)Huang, Lim, and Courville]{huang2021variational}
Huang, C.-W., Lim, J.~H., and Courville, A.
\newblock A variational perspective on diffusion-based generative models and
  score matching.
\newblock \emph{arXiv preprint arXiv:2106.02808}, 2021.

\bibitem[Hutchinson(1989)]{hutchinson1989stochastic}
Hutchinson, M.~F.
\newblock A stochastic estimator of the trace of the influence matrix for
  {L}aplacian smoothing splines.
\newblock \emph{Communications in Statistics-Simulation and Computation},
  18\penalty0 (3):\penalty0 1059--1076, 1989.

\bibitem[Hyv{\"a}rinen \& Dayan(2005)Hyv{\"a}rinen and
  Dayan]{hyvarinen2005estimation}
Hyv{\"a}rinen, A. and Dayan, P.
\newblock Estimation of non-normalized statistical models by score matching.
\newblock \emph{Journal of Machine Learning Research}, 6\penalty0 (4), 2005.

\bibitem[Kingma \& Ba(2014)Kingma and Ba]{kingma2014adam}
Kingma, D.~P. and Ba, J.
\newblock Adam: A method for stochastic optimization.
\newblock \emph{arXiv preprint arXiv:1412.6980}, 2014.

\bibitem[Kingma et~al.(2021)Kingma, Salimans, Poole, and
  Ho]{kingma2021variational}
Kingma, D.~P., Salimans, T., Poole, B., and Ho, J.
\newblock Variational diffusion models.
\newblock \emph{arXiv preprint arXiv:2107.00630}, 2021.

\bibitem[Kong et~al.(2020)Kong, Ping, Huang, Zhao, and
  Catanzaro]{kong2020diffwave}
Kong, Z., Ping, W., Huang, J., Zhao, K., and Catanzaro, B.
\newblock Diffwave: A versatile diffusion model for audio synthesis.
\newblock \emph{arXiv preprint arXiv:2009.09761}, 2020.

\bibitem[Krizhevsky(2009)]{Krizhevsky09learningmultiple}
Krizhevsky, A.
\newblock Learning multiple layers of features from tiny images.
\newblock Technical report, 2009.

\bibitem[Luo \& Hu(2021)Luo and Hu]{luo2021diffusion}
Luo, S. and Hu, W.
\newblock Diffusion probabilistic models for 3{D} point cloud generation.
\newblock In \emph{Proceedings of the IEEE/CVF Conference on Computer Vision
  and Pattern Recognition}, pp.\  2837--2845, 2021.

\bibitem[Meng et~al.(2021{\natexlab{a}})Meng, Song, Li, and
  Ermon]{meng2021estimating}
Meng, C., Song, Y., Li, W., and Ermon, S.
\newblock Estimating high order gradients of the data distribution by
  denoising.
\newblock \emph{Advances in Neural Information Processing Systems}, 34,
  2021{\natexlab{a}}.

\bibitem[Meng et~al.(2021{\natexlab{b}})Meng, Song, Song, Wu, Zhu, and
  Ermon]{meng2021sdedit}
Meng, C., Song, Y., Song, J., Wu, J., Zhu, J.-Y., and Ermon, S.
\newblock {SDE}dit: Image synthesis and editing with stochastic differential
  equations.
\newblock \emph{arXiv preprint arXiv:2108.01073}, 2021{\natexlab{b}}.

\bibitem[Ramachandran et~al.(2017)Ramachandran, Zoph, and
  Le]{ramachandran2017searching}
Ramachandran, P., Zoph, B., and Le, Q.~V.
\newblock Searching for activation functions.
\newblock \emph{arXiv preprint arXiv:1710.05941}, 2017.

\bibitem[Skilling(1989)]{skilling1989eigenvalues}
Skilling, J.
\newblock The eigenvalues of mega-dimensional matrices.
\newblock In \emph{Maximum Entropy and Bayesian Methods}, pp.\  455--466.
  Springer, 1989.

\bibitem[Song \& Ermon(2019)Song and Ermon]{song2019generative}
Song, Y. and Ermon, S.
\newblock Generative modeling by estimating gradients of the data distribution.
\newblock \emph{arXiv preprint arXiv:1907.05600}, 2019.

\bibitem[Song \& Ermon(2020)Song and Ermon]{song2020improved}
Song, Y. and Ermon, S.
\newblock Improved techniques for training score-based generative models.
\newblock \emph{arXiv preprint arXiv:2006.09011}, 2020.

\bibitem[Song et~al.(2020{\natexlab{a}})Song, Garg, Shi, and
  Ermon]{song2020sliced}
Song, Y., Garg, S., Shi, J., and Ermon, S.
\newblock Sliced score matching: A scalable approach to density and score
  estimation.
\newblock In \emph{Uncertainty in Artificial Intelligence}, pp.\  574--584.
  PMLR, 2020{\natexlab{a}}.

\bibitem[Song et~al.(2020{\natexlab{b}})Song, Sohl-Dickstein, Kingma, Kumar,
  Ermon, and Poole]{song2020score}
Song, Y., Sohl-Dickstein, J., Kingma, D.~P., Kumar, A., Ermon, S., and Poole,
  B.
\newblock Score-based generative modeling through stochastic differential
  equations.
\newblock \emph{arXiv preprint arXiv:2011.13456}, 2020{\natexlab{b}}.

\bibitem[Song et~al.(2021)Song, Durkan, Murray, and Ermon]{song2021maximum}
Song, Y., Durkan, C., Murray, I., and Ermon, S.
\newblock Maximum likelihood training of score-based diffusion models.
\newblock \emph{arXiv e-prints}, pp.\  arXiv--2101, 2021.

\bibitem[Vahdat et~al.(2021)Vahdat, Kreis, and Kautz]{vahdat2021score}
Vahdat, A., Kreis, K., and Kautz, J.
\newblock Score-based generative modeling in latent space.
\newblock \emph{arXiv preprint arXiv:2106.05931}, 2021.

\bibitem[Vincent(2011)]{vincent2011connection}
Vincent, P.
\newblock A connection between score matching and denoising autoencoders.
\newblock \emph{Neural computation}, 23\penalty0 (7):\penalty0 1661--1674,
  2011.

\bibitem[Zhou et~al.(2021)Zhou, Du, and Wu]{zhou20213d}
Zhou, L., Du, Y., and Wu, J.
\newblock 3{D} shape generation and completion through point-voxel diffusion.
\newblock \emph{arXiv preprint arXiv:2104.03670}, 2021.

\end{thebibliography}
